\documentclass{article} % For LaTeX2e
\usepackage{iclr2027_conference,times}

\usepackage{amsmath,amsfonts,bm}

\def\eqref#1{equation~\ref{#1}}
\def\1{\bm{1}}

\DeclareMathAlphabet{\mathsfit}{\encodingdefault}{\sfdefault}{m}{sl}
\SetMathAlphabet{\mathsfit}{bold}{\encodingdefault}{\sfdefault}{bx}{n}

\newcommand{\pdata}{p_{\rm{data}}}
\usepackage{hyperref}
\usepackage{url}

\usepackage{booktabs}
\usepackage{tabularx}
\usepackage{graphicx}
\usepackage{amsthm}
\usepackage{amsmath}
\usepackage{amssymb}
\usepackage{mathtools}
\usepackage{subcaption}
\usepackage{algorithm}
\usepackage{algpseudocode}
\usepackage{tcolorbox}
\tcbuselibrary{listings}
\usepackage{wrapfig}
\usepackage{tikz}
\usetikzlibrary{arrows.meta, positioning, calc}

\newtheorem{theorem}{Theorem}
\newtheorem{lemma}{Lemma}
\newtheorem{corollary}{Corollary}
\newtheorem{proposition}{Proposition}
\newtheorem{assumption}{Assumption}

\theoremstyle{remark}
\newtheorem{remark}{Remark}

\title{There and Back Again: Bidirectional Diffusion Bridges for Multimodality Translation}

\author{Gabe Guo, Elon Litman, Thanawat Sornwanee, Jose Blanchet, Stefano Ermon \\
Stanford University\\
}

\newcommand{\bit}{\texttt{BIT}}
\newcommand{\pbase}{p_{\rm{base}}}

\iclrfinalcopy % Uncomment for camera-ready version, but NOT for submission.
\begin{document}

\maketitle

\lhead{Preprint. Under review.}

\begin{abstract}
Multimodality translation (\textit{e.g.}, text-to-image) is a core generative AI task. However, existing approaches (1) follow generative paths that do not directly represent the source modality, limiting the flexibility of some sampling algorithms; and (2) are unidirectional, preventing inversion (\textit{e.g.}, image-to-text). We propose \bit{}: \textbf{B}idirectional \textbf{I}mage-\textbf{T}ext Diffusion Bridges. In contrast to previous approaches, \bit{} starts directly from text and interpolates into images, providing (1) a source-aware generative path that enables diverse and flexible sampling algorithms; and (2) an endpoint-conditioned process that can be traversed from image to text, providing a \textit{unified}, bidirectional generative framework. \bit{} is derived through stochastic calculus, yielding SDE forms amenable to simulation and tractable loss functions that scale to high dimensions. Our experiments show that \bit{} is competitive with denoising-diffusion and deterministic-flow baselines, and outperforms them on several vision--language and natural-science evaluations.
\end{abstract}

\section{Introduction}

\noindent{\textbf{Goal:}} %In the quest to build human-like generative AI, 
It is a fundamental capability for generative AI models to translate across data modalities. We need look no further than common occupations: cartoonists translate textual stories into vivid images, singers translate visual music scores into audible melodies, and (most importantly) scientists translate sensory experiences into textual (and visual, if page limits allow) documents. %Even the not-yet-employed five year old who tells his mother that he is cold is translating his tactile experience: first into a textual, then into an audible representation.

\noindent{\textbf{Existing Diffusion Approaches:}}
Continuous-space diffusion and flow models (aka stochastic interpolants)~\citep{song2020score, albergo2025stochastic} underpin many dominant cross-modality translation approaches, notably in text-to-image (T2I)~\citep{rombach2022high, labs2025flux}. They are even gaining steam for unconditional text generation \citep{hu2026elf}, although continuous-space diffusion models that translate from other modalities (\textit{e.g.}, image) to text are still underexplored.

However, popular diffusion models' modality ``translation'' mechanism (main example: T2I) leaves much to be desired.
%They start from meaningless Gaussian noise that has no mutual information with the final image, then follow an SDE trajectory that interpolates between this meaningless noise and realistic images.
Essentially, they interpolate (via an SDE) between realistic images and Gaussian noise that has no mutual information with the final images.
To account for text input, they add it as conditioning to the drift network, typically via cross-attention \citep{rombach2022high}. The intermediate data samples along the generative trajectory never resemble the actual text prompts. %At no point on the generative trajectory does the intermediate data sample resemble the actual text prompt.

\noindent \textbf{\textit{Semantically Barren Paths Limit Editing Algorithms:}} At the noise endpoint, the trajectory has no mutual information with the final image. Theorem~\ref{thm:mutual_info} gives a local comparison showing when a data-to-data bridge retains more information about the target than an unconditioned noise-to-data bridge.
This matters because many image-editing algorithms (\textit{e.g.}, SDEdit \citep{meng2021sdedit}) traverse and reverse generative trajectories. As noise is injected, the evolving state can lose information about the reference image, limiting the ability of such algorithms to generate semantically related variations when the original conditioning signal is unavailable.

\noindent \textbf{\textit{Translation is Unidirectional:}} Furthermore, due to the indirect way of translating T2I in which a predetermined text prompt is passed as sidecar conditioning, T2I diffusion models provide no mechanism for \textit{reversing} the trajectory to obtain a text caption back from an image. Instead, people typically must use separate VLMs~\citep{bai2025qwen3, liu2023visual} trained with different discrete autoregressive losses. So, the T2I model is not a true cross-modality translator, as it only goes \textit{one way}. This limits the possibilities for cross-modality sampling algorithms. %and editing algorithms. %particularly those that make use of cross-modality information. 
Conceptually, it is also unappealing that T2I generation has a separate framework from image-to-text (I2T) generation.

\noindent{\textbf{Solution Criteria:}} Our generative cross-modality translation model should:
(1) Capture the joint probability distribution (couplings) of the data modalities, rather than a deterministic mapping. (One text prompt can map to many images, making past deterministic approaches like CycleGAN \citep{zhu2017unpaired} or data-to-data flow models \citep{albergo2025stochastic} unsuitable.)
(2) Construct a semantically meaningful interpolation as the generative trajectory between data modalities. This would yield high mutual information between intermediate states and the result, unlocking a new design space of sampling algorithms along the path.
(3) Be reversible at any point along the generative trajectory, so we can handle cross-modality ``inversion.''
(4) Unify T2I and I2T in one framework.
% Black-box: Satisfy joint law
% Generative process: Semantically meaningful interpolation (helps us make edits)
% Reversible along the path

\noindent{\textbf{Our Solution:}} We propose \bit{}: \textbf{B}idirectional \textbf{I}mage-\textbf{T}ext Diffusion Bridges.
% what it does
\bit{} is a unified multimodal model: given text prompts, \bit{} generates images; given images, \bit{} generates captions. \bit{} builds upon recent advances in diffusion bridges \citep{guo2026abc, zhou2023denoising}. 

% how it does it (mathematical framework + neural network learning)
\noindent \textbf{\textit{Stochastic Process Modeling:}} We define a stochastic process that starts from one domain (\textit{e.g.}, text captions) and gradually perturbs data into corresponding observations from another domain (\textit{e.g.}, images). 
We enforce the correct mapping by finite-horizon applications of Girsanov's Theorem~\citep{girsanov1960transforming} and Doob's $h$-transform~\citep{doob1984classical}, with the pinned endpoint obtained by continuous path extension: our derivations result in a simulatable SDE whose drift is estimated by a neural network.
To generate in the opposite direction (\textit{e.g.}, image to text caption), we derive and simulate the analytic time-reversal of this stochastic process. %via Anderson's Theorem~\citep{anderson1982reverse}. 
The reverse SDE's drift can also be learned with a closed-form loss target using the same architecture, with parameters either shared or separated by direction. At the population optimum and under the stated regularity conditions, our derivations show that (1) the endpoint-conditioned forward and reverse SDEs describe the same process in opposite time directions and have the desired endpoint joint distribution (\textit{e.g.}, image--caption pairs; Thms.~\ref{thm:forward_process} and~\ref{thm:reverse_process}); and (2) in the comparison setting of Theorem~\ref{thm:mutual_info}, the data-to-data path has a local mutual-information advantage over an unconditioned noise-to-data path. The implemented networks are trained as finite-sample approximations of these population-optimal drifts.
%Furthermore, our framework is applicable to domains besides text-to-image.
%many domains where we have multiple representations of the same underlying concept.

% engineering details
\noindent \textbf{\textit{Discrete-Continuous Unification:}} 
%Pertinent to the T2I task, 
There has traditionally been a break between generative algorithms for discrete data (text) and continuous data (images).
% So, a crucial question to answer is: how can we handle both discrete text and continuous image data in one framework? (After all, an SDE framework implies continuous-valued data.) 
To handle both modalities in an SDE framework, we make everything continuous by constructing a continuous representation of the text. Specifically, we create a fixed lookup table from discrete text tokens to continuous text embeddings, similar to ELF~\citep{hu2026elf}. %Our mathematical framework is agnostic to reasonable choices of lookup table; 
We find that features in foundation models like Qwen-Embedding~\citep{yang2025qwen3} make for concise, empirically near-decodable text token embeddings.
%concise (64-dimensional), invertible ($99\%+$ decoding accuracy) text token embeddings. 

\noindent\textbf{Results:} Our empirical validation shows:
(1) \bit{} is competitive with, and sometimes better than, noise-to-data diffusion baselines on I2T and T2I generation. 
(2) \bit{} outperforms noise-to-data diffusion in the evaluated data-variation tasks, consistent with the information-retention motivation above.
(3) \bit{} also applies to scientific domains and obtains the best average rank among the evaluated cell-fate models in our experiments.

\noindent{\textbf{Summary of Contributions:}} (1) \bit{} represents a paradigm shift by \textit{challenging the assumption that T2I generation needs to start from a Gaussian noise distribution}. %Instead, \bit{} starts directly from a representation of the source domain (text) and diffuses smoothly to the target domain (image). %; rather than starting from uninformative noise and only injecting text conditioning information through a side channel.
(2) \bit{}'s direct path from text to images unlocks the new capability to trace reverse paths to text \textit{from} images, paving the way for new sampling algorithms and unified text-image modeling. 
(3) We evaluate \bit{} on T2I, I2T, and scientific domains such as cell-fate modeling, where it is competitive with unidirectional diffusion and flow baselines. \textbf{\textit{See \texttt{\href{https://bit-diffusion.github.io}{https://bit-diffusion.github.io}} for demo videos. 
See \textcolor{red}{\texttt{\href{https://github.com/gabeguo/bit_diffusion}{https://github.com/gabeguo/bit\_diffusion}} for code.}
}}

\section{Preliminaries}

\noindent\textbf{Diffusion Models:} Diffusion models turn noise into data via an SDE running from $t = 1$ to $t = 0$:
{\setlength{\abovedisplayskip}{2pt}\setlength{\belowdisplayskip}{2pt}
\begin{equation}
    dx_t = \left[ f(x_t, t) - g(t)^2 \nabla_{x_t} \log p_t(x_t) \right] d\overleftarrow{t} + g(t)\, d\overleftarrow{B_t}, \quad x_1 \sim \mathcal{N}(0, I),
\end{equation}
where $d\overleftarrow{B_t}$ is reverse-time Brownian motion, $d\overleftarrow{t}$ is an infinitesimally small negative time increment, $g(t)$ is a fixed volatility coefficient, $f(x_t, t)$ is a closed-form drift coefficient, and $\nabla_{x_t} \log p_t(x_t)$ is the score function, which is typically approximated by a neural network $\mathbf{s}_\theta(x_t, t)$~\citep{song2020score}.}

\noindent\textbf{Diffusion Bridge Models:}
Rather than starting from noise, diffusion bridges construct SDEs interpolating directly from data-to-data~\citep{zhou2023denoising}. In \cite{guo2026abc}'s formulation, the SDE goes in forward time over $t \in [0, 1]$:
{\setlength{\abovedisplayskip}{2pt}\setlength{\belowdisplayskip}{2pt}
\begin{equation}
    dx_t = \left[-a(t)x_t + \sigma(t)^2 \mathbf{s}_\theta(t, x_t, x_\text{history})\right]dt + \sigma(t)dB_t,
\end{equation}
}
where $dt$ is an infinitesimally small positive time increment, $dB_t$ is Brownian motion, $a(t), \sigma(t)$ are predefined drift and volatility coefficients, and $\mathbf{s}_\theta(t, x_t, x_\text{history})$ is a path-dependent neural network score.
A chief advantage of diffusion bridges over diffusion models is the semantically rich paths.

\section{Probabilistic Modeling Framework}

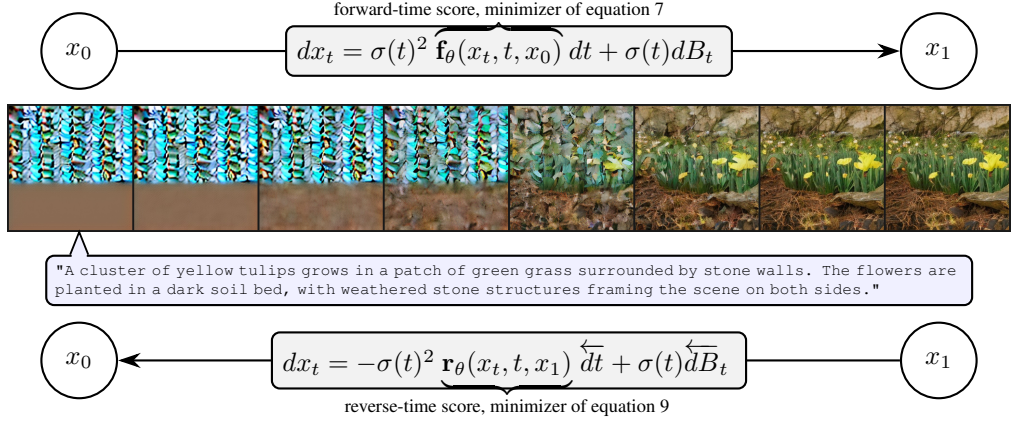
\begin{figure}
\centering
\begin{tikzpicture}[
    node distance=2cm,
    every node/.style={font=\normalsize},
    mycircle/.style={circle, draw=black, thick, minimum size=1cm, inner sep=1pt},
    myarrow/.style={-{Stealth[length=3mm, width=2mm]}, thick}
]
% ---------- forward row (y = 0) ----------
\node[mycircle] (start) at (0.025\textwidth + 0.45cm, 0) {$x_0$};
\node[mycircle] (end)   at (0.975\textwidth - 1.45cm, 0) {$x_1$};
\draw[myarrow] (start) -- (end);
\node[fill=gray!10, draw=black, thick, rounded corners=3pt, inner sep=4pt,
      align=center] at ($(start)!0.5!(end)$) {
$dx_t = \sigma(t)^2 \ \smash{
    \overbrace{\mathbf{f}_\theta\!\left(x_t, t, x_0\right)}^{\mathclap{\substack{
        \text{forward-time score, minimizer of \eqref{eqn:loss_forward}}
        }}
    }
} \ dt + \sigma(t)dB_t$
};

% ---------- the strip, as a node ----------
\node[anchor=north, inner sep=0pt] (strip)
      at ([yshift=-0.7cm] $(start)!0.5!(end)$)
      {\includegraphics[width=0.95\linewidth]{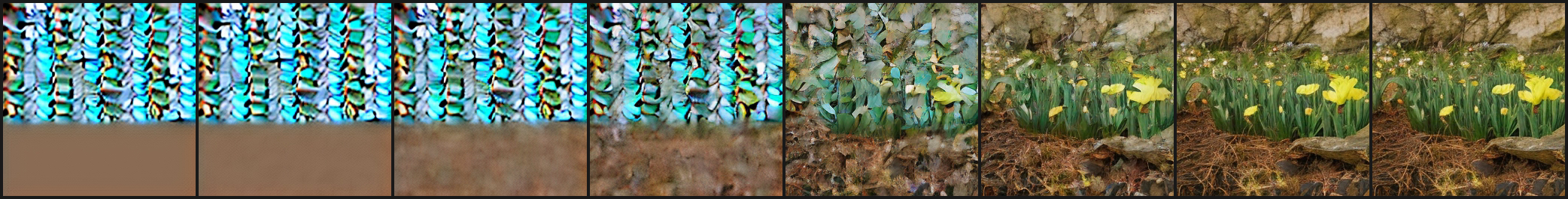}};
% \end{tikzpicture}
% \includegraphics[width=0.95\linewidth]{diagrams/step_0122500_batch_1_fwd_trajectory_cfg_0.0_sde_strip_8.png}
% \begin{tikzpicture}[
%     node distance=2cm,
%     every node/.style={font=\normalsize},
%     mycircle/.style={circle, draw=black, thick, minimum size=1cm, inner sep=1pt},
%     myarrow/.style={-{Stealth[length=3mm, width=2mm]}, thick}
% ]

% ---------- bubble ----------
\node[
    draw=black, thick, rounded corners=3pt, fill=blue!6,
    text width=0.86\textwidth, align=left, font=\tiny, anchor=north west
] (bubble) at ([xshift=-0.45cm, yshift=-0.30cm] strip.south -| start) {
\texttt{"A cluster of yellow tulips grows in a patch of green grass surrounded by stone walls. The flowers are planted in a dark soil bed, with weathered stone structures framing the scene on both sides."}
};

% ---------- tail ----------
% base half-width and apex
\coordinate (bl) at ($(bubble.north -| start)+(-0.12cm,0)$);
\coordinate (br) at ($(bubble.north -| start)+( 0.12cm,0)$);
\coordinate (tip) at (strip.south -| start);

% 1. erase the bubble's top border under the tail (slightly SHORTER than base)
\draw[blue!6, line width=1.6pt]
      ($(bl)+(0.02cm,-0.01cm)$) -- ($(br)+(-0.02cm,-0.01cm)$);
% 2. fill the triangle (no stroke)
\path[fill=blue!6] (bl) -- (tip) -- (br) -- cycle;
% 3. stroke only the two sides — note: no cycle
\draw[line width=0.7pt] (bl) -- (tip) -- (br);

% bump the reverse row down to make room
\node[mycircle] (rstart) at ([yshift=-1.7cm] strip.south -| start) {$x_0$};
% % ---------- reverse row, below the strip ----------
% \node[mycircle] (rstart) at ([yshift=-0.9cm] strip.south -| start) {$x_0$};
\node[mycircle] (rend)   at (rstart -| end) {$x_1$};
\draw[myarrow] (rend) -- (rstart);
\node[fill=gray!10, draw=black, thick, rounded corners=3pt, inner sep=4pt,
      align=center] at ($(rstart)!0.5!(rend)$) {
$dx_t = -\sigma(t)^2 \ \smash{
    \underbrace{\mathbf{r}_\theta\!\left(x_t, t, x_1\right)}_{\mathclap{\substack{
        \text{reverse-time score, minimizer of \eqref{eqn:reverse_time_loss}}
        }}
    }
} \ \overleftarrow{dt} + \sigma(t)\overleftarrow{dB}_t$
};
\end{tikzpicture}
\caption{\textbf{\bit{} Schematic:} Forward and reverse-time SDEs (with drift parameterized by neural net) both traverse the same generative path, allowing text-image translation. Text is represented as invertible token embeddings (blue-ish texture at $x_0$), with zero-padding (brown bottom at $x_0$).}
\end{figure}

We want a stochastic process satisfying endpoint law $(X_0, X_1) \sim \pdata$, where $X_0$ can be text and $X_1$ can be image representations. $\pdata$ denotes the population joint law of the paired endpoints, and $\pdata(X_0), \pdata(X_1)$ denote its marginals. It should be simulatable in forward and reverse time.

\noindent\textbf{Population and realizability convention.}
All objectives in this section are population objectives. Let $\mathbf f^\star$ and $\mathbf r^\star$ denote the fixed jointly Borel conditional-expectation (equivalently, $h$-transform) versions of the forward and reverse oracle drifts constructed in Appendices~\ref{sec:forward_process_proof}--\ref{sec:reverse_process_proof}. Realizability means that their weighted-$L^2$ equivalence classes have representatives in the corresponding parameterized drift classes. Every exact global population minimizer agrees with the relevant oracle under the weighted training-input law. Whenever a drift is used as an SDE coefficient in an exact path-law statement, however, we use the fixed oracle version: equality only on a training-null-set complement does not by itself imply that an arbitrary pointwise modification defines the same, or a unique, SDE law. In practice, the networks are trained as finite-sample approximations of these objectives; no exact path-law claim is made for a finite trained model.

\subsection{Data-Generating Process as Mixture of Endpoint-Pinned Bridges}\label{sec:q_law}

\noindent\textbf{Base Process} We start with the data-independent base process $\left(X_t\right)_{t \in [0, 1]}$ whose dynamics follow
\begin{equation}\label{eqn:base_process}
    dX_t = \sigma(t)dB_t,
\end{equation}
where $B_t$ is Brownian motion under $\mathbb{P}$ and is independent of $X_0$, $\sigma(t)$ is a deterministic positive scalar function satisfying Assumption~\ref{assumption:1}, and $X_0 \sim \pdata(X_0)$. Let $\pbase$ denote the transition density under $\mathbb{P}$; under this base measure, it is \textit{not} required that $X_1 \sim \pdata(X_1)$.

\noindent\textbf{Data-Generating Measure} Using the regular conditional bridge kernels of Eq.~\ref{eqn:base_process}'s base process, we define a path measure $\mathbb{Q}$ with the desired joint endpoint law.\footnote{Those unfamiliar with measure theory can imagine ``lucky'' Brownian motion that hits the desired data.} Conditional on both endpoints, the path measure (and volatility structure) is the same as that of the original scaled Brownian motion:
\begin{align}
    (X_0, X_1) &\sim_{\mathbb{Q}} \pdata \label{eqn:q_endpoint_law} \\
    \mathbb{Q}(d\omega | X_0=x_0, X_1=x_1) &= \mathbb{P}(d\omega | X_0=x_0, X_1=x_1) \label{eqn:pinned_path_measure}
\end{align}
Here the conditional laws are regular conditional Gaussian-bridge kernels, and $\mathbb Q$ is their mixture over $\pdata$. The $h$-transform used below gives absolute continuity with respect to the base law only after restriction to each $\mathcal F_T$, $T<1$. We do not assert absolute continuity on the pinned terminal sigma-field $\mathcal F_1$; the endpoint is instead obtained from the almost-sure continuous extension of the bridge path.
We next derive tractable SDEs for the forward ($X_0 \rightarrow X_1$) and reverse ($X_1 \rightarrow X_0$) dynamics, allowing us to simulate data generation in both directions under $\mathbb{Q}$.

\subsection{Forward-Time Model and Training Objective (Text-to-Image)}

In the following theorems, $\pbase(x_t | x_0, x_1)$ is a probability distribution of $X_t$ under $\mathbb{P}$: we show in Appendix \ref{sec:forward_process_proof} that it is Gaussian with parameters that are primitives of $\sigma(t)$.

\begin{theorem}[Forward-Time SDE and Objective]\label{thm:forward_process}
Fix a regular conditional kernel $x_0\mapsto\pdata(dx_1\mid x_0)$, and let $\mathbf f^\star(x,t;x_0)$ be the resulting fixed Borel $h$-transform version of
\[
\mathbb E_{\mathbb Q}\!\left[\nabla_x\log p_{\rm base}(X_1\mid x)
\,\middle|\,X_t=x,\ X_0=x_0\right],\qquad 0<t<1,
\]
where the conditional-expectation equality is understood under the bridge sampling law. There exists a canonical weak solution $(\overrightarrow X,B)$ on $[0,1)$ such that $\overrightarrow X_0\sim\pdata(X_0)$, $B$ is Brownian with respect to the solution filtration and independent of $\overrightarrow X_0$, and, on every $[0,T]$ with $T<1$,
\begin{equation}\label{eqn:data_gen}
    d\overrightarrow{X}_t = \sigma(t)^2\mathbf{f}^\star(\overrightarrow{X}_t, t; \overrightarrow{X}_0)dt + \sigma(t)dB_t,
    \qquad 0\le t<1.
\end{equation}
The process has an almost-sure continuous extension $\overrightarrow X_1=\lim_{t\uparrow1}\overrightarrow X_t$, and the law of the extended path on $C([0,1];\mathbb R^d)$ is $\mathbb Q$. Moreover, $\mathbf f^\star$ is the unrestricted weighted-$L^2$ minimizer of
\begin{equation}\label{eqn:loss_forward}
    \overrightarrow{\mathcal{L}}({\mathbf{f}}_\theta) = \underset{\substack{t \sim \mathcal{U}((0, 1)) \\ (x_0, x_1) \sim \pdata \\ x_t \sim \pbase(x_t | x_0, x_1)}}{\mathbb{E}}\left[ w(t)\left\| \mathbf{f}_\theta(x_t, t; x_0) - \nabla_{x_t}\text{\normalfont log }p_\text{base}(x_1 | x_t) \right\|^2_{2}\right].
\end{equation}
The weight function $w:(0,1) \to (0,\infty)$ is chosen subject to Assumption~\ref{assumption:weight}. Under realizability, every exact parameterized minimizer agrees with $\mathbf f^\star$ under the weighted training-input law.
\end{theorem}
\begin{proof}
    See Appendix \ref{sec:forward_process_proof}.
\end{proof}

Thm.~\ref{thm:forward_process} gives us the mathematical machinery to simulate the forward data translation process, \textit{e.g.}, from text at $X_0$ to images at $X_1$, giving samples from a fixed version of $\pdata(dx_1\mid x_0)$ at the population oracle for $\pdata(X_0)$-almost every $x_0$. It does this via (1) a data-generating SDE (Eq.~\ref{eqn:data_gen}) on finite horizons, with its endpoint supplied by continuous extension, that we can discretize numerically; (2) a neural network loss function (Eq.~\ref{eqn:loss_forward}) that allows us to estimate this SDE's drift from data. We parameterize $\mathbf{f}_\theta$ with a transformer \citep{peebles2023scalable}; behavior at out-of-distribution source endpoints is extrapolative rather than part of the exact theorem.

If instead we remove conditioning on $\overrightarrow{X}_0$, the population objective uses the loss integrand $w(t)\left\| \mathbf{f}_\theta(x_t, t; \varnothing) - \nabla_{x_t}\text{\normalfont log }p_\text{base}(x_1 | x_t) \right\|^2_{2}$. The canonical projected drift $\bar{\mathbf f}$ is state-only. Corollary~\ref{prop:marginal-match} constructs at least one weak solution with this fixed drift and the same one-time marginals as $\mathbb Q$; it makes no uniqueness or Markov-property claim about the selected law or other weak solutions. This marginal equivalence provides aggregate target calibration only: it does not identify the transition kernel from a fixed intermediate state or guarantee conditional round-trip fidelity, which we evaluate empirically. This state-only version does not require endpoint conditioning and is useful for the sampling schemes in Sections~\ref{sec:cross_modal_variation} and~\ref{sec:miscellaneous_comparisons}. See Appendix~\ref{sec:marginal_sdes} for details.

\subsection{Reverse-Time Model and Training Objective (Image-to-Text)}

%The previous section only let us generate in one direction: text-to-image. 
%We also wish to trace the path in reverse, giving us image-to-text generation.

\begin{theorem}[Reverse-Time SDE and Objective]\label{thm:reverse_process}
Fix a regular conditional kernel $x_1\mapsto\pdata(dx_0\mid x_1)$, and let $\mathbf r^\star(x,t;x_1)$ be the resulting fixed Borel $h$-transform version of
\[
\mathbb E_{\mathbb Q}\!\left[\nabla_x\log p_{\rm base}(x\mid X_0)
\,\middle|\,X_t=x,\ X_1=x_1\right],\qquad 0<t<1.
\]
The conditional-expectation equality is understood under the bridge sampling law. There exists a canonical weak solution $(\overleftarrow X,B)$ on reverse time $u\in[0,1)$ such that $\overleftarrow X_0\sim\pdata(X_1)$, $B$ is Brownian w.r.t. the solution filtration and independent of $\overleftarrow X_0$, and, on every $[0,T]$ with $T<1$,
\begin{equation}\label{eqn:reverse}
    d\overleftarrow{X}_u = \sigma(1-u)^2\mathbf{r}^\star(\overleftarrow{X}_u, 1-u; \overleftarrow{X}_0)du + \sigma(1-u)dB_u,
    \qquad 0\le u<1.
\end{equation}
The process has an almost-sure continuous extension at $u=1$, and
\[
\operatorname{Law}\bigl((\overleftarrow X_{1-t})_{0\le t\le1}\bigr)=\mathbb Q.
\]
Moreover, $\mathbf r^\star$ is the unrestricted weighted-$L^2$ minimizer of
\begin{equation}\label{eqn:reverse_time_loss}
    \overleftarrow{\mathcal{L}}(\mathbf{r}_\theta) = \underset{\substack{t \sim \mathcal{U}((0, 1)) \\ (x_0, x_1) \sim \pdata \\ x_t \sim \pbase(x_t | x_0, x_1)}}{\mathbb{E}}\left[ w(t)\left\| \mathbf{r}_\theta(x_t, t; x_1) - \nabla_{x_t}\text{\normalfont log }p_\text{base}(x_t | x_0) \right\|^2_{2}\right]
\end{equation}
with the same weight function as in Theorem~\ref{thm:forward_process}. Under realizability, every exact parameterized minimizer agrees with $\mathbf r^\star$ under the weighted training-input law.
\end{theorem}
\begin{proof}
    See Appendix~\ref{sec:reverse_process_proof}. Effectively, this is just a time-index inversion of Theorem~\ref{thm:forward_process}.
\end{proof}

This is Thm.~\ref{thm:forward_process}'s dual. It lets us translate data in the opposite direction: from image at $\overrightarrow{X_1}=\overleftarrow{X_0}$ to text at $\overrightarrow{X_0}=\overleftarrow{X_1}$, giving samples from a fixed version of $\pdata(dx_0\mid x_1)$ for $\pdata(X_1)$-almost every $x_1$. Similarly, it provides an SDE we can simulate in reverse-time, and a training objective for the neural network that provides the data-translating drift; there is also a selected state-only reverse-time weak solution with matching marginals in App.~\ref{sec:marginal_sdes} (Corollary~\ref{prop:marginal-match-rev}). Inputs outside the population support are not covered by the exact conditional-law statement.

\subsection{Classifier-Free Guidance}\label{sec:cfg}

Classifier-free guidance improves sample quality in diffusion models \citep{ho2022classifier}.
% \begin{equation}\label{eqn:cfg}
%     \mathbf{s}_\theta(x_t, t, c, w) = (1+w)\mathbf{s}_\theta(x_t, t; c) - w\mathbf{s}_\theta(x_t, t; \emptyset)
% \end{equation}
We use the following CFG-inspired extrapolation between the endpoint-conditioned and unconditioned bridge drifts from Theorems~\ref{thm:forward_process}--\ref{thm:reverse_process} and Corollaries~\ref{prop:marginal-match}--\ref{prop:marginal-match-rev}. Here $\omega\ge 0$ is the guidance strength:
\begin{align}
    \mathbf{f}_\theta(x_t, t, x_0, \omega) &= (1+\omega)\mathbf{f}_\theta(x_t, t; x_0) - \omega\mathbf{f}_\theta(x_t, t; \emptyset) \\
    \mathbf{r}_\theta(x_t, t, x_1, \omega) &= (1+\omega)\mathbf{r}_\theta(x_t, t; x_1) - \omega\mathbf{r}_\theta(x_t, t; \emptyset).
\end{align}
Unlike standard diffusion CFG, this bridge-drift extrapolation is used here as an empirical heuristic; we do not claim that it samples from an analytically identified guided distribution.

% \noindent\textbf{\textcolor{red}{Probability Flow ODE:}}

\section{From Discrete to Continuous Token Embeddings}

\noindent\textbf{Criterion:}
We have one snag in this otherwise elegant continuous-space mathematical framework: text is discrete. We circumvent this by embedding text tokens into continuous space.
This continuous representation must (1) form a near-invertible mapping with discrete text, so we get distinguishable outputs; (2) have the same dimensionality as image data, as per our mathematical framework.

\noindent\textbf{Framework:}
(1) Towards invertibility, we embed each token \textit{independently} without attending to the other tokens in the sequence. If we naively used the hidden state of each token after attending to all the other tokens, there would be too many possibilities for which continuous embedding corresponds to which token, complicating the learning task. Embedding each token independently effectively creates a lookup table whose number of entries is the same as the tokenizer vocabulary size, greatly simplifying the problem. Then, we just need a backbone language model that is powerful enough to create a distinct representation for each token, which we confirm empirically.
(2) Towards dimensionality matching, we fix the number of tokens (by truncation if the caption is too long, by padding if the caption is too short) per caption, and fix the embedding dimension of each token; such that the total size is the same as the number of components in the image latent code.

\noindent\textbf{Instantiation:}
(1) For token-to-embedding, we use \texttt{Qwen3-Embedding-8B}~\citep{zhang2025qwen3} to generate the token-by-token embedding lookup table, because it (a) is very expressive (b) supports truncation of the embedding dimensions to our desired size, via Matryoshka representations~\citep{kusupati2022matryoshka}.
For embedding-to-token, we train a small MLP with cross-entropy loss to convert embeddings into tokens: this achieves $>99\%$ accuracy, even under noisy embeddings, demonstrating empirical near-decodability.
(2) Regarding dimensionality, we use the Stable Diffusion VAE on $256\times256$ images, so the image latents have size $4\times32\times32 = 4096$. To match this, we have $64$ text tokens per image (after caption truncation or padding), each embedded into a truncated representation of $64$ dimensions, making for $64\times64=4096 \checkmark$ total dimensionality in the text side. Then, we can proceed with our diffusion bridge framework. Read App.~\ref{app:text_bridge}.

\section{Comparison to Other Model Classes}\label{sec:comparison}

\noindent\textbf{Noise-to-Data:}
Intuitively, starting the image/text generative process from a text/image representation should be semantically richer and more informative than starting from Gaussian noise. %Theorem~\ref{thm:mutual_info} verifies this from an information-theoretic perspective.
\begin{theorem}[Local mutual-information advantage of a data-to-data bridge]\label{thm:mutual_info}
Let $T$ denote a text representation, let $Y$ denote its paired image
representation, and let $N$ be noise independent of $(T,Y)$. Consider the
text-to-image and noise-to-image bridges
\begin{align}
    X_t^{\mathrm{TI}}
        &= a_t T + b_t Y + \eta_t Z,\label{eq:ti_bridge} \\
    X_t^{\mathrm{NI}}
        &= a_t N + b_t Y + \eta_t Z',\label{eq:ni_bridge}
\end{align}
where $Z$ and $Z'$ are independent noise variables, independent of all
endpoints. Suppose that
% \begin{enumerate}
    $a_0=1$, $b_0=0$, and $\eta_0=0$;
    $0<I(T;Y)<\infty$; and
    the functions
    $
        f(t) \coloneqq I(X_t^{\mathrm{TI}};Y),
        %\qquad
        g(t) \coloneqq I(X_t^{\mathrm{NI}};Y)
    $
    are right-continuous at $t=0$.
% \end{enumerate}
Then there exists $\tau>0$ such that, for every $t\in[0,\tau)$,
\[
    I(X_t^{\mathrm{TI}};Y)
    >
    I(X_t^{\mathrm{NI}};Y).
\]
\end{theorem}

\begin{proof}
See Section \ref{sec:local_mutual_info_proof} for proof and statement of the opposite direction.
\end{proof}
% \note{This assumed right-continuity is very strong assumption. Maybe we can move some stronger results to here, ie prop 1 and cor 4 in Appendix G.}

Thm.~\ref{thm:mutual_info} concerns settings in which the evolving state itself must retain the source information, as in the round-trip procedures of Sections~\ref{sec:cross_modal_variation} and~\ref{sec:bridge_editing}. It does not compare against a standard conditional T2I sampler that retains the text prompt as side information throughout the trajectory, and with retained side information the theorem gives no ordering. The state-only comparison is relevant to the evaluated round-trip setting, where a suitable auxiliary condition need not be supplied a priori (\textit{e.g.}, for an uncaptioned source image).

WLOG, Thm.~\ref{thm:mutual_info} is presented for the text-to-image direction, but it also applies to image-to-text (with flipped time).
% The $\tau$ bound could be vacuous, depending on the bridge geometry. 
App.~\ref{app:gaussian-mi} shows that for positively correlated jointly Gaussian scalar random variables and \textit{under \bit{}'s choice of interpolation coefficients, strict dominance holds for every $t\in[0,1)$, so $\tau=1$}.
For arbitrary joint distributions, how far $\tau$ extends is an empirical question; our experiments (Sec.~\ref{sec:cross_modal_variation}) are consistent with the hypothesis that $\tau$ covers a large range for text-to-image. 
%Particularly, in the near-endpoint regime (when we are close to text), we see a clear benefit in translation adherence, which is exactly the $[0, \tau)$ regime the theory is strongly supported for. 

\noindent\textbf{ODE vs. SDE:}
ODEs linking text modalities directly to image modalities should fail at modeling multi-peaked distributions, since ODEs are deterministic. So, distribution-to-distribution flow matching (without hacks like endpoint noise injection~\citep{albergo2023stochastic, he2025flowtok}) is unsuitable for modeling the range of images that could correspond to a text prompt, and vice versa.
In contrast, SDEs (due to noise injection) have hope for modeling multi-peaked distributions.

\section{Experiments}\label{sec:experiments}

%\input{details/training}

%\subsection{Setup}

\noindent\textbf{Comparisons}
In choosing baselines, we assess the impact of (1) data-to-data versus noise-to-data interpolation in the generative process (corresponding to diffusion~\citep{song2020score}); and (2) probabilistic versus deterministic modeling in the cross-modality bridge (corresponding to flow-matching~\citep{albergo2025stochastic}).
The first baseline is therefore an ablation that replaces one endpoint of the bridge with i.i.d.\ Gaussian noise and passes the cross-modality information through cross-attention conditioning. For text-to-image generation, for example, we start from noise and condition the score network on the text tokens, as in common text-to-image models~\citep{rombach2022high}. For fairness, we keep the SDE volatility schedule fixed.
The second baseline is an ablation where, instead of a stochastic process, we construct a deterministic ODE by flow matching~\citep{albergo2025stochastic} between coupled text-image pairs.

\noindent\textbf{Model Architecture} We build upon \cite{guo2026abc}'s diffusion transformers. We use consistent settings across baselines, to isolate the effects of transport stochasticity and endpoint. See App.~\ref{app:model_architecture}.

\noindent\textbf{Dataset} We use approximately 100 million images from the GPIC text-to-image dataset~\citep{chandrasegaran2026gpic}, with images preprocessed into Stable Diffusion VAE latents~\citep{rombach2022high} and text preprocessed into Qwen embeddings~\citep{zhang2025qwen3}. See Appendix~\ref{app:dataset}.

\noindent\textbf{Training and Inference Details} See App.~\ref{app:training_details} for hyperparameters. See App.~\ref{app:train_alg} for training, and App.~\ref{app:inference_alg} for inference pseudocode. Generally, we use the same hyperparameters across baselines, so that the performance difference can be attributed to the endpoint choice and stochasticity in transport.

\begin{figure}[t]
\centering
\begin{subfigure}[]{0.49\linewidth}
\centering
    \includegraphics[width=0.49\linewidth]{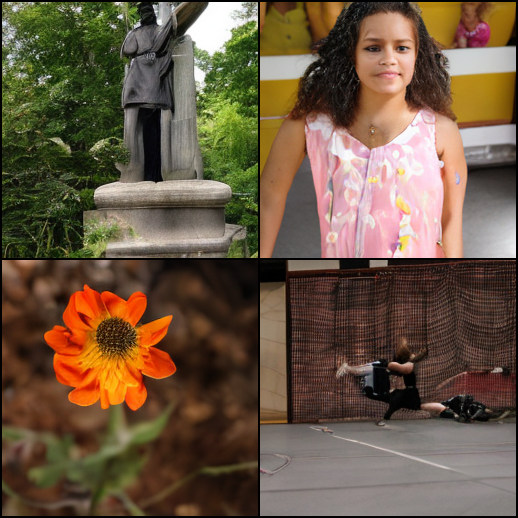}
    \includegraphics[width=0.49\linewidth]{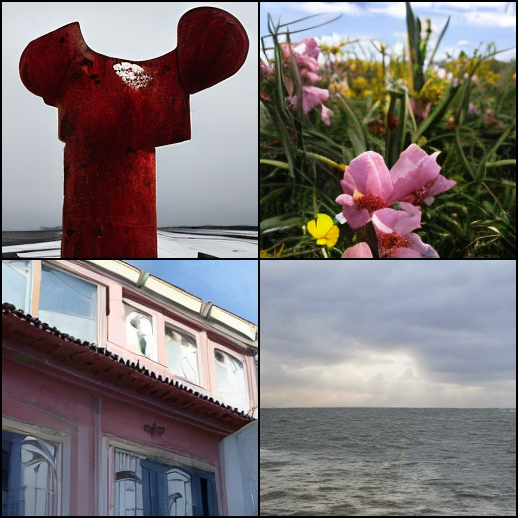}
    \includegraphics[width=0.49\linewidth, trim={0 0 0 0.06\linewidth},clip]{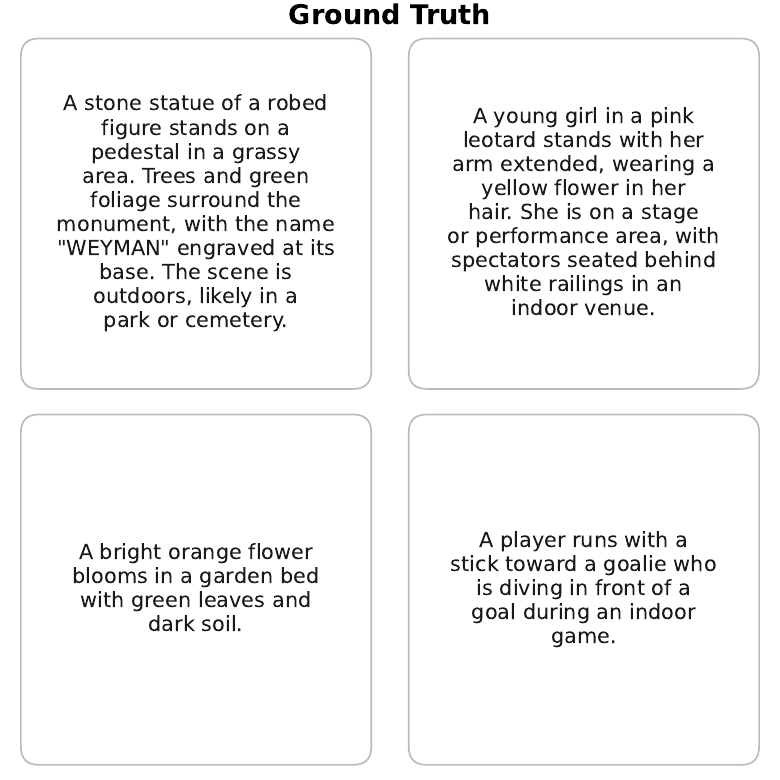}
    \includegraphics[width=0.49\linewidth, trim={0 0 0 0.06\linewidth},clip]{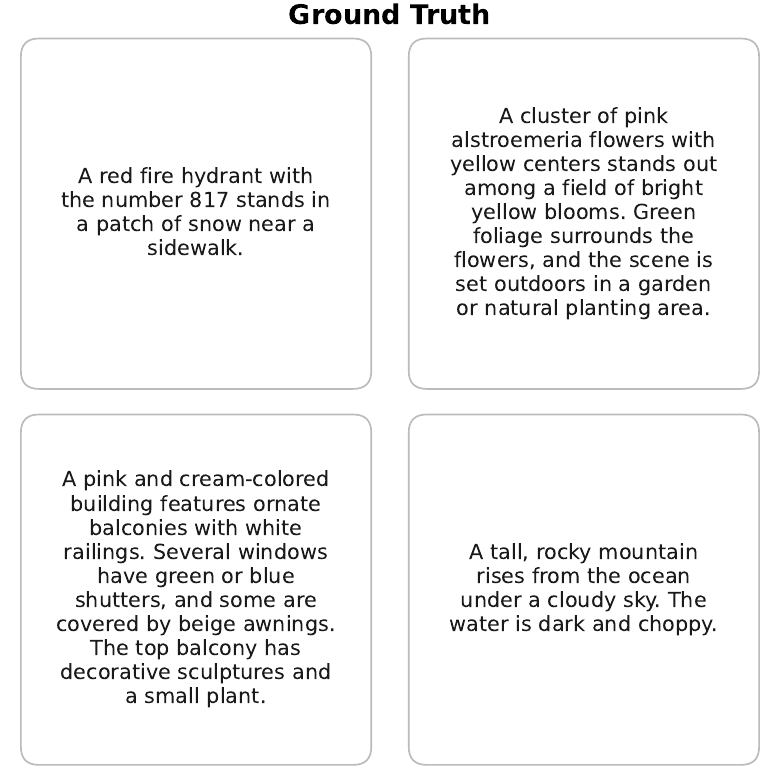}
    \caption{\textbf{Text-to-Image}: Images generated by forward SDE (Eq.~\ref{eqn:data_gen}), captions come from dataset.}\label{fig:t2i_results}
\end{subfigure}
\hfill
\begin{subfigure}[]{0.49\linewidth}
\centering
    \includegraphics[width=0.49\linewidth]{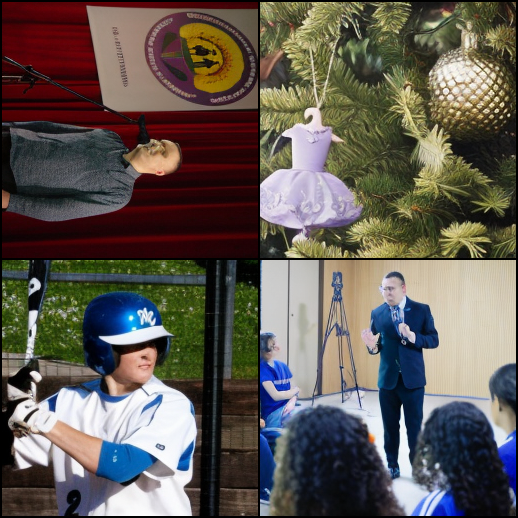}
    \includegraphics[width=0.49\linewidth]{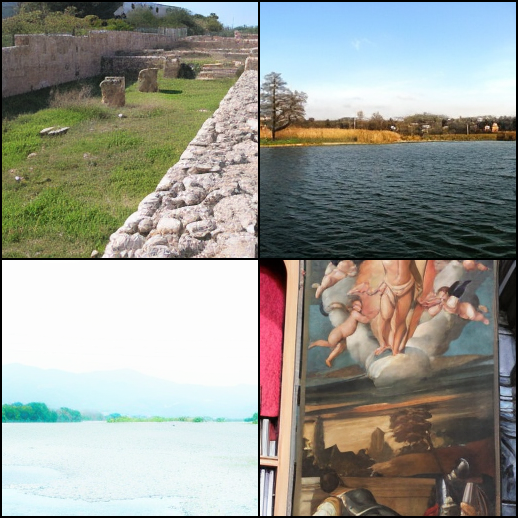}
    \includegraphics[width=0.49\linewidth, trim={0 0 0 0.06\linewidth},clip]{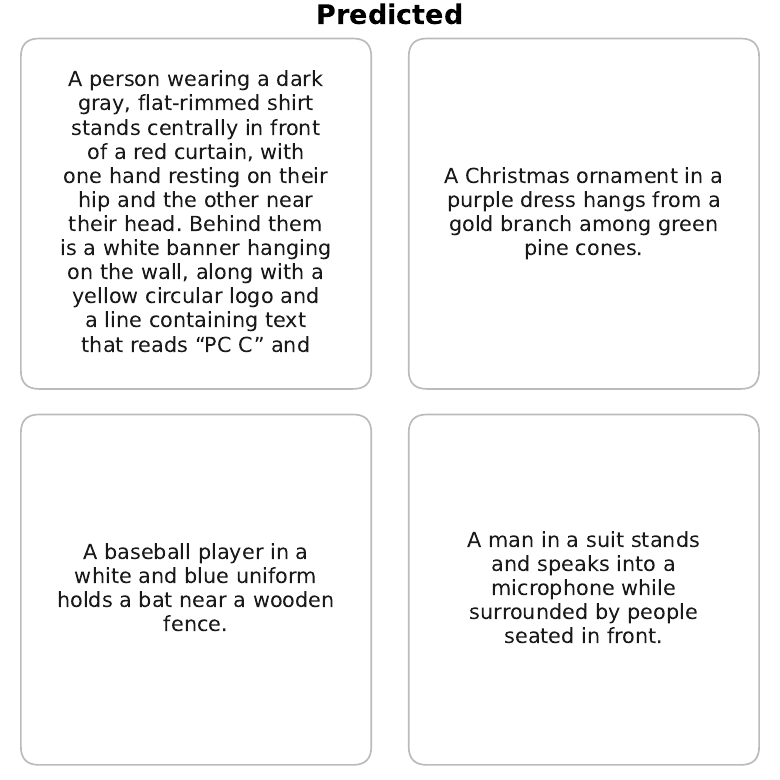}
    \includegraphics[width=0.49\linewidth, trim={0 0 0 0.06\linewidth},clip]{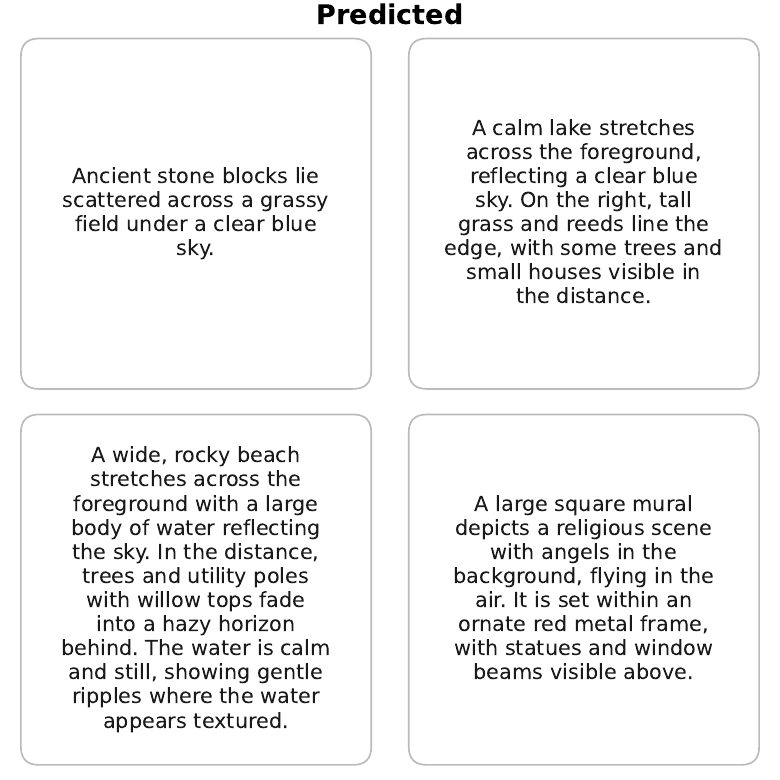}
    \caption{\textbf{Image-to-Text}: Captions generated by reverse SDE (Eq.~\ref{eqn:reverse}), images come from dataset.}\label{fig:i2t_results}
\end{subfigure}
\caption{\textbf{Sample Generations:} Using our DiT-XL/2 model (1,054,621,840 parameters) trained for 122,500 iterations at batch size 1792. CFG scale $\omega=0.5$, 500 SDE steps. %\textbf{\textit{See \texttt{\href{https://bit-diffusion.github.io}{https://bit-diffusion.github.io}} for demo videos.}}
}
\end{figure}

\subsection{Conditional Text-to-Image and Image-to-Text Generation}\label{sec:gen_results}

\noindent\textbf{Goals} 
We benchmark generation quality and prompt adherence in both T2I and I2T directions.

\noindent\textbf{Results}
See Tab.~\ref{tab:gen_results}. \bit{} performs favorably compared to other methods. See Sec.~\ref{subsec:gen_metrics} for metrics.

\noindent\textbf{\textit{Effect of Data-to-Data:}} 
For text-to-image generation, \bit{} is very close to noise-to-data diffusion's performance on FID and CLIP. On image-to-text, \bit{} is by far the best on generative perplexity, and second-best on CLIP. 
The performance advantage of noise-to-data diffusion is small but statistically significant at a confidence level equal to 95\%.
\begin{wraptable}[10]{r}{0.65\textwidth}
% \begin{table}[h]
    \centering
    \resizebox{\linewidth}{!}{
        \begin{tabular}{lrrrr}
            \toprule
            Method & FID ($\downarrow$) & CLIP T2I ($\uparrow$) & Gen PPL ($\downarrow$) & CLIP I2T ($\uparrow$) \\
            \midrule
            \textbf{\bit{}} & 6.23 & \textit{\textbf{27.10}} & \textbf{123.4} & \textit{\textbf{26.99}} \\
            Diffusion (ablate data end) & \textbf{5.73} & \textbf{27.13} & 178.2 & 26.42 \\
            Flow (ablate stochasticity) & 279.13 & 20.13 & 169.2 & \textbf{27.25} \\
            \bottomrule
        \end{tabular}
    }
    \caption{Quality comparison between \bit{}, noise-to-data diffusion \citep{song2020score}, and cross-modality flow matching \citep{albergo2025stochastic}. We use 500 SDE steps, with $\omega=0$. T2I measured at 50k samples, I2T measured at 10k samples.}
    \label{tab:gen_results}
\end{wraptable}
This suggests that the Gaussian source used in diffusion models is not inherently special.
% The marginal per-pair CLIP-score standard deviations for \bit{} and noise-to-data diffusion were approximately $3.3$. \textcolor{red}{\textbf{Author note:} Because both methods are evaluated on the same prompts, a hypothesis test for their mean CLIP difference requires a variance estimator for the paired per-prompt differences (or an equivalent paired bootstrap). Compute this from the per-prompt scores before reporting a p-value.}

\noindent\textbf{\textit{Effect of Stochasticity:}}
Data-to-data flow matching fares the worst on every metric besides image-to-text CLIP, in line with our expectations: multimodality translation is inherently a stochastic task, so a deterministic ODE model is unsuitable.
Without auxiliary randomness, a deterministic source-to-target flow cannot represent non-degenerate conditional distributions; this expressivity limitation helps explain the orders-of-magnitude worse text-to-image FID observed for the flow baseline.
%\textit{it is an intrinsic consequence of the marginal velocity field that averages all the possible targets from the multi-peak distribution into a blur.} 

We interpret this not as a SoTA competitor that \bit{} defeats, but as a measurement that isolates one design axis. FlowTok~\citep{he2025flowtok} and CrossFlow~\citep{liu2025flowing} report strong text-to-image results using data-to-data couplings, but they additionally learn variational encoders with auxiliary contrastive objectives. Because the variational encoder injects noise, these components restore stochasticity to an otherwise deterministic transport and address the conditional-diversity limitation of a bare deterministic map. \bit{} instead obtains stochasticity natively from the bridge SDE. \textit{Our controlled comparison of the ODE and SDE baselines shows that the bare deterministic transport is unsuitable in this multimodal setting.}
That being said, its good performance on image-to-text CLIP indicates that captioning (at least within the GPIC distribution) may be a fairly low entropy task given image conditioning. %; this is also attributable to the sharing of REPA regularization~\citep{yu2024representation} with the image generation direction, since they share the same velocity field. 
However, its generative perplexity is the worst, indicating it has not mastered the syntactic nuances of text.

\subsection{Cross-Modal Round-Trip Stochastic Variation}\label{sec:cross_modal_variation}
\noindent\textbf{Goals} An important task in AI-aided art and design is to create variations of input data; \textit{e.g.}, Midjourney can generate image variations~\citep{midjourney_variation}, while ChatGPT can paraphrase documents.
\begin{wrapfigure}[25]{r}{0.65\linewidth}
\centering
    \includegraphics[width=\linewidth]{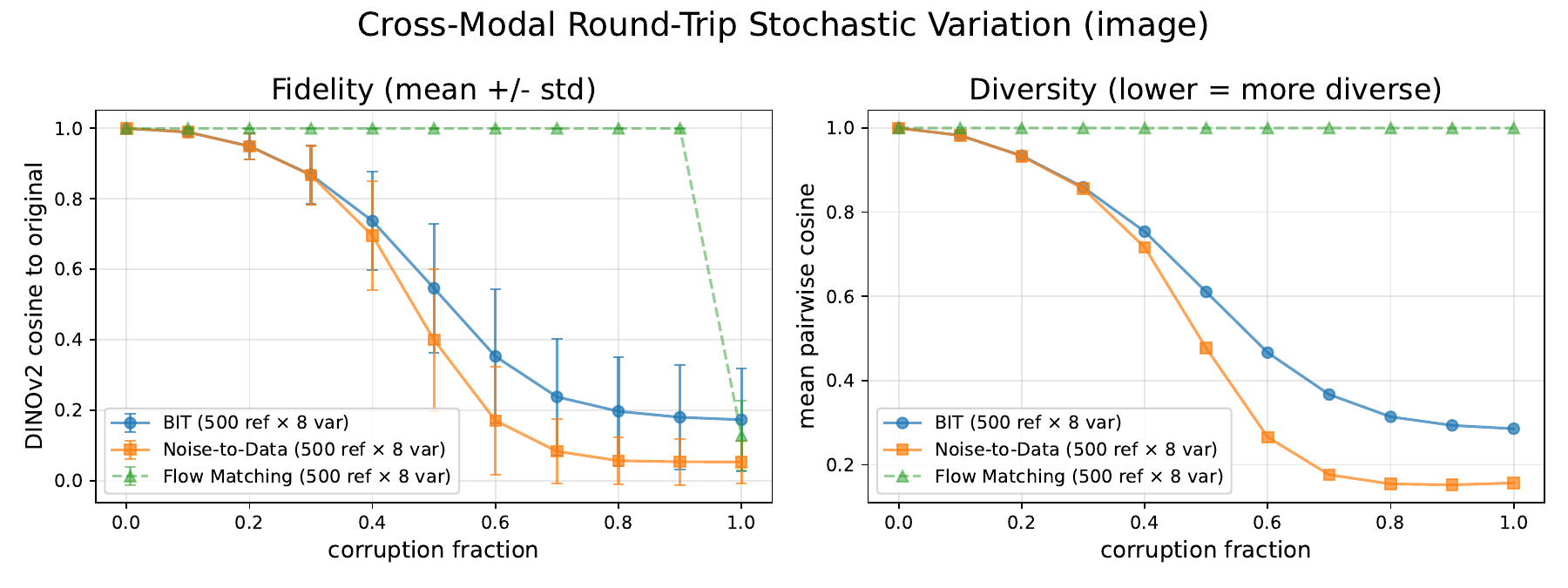}%
    \\
    \includegraphics[width=\linewidth]{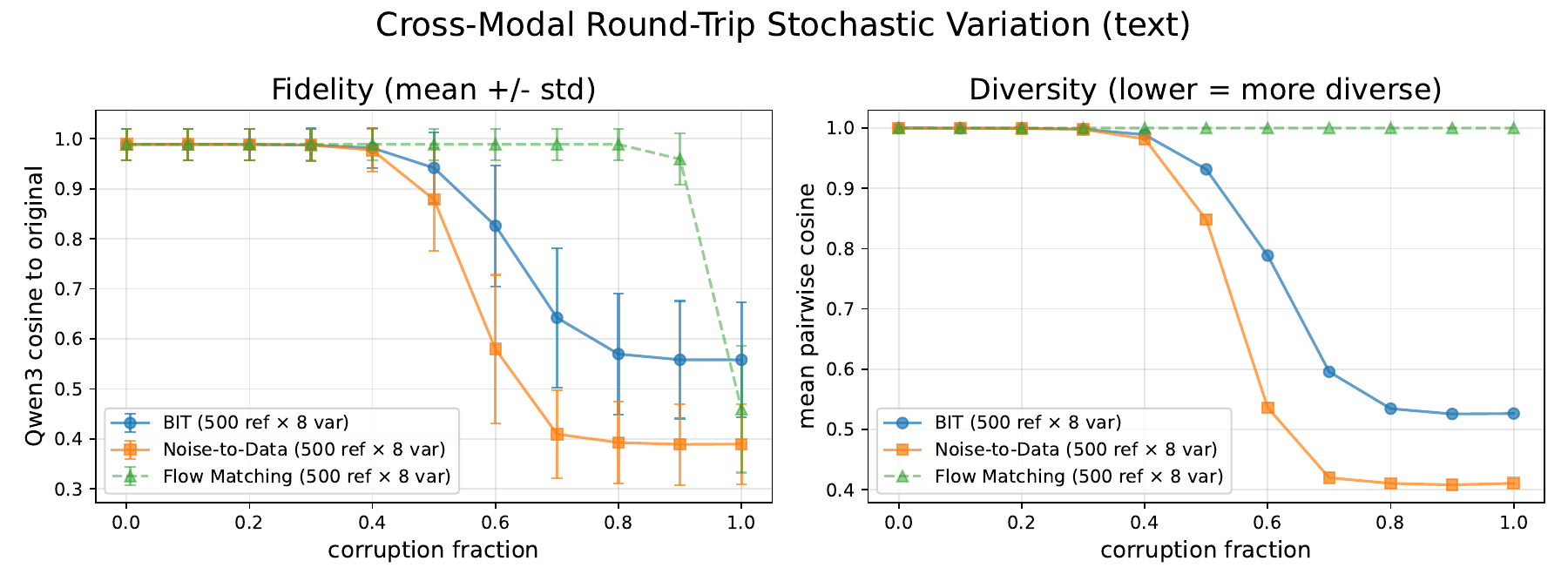}%
    \caption{\textbf{Cross-Modal Round-Trip Stochastic Variation:} We want to generate semantically related variants of a given data sample.%, \textit{e.g.}, paraphrase sentences, drawing dogs in different poses. 
    We run Alg. \ref{alg:roundtrip-image-editing} with $dt=0.004$ with Euler-Maruyama simulation. We compare \bit{} to denoising diffusion and data-to-data flow matching. App.~\ref{sec:qual_results_variation} shows qualitative examples.}\label{fig:variation}
\end{wrapfigure}
With SDE-based (diffusion) generative models, a popular method to create data variations is to run a forward ``noising'' process, then restore the data through a time-reversed score-driven process, as popularized by SDEdit~\citep{meng2021sdedit}. The forward process removes fine details while retaining some high-level structure; time-reversed ``denoising'' then stochastically restores details consistent with that structure.

With noise-to-data diffusion models, increasing corruption progressively removes information about the original data. Thm.~\ref{thm:mutual_info} motivates testing whether \bit{} retains useful source information for longer because its two data endpoints are dependent; the theorem itself is local near the source endpoint, so performance at high corruption levels is an empirical question.

\noindent\textbf{Procedure} We run Alg.~\ref{alg:roundtrip-image-editing} on ground-truth samples at varying corruption levels. The marginal-equivalence corollaries motivate aggregate target calibration but don't guarantee source-conditioned reconstruction; semantic round-trip fidelity is evaluated empirically here.
%\noindent\textbf{Metrics} 
See Sec.~\ref{subsec:cross_modal_round_trip_metrics} for metrics.

\noindent\textbf{Results} See Figures~\ref{fig:variation}, \ref{fig:variation_qual}, and~\ref{fig:variation_text_qual}. 
Consistent with the information-retention motivation of Theorem~\ref{thm:mutual_info}, \bit{}'s fidelity curve degrades more gracefully than that of the noise-to-data baseline at higher corruption fractions in both directions. Greater corruption also produces more diverse restorations. Noise-to-data diffusion exhibits greater diversity because its noise endpoint discards all source-state information, whereas \bit{} trades off fidelity and diversity.

% \subsection{Bridge-Based Image Editing}
% See Sec.~\ref{sec:bridge_editing} for results, Sec.~\ref{app:bridge_editing} for setup.

% \input{experiments/bridge_based_editing}

\subsection{Miscellaneous Comparisons}\label{sec:miscellaneous_comparisons}

This subsection uses the billion-parameter model (trained on about 2x as many effective examples, with one network for both directions) from Sec.~\ref{app:large_scale_hparams}.

\noindent\textbf{Impact of CFG:} See Table~\ref{tab:cfg}. The CFG from Section~\ref{sec:cfg} improves model performance for both text and image generation. 

\noindent\textbf{Impact of Scale and Engineering:} Performance gains are possible with scale and engineering, as even $\omega=0$ in Tab.~\ref{tab:cfg} exceeds all results in Tab.~\ref{tab:gen_results}. This gives hope for industrial-scale \bit{} versions.

\noindent\textbf{Anchors Against External Baselines:} The previous subsections provide controlled scientific comparisons. Table~\ref{tab:cfg} additionally reports external reference points, but these are not like-for-like comparisons: the models differ in training data, compute, architecture, and inference settings. 
\begin{wraptable}{r}{0.65\linewidth}
    \centering
    \resizebox{\linewidth}{!}{
    \begin{tabular}{lcccc}
        \toprule
        Method & FID ($\downarrow$) & CLIP T2I ($\uparrow$) & Gen PPL ($\downarrow$) & CLIP I2T ($\uparrow$) \\
        \midrule
        \bit{} \textbf{$\omega=0.0$} & 4.39 & 27.75 & 78.1 & 29.4 \\
        \bit{} \textbf{$\omega=0.5$} & \textbf{2.54} & \textbf{28.80} & \textbf{61.1} & \textbf{29.5} \\
        \midrule 
        Qwen-3-VL-4B-Instruct & n/a & n/a & 37.7 & 31.9 \\
        Stable Diffusion 1.5 & 9.92 & 33.15 & n/a & n/a \\
        \bottomrule
    \end{tabular}
    }
    \caption{\textbf{Impact of CFG, and Comparison to Foundation Models.} We use 500 SDE steps.}
    \label{tab:cfg}
\end{wraptable}
For I2T, \bit{} remains behind Qwen-3-VL-4B-Instruct~\citep{bai2025qwen3}, the model used to produce the GPIC reference captions, on both reported metrics. For T2I, \bit{} has a better FID than Stable Diffusion 1.5~\citep{rombach2022high} on the GPIC held-out distribution but a lower CLIP score. 
% Because \bit{} was trained on GPIC whereas Stable Diffusion 1.5 was not, the favorable FID partly reflects in-distribution matching and should not be interpreted as an overall quality win.

\subsection{Scientific Domain: Cell Fate Modeling}

\noindent\textbf{Goals} Cell fate modeling studies how an early cellular state gives rise to later, differentiated outcomes. We evaluate \bit's ability to learn a conditional generative map between early progenitor cells and their later descendants in the \textsc{Larry} lineage-tracing dataset~\citep{larry}. The model should generate realistic day-6 cell states from day-2 progenitors while preserving the fate and clone structure implied by lineage barcodes. So, a strong model must do more than match the marginal distribution of later cells: it must capture branching differentiation, place generated cells on the correct fate-specific trajectories, and support meaningful forward and reverse transport between progenitor and descendant states.
See Section~\ref{sec:larry_details} for metrics.
%\noindent\textbf{Metrics} See Section~\ref{sec:larry_details}.

\noindent\textbf{Results} 
Table~\ref{tab:larry_state_fate} reports final test-set metrics. The cosine-volatility bridge has the best average rank (2.67), while different methods lead individual metrics. In particular, the cosine-volatility bridge obtains the highest fate kNN accuracy (0.445), and the separate-network uniform-volatility bridge has the best clone kNN and distributional MMD. These single-run ranks are descriptive; without uncertainty estimates, they do not establish statistical superiority. %and improves cycle consistency relative to the uniform bridge, while maintaining similar forward MSE and MMD.
%This suggests that the choice of bridge noise schedule matters: the cosine bridge preserves state-fate structure more reliably than the uniform schedule in this setting.
While endpoint regression obtains the lowest pointwise MSE, it substantially degrades distributional and biological structure, with the worst fate kNN accuracy and MMD among the methods. %Thus, direct endpoint prediction can recover paired latent coordinates without faithfully modeling the differentiated cell population. 
Rectified flow gives competitive MMD and cycle error, but has much worse forward and reverse MSE, indicating that its generated states match some marginal structure while failing to preserve the paired state-fate map. 

\begin{table}[t]
\centering
\small
\resizebox{\columnwidth}{!}{
\begin{tabular}{lcccccc c}
\toprule
Method & Fwd MSE $\downarrow$ & Rev MSE $\downarrow$ & Fwd MMD $\downarrow$
& Fate kNN $\uparrow$ & Clone kNN $\uparrow$ & Cycle MSE $\downarrow$ & Avg.\ Rank $\downarrow$ \\
\midrule
\textbf{\bit{}, uniform shared} & 1.388 & 1.365 & 0.0446 & 0.404 & 0.182 & 1.307 / 1.418 & 3.33 \\
\textbf{\bit{}, uniform separate} & 1.415 & 1.394 & \textbf{0.0352} & 0.377 & \textbf{0.209} & 1.230 / 1.217 & 3.42 \\
\textbf{\bit{}, cosine shared} & 1.388 & 1.262 & 0.0431 & \textbf{0.445} & 0.160 & 1.042 / 1.136 & \textbf{2.67} \\
Noise-to-data diffusion & 1.329 & 1.381 & 0.0517 & 0.377 & 0.176 & 1.270 / 1.474 & 4.00 \\
Rectified flow & 2.444 & 1.694 & 0.0406 & 0.398 & 0.141 & \textbf{0.621 / 1.020} & 3.83 \\
Endpoint regression & \textbf{1.195} & \textbf{1.176} & 0.0847 & 0.336 & 0.047 & 1.076 / 1.102 & 3.75 \\
\bottomrule
\end{tabular}
}
\caption{State-fate bridge results on LARRY day-2 to day-6 clone-paired endpoints.
Average rank is computed per metric column (ties broken by midranks) and averaged over the six metrics; Cycle MSE sub-columns ranked separately and averaged so cycle consistency contributes a single metric.
%Overall, \bit{} models perform the best.
}
\label{tab:larry_state_fate}
\end{table}
  
% \begin{figure}[t]
%   \centering
%   \includegraphics[width=\linewidth]{trajectory.pdf}
%   \caption{
%   Qualitative state-fate transport learned by \bit{} on the LARRY day-2 to day-6 benchmark.
%   Points show a shared PCA projection of day-2 progenitors, \bit{}-generated day-6 descendants, and real day-6 descendants, colored by fate label.
%   The generated day-6 distribution recovers the main fate-structured regions of the real descendant population while preserving the branching organization from early progenitor states.
%   }
%   \label{fig:larry_bit_cosine_bridge}
% \end{figure}

\subsection{Additional Results}
See Sec.~\ref{sec:bridge_editing} for qualitative image-editing results, Sec.~\ref{sec:repa} for ablation on REPA.

\section{Related Works, Discussion, Conclusion}
Our work builds on advances in diffusion bridges~\citep{guo2026abc, zhou2023denoising, kieu2025bidirectional}, data-to-data translation models~\citep{lee2026flow, liu2025flowing}, and continuous-time and space language generation~\citep{hu2026elf, lee2026flow, potaptchik2026discrete}. \noindent\textbf{\textit{See Section~\ref{sec:related_works} for more related work.}} 
In future work, we hope to explore translation across many domains, \textit{e.g.}, text-image-audio. We also want to conduct a systematic study of how endpoint geometry affects bridge performance.
In this work, we introduced \bit{}: \textbf{B}idirectional \textbf{I}mage-\textbf{T}ext Diffusion Bridges. \bit{} challenges the traditional noise-to-data paradigm for T2I generation: in our controlled experiments, starting from a text representation and interpolating toward images is competitive with the noise-to-data baseline and performs better on several downstream tasks. The endpoint-conditioned construction supports generation in both directions, while its source-aware path enables flexible round-trip sampling within a unified T2I and I2T framework.

\newpage
\clearpage

\subsection*{AI use statement}

In this work, we used generative AI tools for helping to: write mathematical proofs, write code, and draft some descriptions of experimental details. We have manually reviewed all AI-assisted work (proofs, code, writing), and affirm to the best of our knowledge that it is correct.

\subsubsection*{Acknowledgments}
\textcolor{red}{This material is based upon work supported by the U.S. Department of Energy, Office of Science, Office of Advanced Scientific Computing Research, Department of Energy Computational Science Graduate Fellowship under Award Number DE-SC0025528. This research used resources of the National Energy Research Scientific Computing Center (NERSC), a Department of Energy User Facility (projects m5319-2026, m1266-2026). J. Blanchet gratefully acknowledges support from ONR under award N00014-24-1-2655, and the National Science Foundation (NSF) under grants 2312204 and 2403007.}

\textcolor{red}{We thank Yangyi Shen, Lutong Hao, and Miguel Liu-Schiaffini for helpful discussion.}

\bibliography{iclr2027_conference}
\bibliographystyle{iclr2027_conference}

\newpage
\clearpage

\appendix

\section{Conditional Stochastic Process Derivation}

This section provides a self-contained derivation of the theoretical results underlying our algorithmic approach, prioritizing clarity and accessibility over maximum generality. Throughout this section, we denote the population data-generating law as $\Pi \in \Delta((\mathbb{R}^d)^2)$. %in order to fully distinguish it from its probability distribution function.

Under the base measure $\mathbb{P}$, the process $X_{\cdot}$ defined in Equation~\ref{eqn:base_process} is conditionally Gaussian: $(X_t)_{t \in [0,1]} \vert X_0$ is a Gaussian process. Therefore, conditional on both endpoints, $(X_t)_{t \in [0,1]} \vert X_0, X_1$ is also a Gaussian process.

This conditional Gaussian process is a bridge pinned at both endpoints. We denote the law of the Gaussian bridge with endpoints $(x_0, x_1) \in (\mathbb{R}^d)^2$ by $\mathbb{P}^{(x_0,x_1)}$. Therefore,
\begin{align}
    \mathbb{Q}(A) = \int_{(x_0, x_1) \in (\mathbb{R}^d)^2} \mathbb{P}^{(x_0,x_1)}(A) \Pi(dx_0, dx_1)\label{eqn:q_measure_full}
\end{align}
for any measurable set $A$.
Write $\Pi_0$ and $\Pi_1$ for the first and second marginals of $\Pi$.

\begin{assumption}
\label{assumption:1}
    Assume that $0 < \underline{\sigma} \le \sigma(t) \le \bar{\sigma}$ for all $t \in [0,1]$, that $\sigma$ is continuous, and that the endpoint coupling has finite second moments:
    \[
        \mathbb{E}_{\Pi}\!\left[\lVert X_0\rVert_2^2+\lVert X_1\rVert_2^2\right]<\infty.
    \]
\end{assumption}

Under this assumption, $\sigma$ is square-integrable. Moreover, we define
\begin{align*}
    a(t) := \int_{s=0}^t \sigma(s)^2 ds,
\end{align*}
so $a:[0,1] \to [0, \infty)$ is increasing and Lipschitz continuous. Because $\sigma^2$ is continuous, the fundamental theorem of calculus also gives $a\in C^1([0,1])$ with $a'(t)=\sigma(t)^2$.

\paragraph{Gaussian/Brownian Bridge Construction}

Regarding the proofs to follow, it is useful to anchor them in the context of the well-known Brownian bridge machinery~\citep{oksendal2003stochastic}.

It is known that~\citep{doob1984classical}, by defining $(X^{(x_0, x_1)}_t)_{t\in [0,1)}$ via the SDE
\begin{align*}
    \begin{cases}
        X^{(x_0, x_1)}_0 &= x_0\\
        dX^{(x_0, x_1)}_t &= \frac{\sigma(t)^2}{a(1)-a(t)}[x_1 - X^{(x_0, x_1)}_t]dt + \sigma(t) dB_t
    \end{cases},
\end{align*}
we will have that a strong solution exists on $[0,1)$ with a unique continuous extension throughout $[0,1]$; such solution has law $\mathbb{P}^{(x_0, x_1)}$.

Now fix only the initial endpoint $x_0$ and let $X_1$ have a distribution $F_1 \in \Delta(\mathbb{R}^d)$ with finite second moments. On the half-open interval $[0,1)$, consider the SDE
\begin{align}
\label{SDE:x0F1}
    \begin{cases}
        X^{(x_0, F_1)}_0 &= x_0\\
        dX^{(x_0, F_1)}_t &= 
        \frac{\sigma(t)^2}{a(1)-a(t)}
        [m(t, X^{(x_0, F_1)}_t; x_0, F_1) - X^{(x_0, F_1)}_t]dt + \sigma(t) dB_t
    \end{cases},
\end{align}
where the equation is understood on every finite horizon $[0,T]$, $T<1$ (the value of the drift at the single time $t=0$ may be chosen arbitrarily), and where, for $0<t<1$, the tilted mean function is
\begin{align*}
    m(t, x_t; x_0, F_1) :=
    \frac{\int_{y \in \mathbb{R}^d} y
    e^{-\frac{\left\Vert (y-x_0) - (x_t-x_0)\frac{a(1)}{a(t)}\right\Vert_2^2}{2 \frac{a(1)(a(1)-a(t))}{a(t)}}}
    F_1(dy)}{\int_{y \in \mathbb{R}^d}
    e^{-\frac{\left\Vert (y-x_0) - (x_t-x_0)\frac{a(1)}{a(t)}\right\Vert_2^2}{2 \frac{a(1)(a(1)-a(t))}{a(t)}}}
    F_1(dy)}.
\end{align*}
The mixture measure
\begin{align*}
    \mathbb P^{(x_0,F_1)}
    :=\int_{\mathbb R^d}\mathbb P^{(x_0,y)}F_1(dy)
\end{align*}
will provide a canonical weak solution law of this SDE. Its endpoint is supplied by the continuous extension, not by evaluating the singular drift at $t=1$.

A heuristic way to motivate the correctness of this is that the tilted mean (which was written as a weighted average over Gaussian PDFs) is indeed
\begin{align}
    m(t, x_t; x_0,F_1) &= \mathbb{E}_{(X_r)_{r\in [0,1]} \sim \mathbb{P}^{(x_0, F_1)}}\left[X_1 \vert X_t = x_t\right]\\
    &= \mathbb{E}_{(X_r)_{r\in [0,1]} \sim \mathbb{P}^{(x_0, F_1)}}\left[X_1 \vert X_t = x_t, (X_s)_{s=0}^t\right],
\end{align}
so $m(t, X^{(x_0, F_1)}_t;x_0,F_1)$ is the mean of the endpoint $X^{(x_0, F_1)}_1$ given the Brownian filtration at the current time $\mathcal{F}_t$. To see why, we note (and will later show rigorously) that $X^{(x_0, x_1)}_t$ is distributed according to the pdf
\begin{align*}
    p(x_t; x_0, x_1) &=
    \left(2\pi \frac{a(t)(a(1)-a(t))}{a(1)}\right)^{-\frac{d}{2}} e^{-\frac{\left\Vert (x_t-x_0) - (x_1-x_0)\frac{a(t)}{a(1)}\right\Vert_2^2}{2 \frac{a(t)(a(1)-a(t))}{a(1)}}}\\
    &=
    \left(2\pi \frac{a(t)(a(1)-a(t))}{a(1)}\right)^{-\frac{d}{2}} e^{-\frac{\left\Vert (x_1-x_0) - (x_t-x_0)\frac{a(1)}{a(t)}\right\Vert_2^2}{2 \frac{a(1)(a(1)-a(t))}{a(t)}}}.
\end{align*}

\begin{lemma}
\label{lemma:mixture of bridges}
For any $x_0\in\mathbb R^d$, under $\mathbb P^{(x_0,F_1)}$ the coordinate process is a weak solution of~\ref{SDE:x0F1} on every $[0,T]$, $T<1$. It has an almost-sure continuous extension to $t=1$, with $X_1\sim F_1$. Thus $\mathbb P^{(x_0,F_1)}$ is the canonical weak solution law associated with~\ref{SDE:x0F1}; no uniqueness assertion is made for arbitrary pointwise versions of its drift.
\end{lemma}

Intuitively, we have Lemma~\ref{lemma:mixture of bridges} to rigorously show that the SDE described in Equation~\ref{SDE:x0F1} does indeed have the correct data-generating path measure.

\begin{proof}
    Denote $\mathbb{P}^{(x_0)}$ as a path measure of $X_t := x_0 + \int_{0}^{t}\sigma(s)dB_s$, so
    \begin{align*}
        \mathbb{P}^{(x_0, x_1)}(\cdot) = \mathbb{P}^{(x_0)}(\cdot \vert X_1=x_1).
    \end{align*}

    The distribution of $X_t$ under $\mathbb{P}^{(x_0)}$ is $N(x_0, a(t)I)$ for any $t \in [0,1]$, and the distribution of $X_t \vert X_s$ where $0 \le s\le t \le 1$ is $N(X_s, (a(t)-a(s))I)$.

    This means that, with $\lambda$ being the Lebesgue measure, for any $t \in [0,1)$,
    \begin{align}
        \frac{d\mathbb{P}^{(x_0)}(X_1=\cdot \vert X_t)}{d\mathbb{P}^{(x_0)}(X_1=\cdot)} (x_1)
        &=
        \frac{
        \frac{d\mathbb{P}^{(x_0)}(X_1=\cdot \vert X_t)}{d\lambda(\cdot)} (x_1)
        }
        {
        \frac{d\mathbb{P}^{(x_0)}(X_1=\cdot)}{d\lambda(\cdot)} (x_1)
        }\\
        &= 
        \frac{
        (2\pi (a(1)-a(t)))^{-\frac{d}{2}} e^{- \frac{\Vert x_1 - X_t\Vert_2^2}{2(a(1)-a(t))}}
        }
        {
        (2\pi a(1))^{-\frac{d}{2}} e^{- \frac{\Vert x_1 - x_0\Vert_2^2}{2a(1)}}
        }\\
        &= 
        \left(\frac{a(1)}{a(1)-a(t)}\right)^{\frac{d}{2}}
        e^{\frac{\Vert x_1 - x_0\Vert_2^2}{2a(1)}- \frac{\Vert x_1 - X_t\Vert_2^2}{2(a(1)-a(t))}}\label{eqn:change_of_measure_form}
    \end{align}
    almost everywhere $x_1$ and almost surely $X_t$. 

    Thus, we define, for any $x_1\in \mathbb{R}^d$, a random process
    \begin{align}
        Z^{(x_0,x_1)}_t := \left(\frac{a(1)}{a(1)-a(t)}\right)^{\frac{d}{2}}
        e^{\frac{\Vert x_1 - x_0\Vert_2^2}{2a(1)}- \frac{\Vert x_1 - X_t\Vert_2^2}{2(a(1)-a(t))}}.\label{eqn:process_equal_change_of_measure}
    \end{align}
    Crucially, Equation~\ref{eqn:process_equal_change_of_measure} has the same expression as Equation~\ref{eqn:change_of_measure_form}. We defined Equation~\ref{eqn:process_equal_change_of_measure} because Equation~\ref{eqn:change_of_measure_form} is technically not defined everywhere (although its undefined set has measure zero), which makes operations such as integration ill-defined without choosing a version. In contrast, Equation~\ref{eqn:process_equal_change_of_measure} is defined everywhere and provides a convenient version of this fixed-time Bayes density.

    Now, we want to show that this random process (which corresponds to the change-of-measure) is a martingale. For $0 \le s<t < 1$, we have that
    \begin{align*}
        \mathbb{E}^{\mathbb{P}^{(x_0)}}\left[
        Z^{(x_0,x_1)}_t \middle\vert \mathcal{F}_s
        \right]
        &=
        \frac{
            \mathbb{E}^{\mathbb{P}^{(x_0)}}\left[
            (2\pi (a(1)-a(t)))^{-\frac{d}{2}} e^{- \frac{\Vert x_1 - X_t\Vert_2^2}{2(a(1)-a(t))}} \middle\vert \mathcal{F}_s
            \right]
        }
        {
        (2\pi a(1))^{-\frac{d}{2}} e^{- \frac{\Vert x_1 - x_0\Vert_2^2}{2a(1)}}
        }\\
        &=
        \frac{
            \mathbb{E}_{Z \sim N(0,I)}\left[
            (2\pi (a(1)-a(t)))^{-\frac{d}{2}} e^{- \frac{\Vert x_1 - X_s+\sqrt{a(t)-a(s)}Z\Vert_2^2}{2(a(1)-a(t))}}
            \middle\vert X_s
            \right]
        }
        {
        (2\pi a(1))^{-\frac{d}{2}} e^{- \frac{\Vert x_1 - x_0\Vert_2^2}{2a(1)}}
        }\\
        &=
        \frac{
            (2\pi (a(1)-a(s)))^{-\frac{d}{2}} e^{- \frac{\Vert x_1 - X_s\Vert_2^2}{2(a(1)-a(s))}}
        }
        {
        (2\pi a(1))^{-\frac{d}{2}} e^{- \frac{\Vert x_1 - x_0\Vert_2^2}{2a(1)}}
        }\\
        &=
        Z^{(x_0,x_1)}_s.
    \end{align*}
    The cancellation of the auxiliary random variable $Z$ follows from the convolution of two Gaussian densities.
    
    The displayed conditional-expectation identity proves that $\left(Z^{(x_0,x_1)}_t\right)_{t \in [0,1)}$ is a martingale with respect to $\mathbb{P}^{(x_0)}$; its positivity ensures integrability once $\mathbb{E}[Z_t]=Z_0=1$ is established.

    We now work from the already defined bridge mixture rather than trying to
    extend a finite-horizon density to the pinned terminal sigma-field. Define
    the harmonic function~\citep{doob1984classical}, for all
    $t\in(0,1)$ and $x\in\mathbb R^d$, by
    \begin{align*}
        h(t, x) := \int_{y \in \mathbb{R}^d}
        \left(\frac{a(1)}{a(1)-a(t)}\right)^{\frac{d}{2}}
        e^{\frac{\Vert y - x_0\Vert_2^2}{2a(1)}- \frac{\Vert y - x\Vert_2^2}{2(a(1)-a(t))}}
        F_1(dy).
    \end{align*}
    This is the $F_1$-mixture of the martingales
    $Z_t^{(x_0,y)}$. For each compact subinterval of $(0,1)$, the negative
    quadratic term provides an integrable bound for the kernel and its required
    derivatives. Hence Tonelli's theorem and differentiation under the integral
    give $h>0$, $h\in C^{1,2}$, and the backward harmonic equation.

    Fix $T<1$. By the finite-horizon Bayes identity above, for every bounded
    $\mathcal F_T$-measurable random variable $G$,
    \begin{align*}
    \mathbb E^{\mathbb P^{(x_0,F_1)}}[G]
    &=\int_{\mathbb R^d}\mathbb E^{\mathbb P^{(x_0,y)}}[G]F_1(dy)\\
    &=\mathbb E^{\mathbb P^{(x_0)}}\!\left[
        G\int_{\mathbb R^d}Z_T^{(x_0,y)}F_1(dy)\right]\\
    &=\mathbb E^{\mathbb P^{(x_0)}}[G h(T,X_T)].
    \end{align*}
    Consequently,
    \begin{align*}
       \left.\frac{d\mathbb P^{(x_0,F_1)}}{d\mathbb P^{(x_0)}}
       \right|_{\mathcal F_T}=h(T,X_T),
    \end{align*}
    a statement made only for $T<1$. In particular, this calculation neither
    asserts nor requires absolute continuity on $\mathcal F_1$; for singular
    $F_1$ such terminal absolute continuity is false, and the density martingale
    need not be uniformly integrable.

    The finite-horizon $h$-transform/Girsanov theorem~\citep{girsanov1960transforming}
    now shows that, for every $T<1$,
    \begin{align*}
        B_t^{(x_0,F_1)}
        :=B_t-\int_0^t\sigma(s)\nabla_x\log h(s,X_s)\,ds,
        \qquad 0\le t\le T,
    \end{align*}
    is Brownian under $\mathbb P^{(x_0,F_1)}$. These definitions are consistent
    as $T$ varies and therefore define one Brownian motion on $[0,1)$.

    Direct differentiation of $h$ yields
    $\nabla_x\log h(t,x)=[m(t,x;x_0,F_1)-x]/[a(1)-a(t)]$.
    Thus, on every finite horizon,
    \begin{align*}
        dX_t
        &=\frac{\sigma(t)^2}{a(1)-a(t)}
        [m(t,X_t;x_0,F_1)-X_t]dt
        +\sigma(t)dB_t^{(x_0,F_1)}.
    \end{align*}
    Finally, the continuous extension and its terminal law $F_1$ hold by the
    definition of $\mathbb P^{(x_0,F_1)}$ as a mixture of continuously extended
    pinned bridge laws. This proves the claimed canonical weak-solution result.
\end{proof}

Next, we will prove the main theorem (Theorem~\ref{thm:forward_process}).

\subsection{Proof of Forward Process SDE and Objective}\label{sec:forward_process_proof}

Before proving the theorem, we first provide the regularity assumption we impose on the weight function $w: (0,1) \to (0, \infty)$.

\begin{assumption}
\label{assumption:weight}
    There exist constants $C_0,C_1,\beta_0,\beta_1>0$ such that
    \begin{align*}
        w(t) \le \min\!\left(C_0t^{\beta_0}, C_1(1-t)^{\beta_1}\right),
    \end{align*}
    and $w(\cdot)$ is continuous.
\end{assumption}
This assumption controls the contribution of the sample close to the two endpoints. Moreover, this can also be incorporated as importance sampling as well as other variance reduction methods.

Under this assumption, we can prove the following lemma.
\begin{lemma}
\label{lemma:2}
    Under the assumptions~\ref{assumption:1},\ref{assumption:weight}, we have that
    \begin{align*}
        \int_{t=0}^1 \frac{1}{a(t)(a(1)-a(t))} w(t)dt < \infty.
    \end{align*}
\end{lemma}
% \begin{proof}
%     From continuity, it suffices to show that, for some $\epsilon > 0$,
%     $\int_{t=0}^{\epsilon} \frac{1}{a(t)(a(1)-a(t))}w(t)dt + \int_{t=1-\epsilon}^{1} \frac{1}{a(t)(a(1)-a(t))}w(t)dt < \infty$.

%     Since $\sigma(t) \in [\underline{\sigma}, \bar{\sigma}] \subseteq (0, \infty)$, we further have that the first term
%     \begin{align*}
%         \int_{t=0}^{\epsilon} \frac{1}{a(t)(a(1)-a(t))}w(t)dt
%         &\le 
%         \frac{\bar{\sigma}}{\underline{\sigma}(1-\epsilon)} \int_{t=0}^{\epsilon} \frac{w(t)}{t} dt\\
%         &\le 
%         \frac{a\bar{\sigma}}{\underline{\sigma}(1-\epsilon)} \int_{t=0}^{\epsilon} t^{b-1} dt\\
%         &\le 
%         \frac{a\bar{\sigma}}{\underline{\sigma}(1-\epsilon)} 
%         \frac{\epsilon^b}{b}\\
%         &< \infty.
%     \end{align*}
%     The same argument goes for the second integral.
% \end{proof}

\begin{proof}
Fix any $\epsilon\in(0,\tfrac12)$. From continuity, the integrand is bounded on
$[\epsilon,1-\epsilon]$, so it suffices to show that
\[
\int_{t=0}^{\epsilon}\frac{w(t)\,dt}{a(t)\bigl(a(1)-a(t)\bigr)}
+\int_{t=1-\epsilon}^{1}\frac{w(t)\,dt}{a(t)\bigl(a(1)-a(t)\bigr)}<\infty .
\]
Since $\sigma(t)\in[\underline{\sigma},\overline{\sigma}]\subseteq(0,\infty)$, integrating the
lower bound gives
\[
a(t)=\int_{0}^{t}\sigma(s)^{2}\,ds\;\ge\;\underline{\sigma}^{2}t,
\qquad
a(1)-a(t)=\int_{t}^{1}\sigma(s)^{2}\,ds\;\ge\;\underline{\sigma}^{2}(1-t).
\]
For $t\le\epsilon$ we have $1-t\ge 1-\epsilon$, hence
$a(1)-a(t)\ge\underline{\sigma}^{2}(1-\epsilon)$, and we further have that the first term
satisfies
\begin{align*}
\int_{t=0}^{\epsilon}\frac{w(t)\,dt}{a(t)\bigl(a(1)-a(t)\bigr)}
&\le\frac{1}{\underline{\sigma}^{4}(1-\epsilon)}\int_{t=0}^{\epsilon}\frac{w(t)}{t}\,dt\\
&\le\frac{C_0}{\underline{\sigma}^{4}(1-\epsilon)}\int_{t=0}^{\epsilon}t^{\,\beta_0-1}\,dt\\
&=\frac{C_0}{\underline{\sigma}^{4}(1-\epsilon)}\cdot\frac{\epsilon^{\,\beta_0}}{\beta_0}<\infty,
\end{align*}
where we used $w(t)\le C_0t^{\beta_0}$ from Assumption~\ref{assumption:weight} and $\beta_0>0$.

The same argument goes for the second integral, with the roles of the two endpoints
exchanged: for $t\ge 1-\epsilon$ we have $a(t)\ge\underline{\sigma}^{2}(1-\epsilon)$ and
$a(1)-a(t)\ge\underline{\sigma}^{2}(1-t)$, so with $w(t)\le C_1(1-t)^{\beta_1}$,
\[
\int_{t=1-\epsilon}^{1}\frac{w(t)\,dt}{a(t)\bigl(a(1)-a(t)\bigr)}
\le\frac{C_1}{\underline{\sigma}^{4}(1-\epsilon)}\int_{t=1-\epsilon}^{1}(1-t)^{\,\beta_1-1}\,dt
=\frac{C_1}{\underline{\sigma}^{4}(1-\epsilon)}\cdot\frac{\epsilon^{\,\beta_1}}{\beta_1}<\infty. \qedhere
\]
\end{proof}

\begin{lemma}[Weighted square-integrability of the score targets]
\label{lemma:score-target-l2}
Under Assumptions~\ref{assumption:1} and~\ref{assumption:weight}, both score
targets in Equations~\ref{eqn:loss_forward} and~\ref{eqn:reverse_time_loss}
belong to their weighted $L^2$ spaces.
\end{lemma}

\begin{proof}
Write $A=a(1)$ and $\Delta=X_1-X_0$. The bridge sampling formula
\begin{equation}\label{eqn:cond_dist_sample}
X_t=\left(1-\frac{a(t)}A\right)X_0+\frac{a(t)}A X_1
    +\sqrt{\frac{a(t)(A-a(t))}{A}}\,Z,
\end{equation}
holds with $Z\sim N(0,I_d)$ independent of $(X_0,X_1)$. Consequently,
\begin{align*}
\Phi_f
&:=\nabla_{X_t}\log p_{\rm base}(X_1\mid X_t)
=\frac{\Delta}{A}
 -\sqrt{\frac{a(t)}{A(A-a(t))}}\,Z,\\
\Phi_r
&:=\nabla_{X_t}\log p_{\rm base}(X_t\mid X_0)
=-\frac{\Delta}{A}
 -\sqrt{\frac{A-a(t)}{Aa(t)}}\,Z.
\end{align*}
Independence and centering of $Z$ give
\begin{align*}
\mathbb E\|\Phi_f\|_2^2
&=\frac{\mathbb E\|\Delta\|_2^2}{A^2}
  +\frac{d\,a(t)}{A(A-a(t))},\\
\mathbb E\|\Phi_r\|_2^2
&=\frac{\mathbb E\|\Delta\|_2^2}{A^2}
  +\frac{d(A-a(t))}{Aa(t)}.
\end{align*}
The endpoint term is finite by Assumption~\ref{assumption:1}. Moreover,
\[
\frac{a(t)}{A-a(t)}\le
\frac{A^2}{a(t)(A-a(t))},\qquad
\frac{A-a(t)}{a(t)}\le
\frac{A^2}{a(t)(A-a(t))},
\]
and $\int_0^1w(t)\,dt<\infty$. Lemma~\ref{lemma:2} therefore makes the
weighted time integrals of both displayed second moments finite.
\end{proof}

\paragraph{Proof of Theorem \ref{thm:forward_process} (score-matching objective).} 

\begin{proof}
Let $\Pi_0$ be the first marginal of $\Pi$ and fix a regular conditional
kernel $x_0\mapsto\Pi_1^{x_0}$ for $X_1$ given $X_0=x_0$. The spaces are
standard Borel, so such a kernel and the jointly Borel versions below exist.
Under the sampling law in Equation~\ref{eqn:loss_forward}, set
\[
\mathbf f^\star(X_t,t;X_0)
:=\mathbb E[\Phi_f\mid X_t,t,X_0],
\]
choosing the $h$-transform version
\begin{equation}\label{eqn:canonical-forward-drift}
\mathbf f^\star(x,t;x_0)
=\frac{m(t,x;x_0,\Pi_1^{x_0})-x}{a(1)-a(t)}
\end{equation}
for $0<t<1$ and $\Pi_0$-almost every $x_0$. Lemma~\ref{lemma:score-target-l2}
justifies the $L^2$ conditional expectation. The Pythagorean identity gives,
for every square-integrable candidate $f$,
\[
\overrightarrow{\mathcal L}(f)
=\overrightarrow{\mathcal L}(\mathbf f^\star)
 +\mathbb E\!\left[w(t)\|f(X_t,t;X_0)
                   -\mathbf f^\star(X_t,t;X_0)\|_2^2\right].
\]
Thus $\mathbf f^\star$ is the unrestricted minimizer, and realizability implies
that every exact parameterized minimizer agrees with it under the weighted
training-input law. This conclusion concerns an $L^2$ equivalence class; the
SDE below uses the fixed version~\eqref{eqn:canonical-forward-drift}.

For $\Pi_0$-almost every $x_0$, Lemma~\ref{lemma:mixture of bridges} applied
with $F_1=\Pi_1^{x_0}$ makes the coordinate process under
$\mathbb P^{(x_0,\Pi_1^{x_0})}$ a weak solution of
Equation~\ref{eqn:data_gen} on every $[0,T]$, $T<1$. Mixing these laws over
$\Pi_0(dx_0)$ gives
\[
\int\mathbb P^{(x_0,\Pi_1^{x_0})}\Pi_0(dx_0)
=\int\mathbb P^{(x_0,x_1)}\Pi(dx_0,dx_1)=\mathbb Q.
\]
Conditionally on each initial value, the innovation in the finite-horizon
$h$-transform is a standard Brownian motion with the same Wiener law. Hence,
after mixing, it is independent of $\overrightarrow X_0$. Finally, continuous
extension at $t=1$ follows from the mixture-of-pinned-bridges construction,
and the extended path law is $\mathbb Q$.
\end{proof}

\begin{remark}
The theorem is a population-level oracle statement under realizability.
Practical training may incur estimation, optimization, and approximation error.
Even at the population optimum, an arbitrary modification of a minimizer on a
training-null set is not automatically an interchangeable SDE coefficient; the
exact path-law conclusion uses the canonical version displayed above.
\end{remark}

\subsection{Proof of Reverse Process SDE and Objective}\label{sec:reverse_process_proof}

The proof for this theorem (Theorem~\ref{thm:reverse_process}) mirrors the proof of Theorem~\ref{thm:forward_process} but with the flipped time notations (and appropriately mirrored coefficients). %If the loss term has $\pbase(x_0 \vert x_t)$ instead of $\pbase(x_t \vert x_0)$, then the proof will be the same. However, $\pbase$ is defined from $\mathbb{P}$, so $\pbase(x_0 \vert x_t)$ is data-dependent and can be ill-defined unlike $\pbase(x_t \vert x_0)$. Still, we can show that the change will not alter the objective, and the proof used in the first theorem is applicable.

% \paragraph{Proof of Theorem \ref{thm:reverse_process}.}
% \begin{proof}

%     Similar to the proof of Theorem~\ref{thm:forward_process} and with the same coupling, we have that, almost surely
%     \begin{align*}
%         \mathbf{r}_{\theta}(X_t,t;X_1)
%         &= \mathbb{E}\left[
%         \nabla_{x_t} \log \pbase(X_t \vert X_0) \middle\vert X_1, X_t
%         \right]\\
%         &=
%         \mathbb{E}\left[
%         \nabla_{x_t} \log e^{- \frac{\Vert X_t - X_0 \Vert_2^2}{2a(t)}} \middle\vert X_1, X_t
%         \right]\\
%         &=
%         \frac{\mathbb{E}[X_0 \vert X_1, X_t] - X_t}{a(t)}.
%     \end{align*}
%     Note that this mirror the drift of $\mathbf{f}_{\theta}(X_t,t;X_0) = \frac{\mathbb{E}[X_1\vert X_0, X_t] - X_t}{a(1)-a(t)}$ for the forward theorem. The rest of the proof proceeds in the same manner, because the conditional distribution of $X_t$ given $(X_0, X_1) = (x_0,x_1)$ and the volatility schedule $\sigma$ is the same as the distribution of $X_{1-t}$ given $(X'_0, X'_1) = (x_1, x_0)$ with the counterfactual volatility schedule of $\sigma'(t) = \sigma(1-t)$. 
% \end{proof}

% \textcolor{red}{check this}

\begin{proof}[Proof of Theorem~\ref{thm:reverse_process}]
Consider the endpoint-swapped coupling
\[
    (X'_0,X'_1) \coloneqq (X_1,X_0)
\]
and define the reversed volatility schedule and weight by
\[
    \sigma'(u) \coloneqq \sigma(1-u),
    \qquad
    w'(u) \coloneqq w(1-u).
\]
The corresponding integrated variance is
\begin{align*}
    a'(u)
    &\coloneqq
    \int_0^u \sigma'(s)^2\,ds \\
    &=
    \int_0^u \sigma(1-s)^2\,ds \\
    &=
    a(1)-a(1-u).
\end{align*}
In particular,
\[
    a'(1)-a'(u)=a(1-u).
\]

Applying Theorem~\ref{thm:forward_process} to the swapped coupling and the volatility
schedule $\sigma'$ gives a forward process
$(\overleftarrow{X}_u)_{u\in[0,1)}$ starting from
\[
    \overleftarrow{X}_0\sim p_{\mathrm{data}}(X_1).
\]
Its canonical population drift is the fixed Borel version
\begin{align*}
    \mathbf r^\star(x,1-u;x_1)
    &=
    \mathbb{E}
    \left[
        \nabla_x
        \log p_{\mathrm{base}}(x\mid X_0)
        \,\middle|\,
        X_1=x_1,\ X_{1-u}=x
    \right] \\
    &=
    \frac{
        \mathbb{E}[X_0\mid X_1=x_1,\ X_{1-u}=x]-x
    }{
        a(1-u)
    }.
\end{align*}
Therefore, the resulting SDE is
\[
    d\overleftarrow{X}_u
    =
    \sigma(1-u)^2
    \mathbf r^\star(
        \overleftarrow{X}_u,
        1-u;
        \overleftarrow{X}_0
    )\,du
    +
    \sigma(1-u)\,dB_u,
\]
which is exactly Equation~\ref{eqn:reverse}.

Lemma~\ref{lemma:score-target-l2} first establishes that the target in
Equation~\ref{eqn:reverse_time_loss} belongs to the relevant weighted $L^2$
space. Changing variables according to $t=1-u$ transforms the forward
objective for the swapped process into Equation~\ref{eqn:reverse_time_loss},
and $w'$ satisfies Assumption~\ref{assumption:weight}, with the constants and
endpoint exponents interchanged. The same Pythagorean identity used in the
forward proof therefore shows that $\mathbf r^\star$ is the unrestricted
minimizer and that every exact realizable parameterized minimizer agrees with
it under the weighted training-input law. The reverse SDE nevertheless uses
the fixed displayed version, not an arbitrary representative of that class.

For completeness, the reversal of the bridge law follows directly from its
Gaussian mean and two-time covariance. For $s\le t$, the conditional covariance
of the original bridge is
\[
\operatorname{Cov}(X_s,X_t\mid X_0,X_1)
=\frac{a(s)(a(1)-a(t))}{a(1)}I_d.
\]
Together with $a'(u)=a(1)-a(1-u)$, this shows that the $\sigma'$-bridge from
$(x_1,x_0)$ is exactly the image of the $\sigma$-bridge from $(x_0,x_1)$ under
$\omega(t)\mapsto\omega(1-t)$. Theorem~\ref{thm:forward_process} applied to the
swapped coupling thus gives, after the almost-sure continuous extension at
$u=1$,
\[
    \operatorname{Law}
    \bigl(
        (\overleftarrow{X}_{1-t})_{t\in[0,1]}
    \bigr)
    =
    \mathbb{Q},
\]
as claimed. Its innovation Brownian motion is independent of the reverse
initial endpoint by the independence statement in the forward construction.
\end{proof}

Alternatively, we could prove this via Anderson's time-reversal~\citep{anderson1982reverse}.

\newpage\clearpage

\section{Marginal SDEs}\label{sec:marginal_sdes}

\begin{lemma}[Endpoint projection]\label{lem:projection}
Let $\Pi$ be the endpoint law and let
\[
K_t((x_0,x_1),dx)
:=p_{\rm base}(x\mid x_0,x_1)\,dx,
\qquad 0<t<1,
\]
be the Gaussian bridge kernel. With
$Z_w:=\int_0^1w(t)\,dt$, define the probability measure
\begin{equation}\label{eqn:weighted-sampling-measure}
\nu(dt,dx_0,dx_1,dx)
:=Z_w^{-1}w(t)\,dt\,\Pi(dx_0,dx_1)
  K_t((x_0,x_1),dx).
\end{equation}
Suppose that the Borel target $\Phi(x,t,x_0,x_1)$ belongs to $L^2(\nu)$.
Let $\kappa_{t,x}(dx_0,dx_1)$ be a regular conditional endpoint kernel under
$\nu$ given $(t,X_t)=(t,x)$; let $\kappa^0_{t,x}$ and $\kappa^1_{t,x}$ be its
marginals, and fix conditional kernels
$\kappa^{1\mid0}_{t,x,x_0}(dx_1)$ and
$\kappa^{0\mid1}_{t,x,x_1}(dx_0)$. Then jointly Borel versions of the three
unrestricted square-loss minimizers are
\begin{align*}
\phi_\varnothing^\star(x,t)
&=\int\Phi(x,t,x_0,x_1)\,\kappa_{t,x}(dx_0,dx_1),\\
\phi_0^\star(x,t;x_0)
&=\int\Phi(x,t,x_0,x_1)\,
       \kappa^{1\mid0}_{t,x,x_0}(dx_1),\\
\phi_1^\star(x,t;x_1)
&=\int\Phi(x,t,x_0,x_1)\,
       \kappa^{0\mid1}_{t,x,x_1}(dx_0).
\end{align*}
They satisfy, for the $(t,X_t)$-marginal of $\nu$-almost every $(t,x)$,
\begin{equation}\label{eq:projection}
\phi_\varnothing^\star(x,t)
=\int\phi_0^\star(x,t;x_0)\,\kappa^0_{t,x}(dx_0)
=\int\phi_1^\star(x,t;x_1)\,\kappa^1_{t,x}(dx_1).
\end{equation}
Under realizability, every exact parameterized population minimizer agrees
with its corresponding displayed oracle under the relevant marginal of $\nu$.
The displayed Borel versions, rather than arbitrary modifications on
$\nu$-null sets, are used below as SDE coefficients.
\end{lemma}

\begin{proof}
All state spaces are standard Borel, so the stated regular conditional kernels
exist. The $L^2$ projection theorem identifies each minimizer with the
conditional expectation of $\Phi$ given its network inputs. The three kernel
integrals are versions of those conditional expectations. Equation
\eqref{eq:projection} is the tower property, first conditioning on
$(t,X_t,X_0)$ and then on $(t,X_t)$, or symmetrically through
$(t,X_t,X_1)$. The final realizability statement follows from the Pythagorean
identity for square loss.
\end{proof}

\begin{lemma}[Superposition lift to a weak SDE]
\label{lem:superposition-lift}
Let $(\mu_t)_{0\le t\le1}$ be a narrowly continuous probability-valued flow,
let $b:[0,1]\times\mathbb R^d\to\mathbb R^d$ be Borel, and let
$\gamma:[0,1]\to(0,\infty)$ be continuous and bounded. Suppose
\[
\int_0^1\!\int_{\mathbb R^d}\|b(t,x)\|_2\,\mu_t(dx)\,dt<\infty
\]
and, for every $\varphi\in C_c^\infty(\mathbb R^d)$ and $t\in[0,1]$,
\begin{align}
\int\varphi\,d\mu_t
&=\int\varphi\,d\mu_0
 +\int_0^t\!\int
 \left[b(s,x)\cdot\nabla\varphi(x)
 +\frac{\gamma(s)^2}{2}\Delta\varphi(x)\right]
 \mu_s(dx)\,ds.
\label{eq:generic-weak-fpe}
\end{align}
Then there is a probability law on $C([0,1];\mathbb R^d)$ whose coordinate
process has marginals $(\mu_t)$ and, together with a Brownian motion $B$, is a
weak solution of
\[
dY_t=b(t,Y_t)\,dt+\gamma(t)\,dB_t.
\]
The Brownian motion may be taken relative to the solution filtration and is
therefore independent of $Y_0$. This is an existence statement only; neither
uniqueness nor a Markov property of the selected law is asserted.
\end{lemma}

\begin{proof}
The coefficient integrability above, together with boundedness of the
covariance $\gamma(t)^2I_d$, permits application of the superposition
principle~\citep[Theorem~2.5]{trevisan2016wellposedness} to
Equation~\eqref{eq:generic-weak-fpe}. It yields a martingale-problem solution
on $C([0,1];\mathbb R^d)$ having exactly the prescribed marginals. Under this
law, localization of the coordinate martingale problem shows that
\[
M_t:=Y_t-Y_0-\int_0^t b(s,Y_s)\,ds
\]
is a continuous local martingale with
$[M^i,M^j]_t=\delta_{ij}\int_0^t\gamma(s)^2ds$. Since $\gamma$ is positive,
$B_t:=\int_0^t\gamma(s)^{-1}\,dM_s$ is well defined and has quadratic
covariation $[B^i,B^j]_t=\delta_{ij}t$. The multidimensional L\'evy
characterization makes $B$ Brownian relative to the solution filtration.
No uniqueness conclusion is part of the superposition principle.
\end{proof}

\begin{remark}\label{remark:projected-drifts}
For the forward target
$\Phi_f=\nabla_x\log p_{\rm base}(X_1\mid x)$, write the canonical conditional
and unconditional versions as $\mathbf f^\star(x,t;x_0)$ and
$\bar{\mathbf f}(x,t)$, respectively. For the reverse target
$\Phi_r=\nabla_x\log p_{\rm base}(x\mid X_0)$, write them as
$\mathbf r^\star(x,t;x_1)$ and $\bar{\mathbf r}(x,t)$. Thus
\begin{align*}
\bar{\mathbf f}(x,t)
&=\int\mathbf f^\star(x,t;x_0)\,\kappa^0_{t,x}(dx_0),\\
\bar{\mathbf r}(x,t)
&=\int\mathbf r^\star(x,t;x_1)\,\kappa^1_{t,x}(dx_1)
\end{align*}
for the weighted $(t,X_t)$-sampling law almost everywhere. Lemma
\ref{lemma:score-target-l2} verifies the required $L^2(\nu)$ hypothesis for
both targets. Since $w(t)>0$ on $(0,1)$, this weighted law and
$dt\,q_t(dx)$ have the same null sets, so the projection identities also hold
$dt\,q_t(dx)$-almost everywhere.
\end{remark}

\begin{lemma}[Integrability of the projected coefficients]
\label{lem:projected-coefficient-integrability}
Let $q_t:=\mathbb Q\circ X_t^{-1}$, set $A:=a(1)$ and
$D:=\mathbb E_\Pi\|X_1-X_0\|_2^2$, and define, for $0<t,u<1$,
\[
b^\rightarrow(t,x):=\sigma(t)^2\bar{\mathbf f}(x,t),\qquad
b^\leftarrow(u,x):=\sigma(1-u)^2\bar{\mathbf r}(x,1-u).
\]
Assign arbitrary Borel values (for example, zero) to these coefficients at
the two time endpoints. Then $(q_t)_{0\le t\le1}$ is narrowly continuous,
$q_0=\Pi_0$, $q_1=\Pi_1$, and
\begin{align}
\int_0^1\!\int_{\mathbb R^d}
 \|b^\rightarrow(t,x)\|_2\,q_t(dx)\,dt&<\infty,
 \label{eq:forward-projected-l1}\\
\int_0^1\!\int_{\mathbb R^d}
 \|b^\leftarrow(u,x)\|_2\,q_{1-u}(dx)\,du&<\infty.
 \label{eq:reverse-projected-l1}
\end{align}
Together with boundedness of $\sigma^2I_d$, these are precisely the
coefficient-integrability conditions needed below for the superposition
principle.
\end{lemma}

\begin{proof}
On a common probability space take $(X_0,X_1)\sim\Pi$ and
$Z\sim N(0,I_d)$ independently, and define
\[
Y_t=\left(1-\frac{a(t)}A\right)X_0+\frac{a(t)}A X_1
 +\sqrt{\frac{a(t)(A-a(t))}{A}}\,Z.
\]
Equation~\eqref{eqn:cond_dist_sample} gives $\operatorname{Law}(Y_t)=q_t$.
The continuity of $a$ and the finite second moments imply
$Y_t\to Y_s$ in $L^2$ as $t\to s$. Hence $(q_t)$ is continuous even in
$\mathcal P_2(\mathbb R^d)$, and $q_0=\Pi_0$, $q_1=\Pi_1$.

Because $w(t)>0$, disintegrating conditional Jensen under the weighted law in
time and using Lemma~\ref{lemma:score-target-l2} give, for Lebesgue almost
every $t\in(0,1)$,
\begin{align*}
\int\|\bar{\mathbf f}(x,t)\|_2^2q_t(dx)
&\le \frac{D}{A^2}+\frac{d\,a(t)}{A(A-a(t))}
 \le \frac{D}{A^2}+\frac{d}{\underline\sigma^2(1-t)},\\
\int\|\bar{\mathbf r}(x,t)\|_2^2q_t(dx)
&\le \frac{D}{A^2}+\frac{d(A-a(t))}{Aa(t)}
 \le \frac{D}{A^2}+\frac{d}{\underline\sigma^2t}.
\end{align*}
Indeed, $A-a(t)\ge\underline\sigma^2(1-t)$ and
$a(t)\ge\underline\sigma^2t$. A second application of Jensen, the bound
$\sigma^2\le\bar\sigma^2$, and
$\sqrt{x+y}\le\sqrt x+\sqrt y$ therefore yield
\begin{align*}
\int_0^1\!\int\|b^\rightarrow(t,x)\|_2q_t(dx)\,dt
&\le\bar\sigma^2
 \left(\frac{\sqrt D}{A}+\frac{2\sqrt d}{\underline\sigma}\right),\\
\int_0^1\!\int\|b^\leftarrow(u,x)\|_2q_{1-u}(dx)\,du
&\le\bar\sigma^2
 \left(\frac{\sqrt D}{A}+\frac{2\sqrt d}{\underline\sigma}\right).
\end{align*}
This proves \eqref{eq:forward-projected-l1}--
\eqref{eq:reverse-projected-l1}. Finally, the diffusion covariance is
$\sigma(t)^2I_d$ in forward time and $\sigma(1-u)^2I_d$ in reverse time, so
its norm is bounded and hence integrable against either probability-valued
marginal flow.
\end{proof}

\begin{corollary}[Marginal equivalence, forward direction]\label{prop:marginal-match}
Let $q_t:=\mathbb Q\circ X_t^{-1}$ be the time-$t$ marginal of the bridge law
in Equation~\ref{eqn:q_measure_full}. There exists a probability law
$\tilde{\mathbb Q}$ on $C([0,1];\mathbb R^d)$ under which the coordinate
process, together with a Brownian motion $B$ independent of $X_0$, is a weak
solution on every $[0,T]$, $T<1$, of the state-only projected SDE
\begin{equation}
dX_t=\sigma(t)^2\bar{\mathbf f}(X_t,t)\,dt+\sigma(t)\,dB_t,
\qquad X_0\sim\Pi_0,\qquad 0\le t<1,
\label{eq:markov-sde}
\end{equation}
and the coordinate path itself supplies a continuous extension at $t=1$. Here
$\bar{\mathbf f}$ is the fixed Borel projection version in
Remark~\ref{remark:projected-drifts}; it is the unrestricted minimizer of
\begin{equation}\label{eqn:loss_forward_cfg}
\mathcal{L}(\mathbf{f}_\theta)
=\underset{\substack{t \sim \mathcal{U}((0, 1)) \\ (x_0, x_1) \sim \pdata \\ x_t \sim \pbase(x_t | x_0, x_1)}}{\mathbb{E}}
\left[w(t)\left\| \mathbf{f}_\theta(x_t,t;\emptyset)
-\nabla_{x_t}\log p_\text{base}(x_1|x_t)\right\|_2^2\right].
\end{equation}
Under realizability, every exact parameterized minimizer agrees with
$\bar{\mathbf f}$ under the weighted training-input law, but Equation
\ref{eq:markov-sde} uses the fixed displayed version. Then
\[
\tilde{\mathbb Q}\circ X_t^{-1}=q_t,
\qquad 0\le t\le1.
\]
The conclusion is existential: it selects at least one weak solution with the
fixed canonical state-only drift. No uniqueness in law or Markov property of
the selected law is asserted, and the corollary does not claim that every weak
solution of Equation~\ref{eq:markov-sde} has these marginals.
\end{corollary}

\begin{proof}
See Appendix~\ref{sec:forward_process_marginal_proof}.
\end{proof}

Thus the fixed projected coefficient $\bar{\mathbf f}(x,t)$ admits a
state-only weak solution with the same one-time marginals as the canonical
process in Equation~\ref{eqn:data_gen}. The two solutions start from the same
textual data distribution and end at the same image data distribution, but
the state-only coefficient does not retain the starting text state. This
population construction motivates sampling and restoration schemes that do
not require knowledge of the endpoint, including the uses in
Algorithm~\ref{alg:roundtrip-image-editing} and the CFG-inspired heuristic in
Section~\ref{sec:cfg}; it does not supply uniqueness for the implemented SDE.

\begin{corollary}[Marginal equivalence, reverse]
\label{prop:marginal-match-rev}
There exists a probability law $\overleftarrow{\mathbb Q}$ on
$C([0,1];\mathbb R^d)$ under which the coordinate process, together with a
Brownian motion $B$ independent of $\overleftarrow X_0$, is a weak solution on
every $[0,T]$, $T<1$, of the state-only projected SDE
\begin{equation}
d\overleftarrow X_u
=\sigma(1-u)^2\bar{\mathbf r}(\overleftarrow X_u,1-u)\,du
 +\sigma(1-u)\,dB_u,
\qquad \overleftarrow X_0\sim\Pi_1,
\qquad 0\le u<1,
\label{eq:rev-markov-sde}
\end{equation}
and the coordinate path itself supplies a continuous extension at $u=1$. Here
$\bar{\mathbf r}$ is the fixed Borel projection version in
Remark~\ref{remark:projected-drifts}; it is the unrestricted minimizer of
\begin{equation}\label{eqn:reverse_time_loss_cfg}
\mathcal{L}(\mathbf r_\theta)
=\underset{\substack{t \sim \mathcal{U}((0, 1)) \\ (x_0, x_1) \sim \pdata \\ x_t \sim \pbase(x_t | x_0, x_1)}}{\mathbb{E}}
\left[w(t)\left\|\mathbf r_\theta(x_t,t;\emptyset)
-\nabla_{x_t}\log p_\text{base}(x_t|x_0)\right\|_2^2\right].
\end{equation}
Under realizability, every exact parameterized minimizer agrees with
$\bar{\mathbf r}$ under the weighted training-input law, while Equation
\ref{eq:rev-markov-sde} uses the fixed displayed version. Then
\[
\overleftarrow{\mathbb Q}\circ\overleftarrow X_u^{-1}=q_{1-u},
\qquad 0\le u\le1.
\]
This is an existence statement for a selected weak solution with the fixed
canonical state-only drift. No uniqueness in law or Markov property of the
selected law is asserted, and not every weak solution of
Equation~\ref{eq:rev-markov-sde} is claimed to have these marginals.
\end{corollary}

\begin{proof}
Let
\[
\Gamma_t^1(dx_1,dx)
:=\operatorname{Law}_{\mathbb Q}(X_1,X_t)
=\kappa^1_{t,x}(dx_1)q_t(dx)
\]
for almost every $t$. Apply It\^o's formula to the continuously extended
reverse endpoint-conditioned process from Theorem~\ref{thm:reverse_process}.
For $\varphi\in C_c^\infty(\mathbb R^d)$ and $u\le T<1$, disintegration and
the reverse tower identity in Equation~\eqref{eq:projection} give
\begin{align*}
\int\varphi(x)q_{1-u}(dx)
&=\int\varphi(x)\Pi_1(dx)\\
&\quad+\int_0^u\!\int
\left[\sigma(1-s)^2\bar{\mathbf r}(x,1-s)\cdot\nabla\varphi(x)
 +\frac{\sigma(1-s)^2}{2}\Delta\varphi(x)\right]
q_{1-s}(dx)\,ds.
\end{align*}
The local integrability follows from Lemma~\ref{lemma:score-target-l2} and
conditional Jensen, because $1-s\ge1-T>0$. Thus $(q_{1-u})_{u\le T}$ solves
the weak Fokker--Planck equation associated with
Equation~\ref{eq:rev-markov-sde}. Lemma
\ref{lem:projected-coefficient-integrability} gives narrow continuity through
$u=1$ and absolute integrability of the reverse projected drift over the full
interval. Letting $u\uparrow1$ in the last display therefore proves the same
weak equation on $[0,1]$, with terminal marginal $q_0=\Pi_0$.

Lemma~\ref{lem:superposition-lift}, applied with
$\mu_u=q_{1-u}$, $b=b^\leftarrow$, and $\gamma(u)=\sigma(1-u)$, now yields a
law $\overleftarrow{\mathbb Q}$ on $C([0,1];\mathbb R^d)$ with those
one-time marginals and a weak solution of
Equation~\ref{eq:rev-markov-sde} on every $[0,T]$, $T<1$. Its continuous
endpoint has law $q_0=\Pi_0$. Again, neither the superposition principle nor
this corollary asserts uniqueness.
\end{proof}

\newpage
\clearpage

\section{Proof of Forward-Time Marginal Equivalence}
\label{sec:forward_process_marginal_proof}

Proof of Corollary~\ref{prop:marginal-match} follows.

\begin{proof}
Write
\[
\Gamma_t^0(dx_0,dx)
:=\operatorname{Law}_{\mathbb Q}(X_0,X_t),
\qquad q_t(dx):=\operatorname{Law}_{\mathbb Q}(X_t).
\]
For almost every $t\in(0,1)$, disintegration with the kernel selected in
Lemma~\ref{lem:projection} gives
\[
\Gamma_t^0(dx_0,dx)=\kappa^0_{t,x}(dx_0)q_t(dx).
\]

Fix $T<1$ and $\varphi\in C_c^\infty(\mathbb R^d)$. Apply It\^o's formula to
the canonical endpoint-conditioned weak solution of
Theorem~\ref{thm:forward_process}, then integrate first over the conditional
endpoint laws and then over $\Pi_0$. The local integrability needed for this
step follows from the score formulas in Lemma~\ref{lemma:score-target-l2}
(on $[0,T]$ the forward singular factor is bounded) and conditional Jensen.
For brevity, set
\begin{align*}
\mathcal A_s^{x_0}\varphi(x)
&:=\sigma(s)^2\mathbf f^\star(x,s;x_0)\cdot\nabla\varphi(x)
  +\frac{\sigma(s)^2}{2}\Delta\varphi(x),\\
\bar{\mathcal A}_s\varphi(x)
&:=\sigma(s)^2\bar{\mathbf f}(x,s)\cdot\nabla\varphi(x)
  +\frac{\sigma(s)^2}{2}\Delta\varphi(x).
\end{align*}
For every $t\le T$,
\begin{align}
\int\varphi(x)q_t(dx)
&=\int\varphi(x)\Pi_0(dx)
 +\int_0^t\!\int\mathcal A_s^{x_0}\varphi(x)
 \Gamma_s^0(dx_0,dx)\,ds \notag\\
&=\int\varphi(x)\Pi_0(dx)
 +\int_0^t\!\int\bar{\mathcal A}_s\varphi(x)
 q_s(dx)\,ds.\label{eqn:fokker_planck_qt}
\end{align}
The second equality is exactly the tower identity
\eqref{eq:projection}; an identity holding for $ds\,q_s(dx)$-almost every
$(s,x)$ is sufficient in this integrated weak equation.

Thus $(q_t)_{0\le t\le T}$ is a probability-measure-valued
weak solution of the Fokker--Planck Cauchy problem associated with
Equation~\ref{eq:markov-sde}. Lemma
\ref{lem:projected-coefficient-integrability} gives narrow continuity on the
full closed interval and absolute integrability of the projected drift against
$dt\,q_t(dx)$. Since the diffusion term is bounded, we may let
$t\uparrow1$ in Equation~\eqref{eqn:fokker_planck_qt}; the left side converges
to $\int\varphi\,d\Pi_1$, while both time integrals converge absolutely.
Consequently Equation~\eqref{eqn:fokker_planck_qt} holds for every
$t\in[0,1]$, so $(q_t)$ solves the projected Fokker--Planck equation on the
entire closed interval.

Apply Lemma~\ref{lem:superposition-lift} with
$b=b^\rightarrow$, $\gamma=\sigma$, and $\mu_t=q_t$. It produces a law
$\tilde{\mathbb Q}$ on $C([0,1];\mathbb R^d)$ with all one-time marginals
$q_t$ and under which the coordinate process satisfies
Equation~\ref{eq:markov-sde} (in particular on every finite horizon
$[0,T]$, $T<1$). The coordinate path supplies the continuous endpoint
extension and has terminal law $q_1=\Pi_1$. The superposition argument is
existential and makes no uniqueness claim.
\end{proof}

\newpage
\clearpage

\section{Proof of Local Mutual Information Advantage}\label{sec:local_mutual_info_proof}

Proof of Theorem \ref{thm:mutual_info} follows.

\begin{proof}
Define the mutual-information gap
\[
    \Delta I(t)
    \coloneqq
    I(X_t^{\mathrm{TI}};Y)
    -
    I(X_t^{\mathrm{NI}};Y).
\]
At $t=0$, we have $X_0^{\mathrm{TI}}=T$ and
$X_0^{\mathrm{NI}}=N$. Since $N$ is independent of $Y$,
\[
    \Delta I(0)=I(T;Y)-I(N;Y)=I(T;Y)>0.
\]
By the assumed right-continuity of the mutual-information functions,
$\Delta I$ is right-continuous at $0$. Therefore, there exists $\tau>0$ such
that $\Delta I(t)>0$ for every $t\in[0,\tau)$, which proves
\[
    I(X_t^{\mathrm{TI}};Y)
    >
    I(X_t^{\mathrm{NI}};Y).
\]
\end{proof}

\begin{corollary}[Image-to-text direction]
The analogous statement holds, without loss of generality, for the
image-to-text bridge. In particular, let $M$ be noise independent of
$(Y,T)$, and define
\begin{align}
    \widetilde X_s^{\mathrm{IT}}
        &= \widetilde a_s Y+\widetilde b_s T+\widetilde\eta_s Z,\\
    \widetilde X_s^{\mathrm{NT}}
        &= \widetilde a_s M+\widetilde b_s T+\widetilde\eta_s Z'.
\end{align}
Under the corresponding endpoint and continuity assumptions, there
exists $\widetilde\tau>0$ such that
\[
    I(\widetilde X_s^{\mathrm{IT}};T)
    >
    I(\widetilde X_s^{\mathrm{NT}};T),
    \qquad s\in[0,\widetilde\tau).
\]
\end{corollary}

\begin{proof}
Mutual information is symmetric, so $I(Y;T)=I(T;Y)>0$. The result
follows by exchanging the roles of the image and text endpoints and
applying the theorem in reverse time, $s=1-t$.
\end{proof}

\newpage
\clearpage

\section{A Gaussian Instantiation of Theorem \ref{thm:mutual_info}: When Is The Margin Vacuous?}
\label{app:gaussian-mi}

Theorem~\ref{thm:mutual_info} is a right-continuity argument: it establishes $\Delta I(0) > 0$ for $\Delta I(t) := I(X^{\mathrm{TI}}_t;Y) - I(X^{\mathrm{NI}}_t;Y)$ and concludes that $\Delta I$ remains positive on some interval $[0,\tau)$. As noted in Section~\ref{sec:comparison}, the argument by itself places no lower bound on $\tau$, and the guarantee would be of little practical interest if $\tau$ were, say, $10^{-6}$. This appendix studies a tractable special case. We instantiate both bridges with scalar jointly Gaussian endpoints, for which every mutual information admits a closed form at \emph{all} $t$, and obtain an exact characterization of the region on which the data-to-data bridge dominates. When the endpoint correlation is positive and the interpolation coefficients are those used by \bit{}, the calculation gives $\tau = 1$.

Of course, in practice, it's unlikely that the text-image endpoints are mutually Gaussian, but this analysis nonetheless gives intuition on why Theorem~\ref{thm:mutual_info} is useful. Also, there could be domains in which the text-image endpoints are mutually Gaussian.

\subsection{Background: mutual information between jointly Gaussian variables}
\label{app:gauss-background}

Background on information theory can be found in ~\cite{cover1999elements}. Recall that the \emph{differential entropy} of a random vector $X \in \mathbb{R}^n$ with density $p$ is $h(X) := -\int p(x)\log p(x)\,dx$, and that mutual information decomposes as
\begin{equation}
\label{eq:mi-entropy}
I(X;Y) = h(X) + h(Y) - h(X,Y).
\end{equation}
For a Gaussian vector $W \sim \mathcal{N}(\mu, \Sigma)$ on $\mathbb{R}^n$, the entropy depends on the covariance alone:
\begin{equation}
\label{eq:gauss-entropy}
h(W) = \tfrac{1}{2}\log\bigl((2\pi e)^n \det \Sigma\bigr).
\end{equation}
Now let $(X,Y)$ be jointly Gaussian and scalar, with variances $\sigma_X^2, \sigma_Y^2$ and correlation $\varrho := \operatorname{Corr}(X,Y)$, so that
\[
\Sigma \;=\; \begin{pmatrix} \sigma_X^2 & \varrho\,\sigma_X\sigma_Y \\[2pt]
\varrho\,\sigma_X\sigma_Y & \sigma_Y^2\end{pmatrix},
\qquad
\det\Sigma = \sigma_X^2\sigma_Y^2\bigl(1 - \varrho^2\bigr).
\]
Substituting Equation~\ref{eq:gauss-entropy} into Equation~\ref{eq:mi-entropy}, the marginal terms $\tfrac12\log(2\pi e\,\sigma_X^2)$ and $\tfrac12\log(2\pi e\,\sigma_Y^2)$ cancel against the corresponding factors in the joint entropy, leaving
\begin{equation}
\label{eq:mi-gauss}
I(X;Y) \;=\; -\tfrac{1}{2}\log\bigl(1 - \varrho^2\bigr).
\end{equation}
Two consequences drive everything below. First, the mutual information between jointly Gaussian scalars depends on the joint law \emph{only} through the squared correlation
\begin{equation}
\label{eq:Rdef}
R \;:=\; \varrho^2 \;=\; \frac{\operatorname{Cov}(X,Y)^2}{\operatorname{Var}(X)\operatorname{Var}(Y)} \;\in\; [0,1).
\end{equation}
Second, $R \mapsto -\tfrac12\log(1-R)$ is strictly increasing on $[0,1)$. Therefore, comparing two mutual informations between Gaussian pairs reduces to comparing two squared correlations:
\begin{equation}
\label{eq:mono}
I(X^{\mathrm{TI}}_t;Y) > I(X^{\mathrm{NI}}_t;Y)
\quad\Longleftrightarrow\quad
R_{\mathrm{TI}}(t) > R_{\mathrm{NI}}(t).
\end{equation}
This is the single reduction that makes the Gaussian case tractable at all $t$: no integrals or entropies appear again after this point, only second moments.

\subsection{Setup and the two squared correlations}

Let $(T,Y)$ be jointly Gaussian with mean zero, unit variances, and correlation $\rho = \operatorname{Corr}(T,Y) \in (-1,1)\setminus\{0\}$. Let $N \sim \mathcal{N}(0,1)$ be independent of $(T,Y)$, and let $Z, Z'$ be independent standard normals, independent of all endpoints. Following Equations~\ref{eq:ti_bridge}--\ref{eq:ni_bridge}, the two bridges are
\begin{equation}
X^{\mathrm{TI}}_t = a_t T + b_t Y + \eta_t Z,
\qquad
X^{\mathrm{NI}}_t = a_t N + b_t Y + \eta_t Z'.
\end{equation}
Each is a linear combination of jointly Gaussian variables, hence $(X_t, Y)$ is jointly Gaussian in both cases and Section~\ref{app:gauss-background} applies. We suppress the subscript $t$ on $a_t, b_t, \eta_t$ for readability.

\paragraph{Text-to-image bridge.} Using bilinearity of covariance, independence of $Z$ from
$Y$, and $\operatorname{Var}(Y) = 1$:
\begin{align}
\operatorname{Cov}(X^{\mathrm{TI}}_t, Y)
&= a\operatorname{Cov}(T,Y) + b\operatorname{Var}(Y) + \eta\operatorname{Cov}(Z,Y)
= a\rho + b, \\[2pt]
\operatorname{Var}(X^{\mathrm{TI}}_t)
&= a^2\operatorname{Var}(T) + b^2\operatorname{Var}(Y) + \eta^2\operatorname{Var}(Z)
   + 2ab\operatorname{Cov}(T,Y) \notag \\
&= a^2 + b^2 + 2ab\rho + \eta^2 \;=:\; D .
\end{align}

\paragraph{Noise-to-image bridge.} Identically, except that $N \perp Y$ kills both the covariance contribution of the first term and the cross term in the variance:
\begin{align}
\operatorname{Cov}(X^{\mathrm{NI}}_t, Y) &= b, \\[2pt]
\operatorname{Var}(X^{\mathrm{NI}}_t) &= a^2 + b^2 + \eta^2 \;=:\; S .
\end{align}
Note for later that the two denominators differ only by the cross term,
\begin{equation}
\label{eq:DS}
D = S + 2ab\rho .
\end{equation}
Substituting into Equation~\ref{eq:Rdef} with $\operatorname{Var}(Y)=1$ gives the closed forms
\begin{equation}
\label{eq:R-forms}
R_{\mathrm{TI}}(t) = \frac{(a\rho + b)^2}{a^2 + b^2 + 2ab\rho + \eta^2},
\qquad
R_{\mathrm{NI}}(t) = \frac{b^2}{a^2 + b^2 + \eta^2}.
\end{equation}

\paragraph{Well-posedness.} Both quantities are legitimate squared correlations whenever $\eta^2 > 0$. For the denominators, completing the square gives $D = (a-b)^2 + \eta^2 + 2ab(1+\rho) > 0$ when $a,b \ge 0$ and $\rho > -1$, and $S > 0$ is immediate; so no division by zero occurs. For finiteness of the mutual information we need $R_{\mathrm{TI}} < 1$, which holds because
\begin{equation}
D - (a\rho + b)^2 = a^2\bigl(1 - \rho^2\bigr) + \eta^2 \;>\; 0 ,
\end{equation}
the two $2ab\rho$ terms having cancelled. The same computation with $\rho = 0$ gives $S - b^2 = a^2 + \eta^2 > 0$, so $R_{\mathrm{NI}} < 1$ as well.

\subsection{The dominance criterion}

By Equation~\ref{eq:mono} we need only determine the sign of $R_{\mathrm{TI}} - R_{\mathrm{NI}}$. Since both denominators are strictly positive, this sign equals the sign of the cross-multiplied numerator
\begin{equation}
\label{eq:Ndef}
\mathcal{N} \;:=\; (a\rho + b)^2\,S \;-\; b^2\,D .
\end{equation}
We simplify $\mathcal{N}$ in three steps. Substituting $D = S + 2ab\rho$ from Equation~\ref{eq:DS} and grouping the terms that carry $S$:
\begin{align}
\mathcal{N}
&= (a\rho + b)^2 S - b^2\bigl(S + 2ab\rho\bigr) \notag \\
&= S\Bigl[(a\rho + b)^2 - b^2\Bigr] - 2ab^3\rho .
\label{eq:step1}
\end{align}
Expanding the bracket, $(a\rho+b)^2 - b^2 = a^2\rho^2 + 2ab\rho = a\rho\,(a\rho + 2b)$, so a factor of $a\rho$ can be pulled out of the whole expression:
\begin{equation}
\label{eq:step2}
\mathcal{N} = a\rho\Bigl[S\,(a\rho + 2b) - 2b^3\Bigr].
\end{equation}
Finally, regroup the bracket as $a\rho S + 2b\bigl(S - b^2\bigr)$ and use
$S - b^2 = a^2 + \eta^2$:
\begin{equation}
\label{eq:step3}
\mathcal{N} = a\rho\Bigl[\,a\rho\,S + 2b\bigl(a^2 + \eta^2\bigr)\Bigr].
\end{equation}
This proves the following.

\begin{proposition}[Exact dominance criterion, Gaussian case]
\label{prop:gauss-criterion}
Under the assumptions above, for every $t$ with $\eta_t^2 > 0$,
\begin{equation}
\label{eq:criterion}
I(X^{\mathrm{TI}}_t;Y) > I(X^{\mathrm{NI}}_t;Y)
\quad\Longleftrightarrow\quad
a_t\rho\,\Bigl[\,a_t \rho\, S_t + 2 b_t \bigl(a_t^2 + \eta_t^2\bigr)\Bigr] > 0 ,
\qquad S_t := a_t^2 + b_t^2 + \eta_t^2 .
\end{equation}
\end{proposition}

% Two features of Equation~\ref{eq:criterion} are worth emphasizing. First, the criterion depends on $\rho$ only through its sign and its magnitude relative to the interpolation coefficients --- not on $I(T;Y)$ in isolation, which is what the hypothesis of Theorem~\ref{thm:mutual_info} controls. Second, the leading factor $a_t\rho$ vanishes exactly when $a_t = 0$, i.e.\ at the image endpoint, where the two bridges coincide by construction and the comparison is degenerate. All of the content lies in the bracket.

\subsection{The positively coupled regime}

\begin{corollary}[Full-interval dominance]
\label{cor:tau-one}
Suppose $\rho>0$ and, for every $t\in(0,1)$,
$a_t>0$, $b_t\ge0$, and $\eta_t^2>0$. Then Equation~\ref{eq:criterion}
gives
$I(X^{\mathrm{TI}}_t;Y)>I(X^{\mathrm{NI}}_t;Y)$ for every
$t\in(0,1)$. At $t=0$, the endpoint identities give
$I(X^{\mathrm{TI}}_0;Y)=I(T;Y)>0=I(N;Y)=I(X^{\mathrm{NI}}_0;Y)$.
Thus strict dominance holds on $[0,1)$; equivalently, $\tau=1$.
\end{corollary}

The coefficient conditions $a_t,b_t\ge0$ hold for \bit{}, but the additional condition $\rho>0$ concerns the endpoint coupling and does not follow from the interpolation coefficients. Specializing Equation~\ref{eqn:cond_dist_sample} to the endpoints $x_0 = T$, $x_1 = Y$ of the scaled Brownian base process, and writing $\lambda_t := C(0,t)/C(0,1)$ where $C(t_a,t_b) = \int_{t_a}^{t_b}\sigma(s)^2\,ds$ for the normalized time change,
\begin{equation}
\label{eq:bit-coeffs}
a_t = 1 - \lambda_t, \qquad b_t = \lambda_t,
\qquad \eta_t^2 = C(0,1)\,\lambda_t(1-\lambda_t).
\end{equation}
Under Assumption~\ref{assumption:1}, $\sigma$ is bounded below by a
positive constant. Consequently $\lambda_t\in(0,1)$ for every $t\in(0,1)$,
so $a_t>0$, $b_t>0$, and $\eta_t^2>0$ throughout the open interval.
Corollary~\ref{cor:tau-one} therefore applies whenever the scalar endpoint
representations are positively correlated; the paper does not assume that
this sign condition holds for every coordinate of the empirical text--image
representation.

\subsection{Scope}

Three caveats bound the reach of this analysis. First, real text--image couplings are not jointly Gaussian, and Gaussian mutual information captures only the linear component of the dependence between the endpoints; the construction here shows that the bound of Theorem~\ref{thm:mutual_info} \emph{can} be tight, not that it is tight for the empirical distribution of Section~\ref{sec:experiments}. Second, the calculation is scalar and does not establish an analogous full-interval result for general vector-valued jointly Gaussian endpoints. Third, the closed forms assume the idealized bridge of Equation~\ref{eqn:cond_dist_sample} rather than the learned drift $f_\theta$, so they describe the target process rather than a network approximation. We view the Gaussian case as identifying a non-degenerate regime with an explicit boundary, complementing the empirical evidence in Section~\ref{sec:cross_modal_variation}.

\newpage
\clearpage

\section{Construction of the Text Bridge Endpoint}
\label{app:text_bridge}

We tokenize each caption using the tokenizer of \texttt{Qwen/Qwen3-Embedding-8B}, without adding special tokens~\citep{zhang2025qwen3}. We store up to 64 tokens per caption. We pad the captions that are less than 64 tokens, so that we have constant dimensionality. %Up to 128 token identifiers are retained in the preprocessed dataset. 
Rather than storing contextual embeddings for every caption, we construct a vocabulary embedding table: each vocabulary item is independently passed through the \texttt{Qwen/Qwen3-Embedding-8B} model, and the corresponding hidden state is stored. Thus, each token identifier maps to a fixed, context-independent vector. This representation substantially reduces storage requirements and provides a stable, directly decodable correspondence between bridge vectors and vocabulary items.

%The preprocessing stage stores the first 128 dimensions of each vocabulary vector. For the Stable-Diffusion bridge geometry used in our experiments, we retain the first 64 caption tokens and the first 64 embedding dimensions. 
To match the Stable-Diffusion bridge geometry of our experiments, which requires $4\times32\times32=4096$-dimensional image tensors, we retain the first 64 embedding dimensions per token (this is supported, as Qwen3 uses Matryoshka Representation Learning~\citep{kusupati2022matryoshka}), making for the matching text representation dimension of $64\frac{\text{tokens}}{\text{caption}}\times64\frac{\text{dim}}{\text{token}}=4096\frac{\text{dim}}{\text{caption}}$. 

Each non-padding token vector is independently $\ell_2$-normalized and multiplied by
\[
    \sqrt{64}=8,
\]
while padding vectors are set exactly to zero. The measure-valued bridge construction permits these atomic padding coordinates directly; no Gaussian perturbation of the population endpoint law is required. The scaling gives non-padding embeddings approximately unit second moment per coordinate, bringing the numerical scale of the text endpoint closer to that of the image latent. It also preserves a clear norm-based distinction between content and padding.

% Let
% \[
%     \mathbf{E}\in\mathbb{R}^{64\times64}
% \]
% denote the resulting sequence of 64 token embeddings. Because
% \[
%     64\cdot64
%     =
%     4\cdot32\cdot32
%     =
%     4096,
% \]
% the text sequence and image latent contain exactly the same number of scalar values. 
We then reshape the $64\times64$ text representation into
\[
    \mathbf{x}_{\mathrm{text}}
    \in\mathbb{R}^{4\times32\times32}
\]
without projection, padding, duplication, or information loss. This way, the transformer model can process both images and text in the same ambient dimension.

More precisely, each 64-dimensional token embedding is divided into four contiguous 16-dimensional payloads. A $2\times2$ patch of a four-channel bridge tensor also contains
\[
    4\cdot2\cdot2=16
\]
values, so each text token occupies exactly four DiT patches. The bridge contains 256 patches in total, matching $64$ tokens times four patches per token. Under the row-major layout, token $i$ is assigned to patch indices
\[
    4i,\;4i+1,\;4i+2,\;4i+3
\]
in the raster ordering of the $16\times16$ patch grid. The patch payloads are then unpatchified to obtain the final $4\times32\times32$ text endpoint.

Row-major packing serves three purposes. First, it provides a simple, exactly invertible correspondence between token order, embedding coordinates, and transformer patches. Second, consecutive caption tokens remain consecutive in the transformer's linear patch sequence, giving the model a deterministic notion of textual order through the same positional encoding used for image patches. Third, it allows text and image endpoints to share precisely the same tensor shape, patchification procedure, and DiT backbone, avoiding modality-specific input projections that could confound comparisons. The inverse operation uses the same patch ordering to recover the original $64\times64$ token matrix before applying the shared token decoder.

\newpage
\clearpage

\section{Training Details}\label{app:training_details}

\noindent\textbf{Training hyperparameters.}
All models are trained on the same paired latent dataset for 200,000 steps with a global batch size of 512. We use AdamW with learning rate $1.5\times10^{-4}$, $(\beta_1,\beta_2)=(0.9,0.999)$, zero weight decay, gradient clipping at 1.0, 5,000 linear warm-up steps, and EMA decay 0.9995. Time is sampled uniformly from $[\epsilon, 1-\epsilon]$, where $\epsilon=9.9\times10^{-4}$ for the score models, and from $[0,1]$ for the flow model.

Score models use the same periodic-SDE (see Section L.2 of \cite{guo2026abc}) parameters $(\alpha,k,\epsilon_{\mathrm{SDE}})=(0.95,1.0,0.05)$, such that the volatility peaks in the middle and is smallest at the data endpoints. For the score-based models, the percentage of training samples used in the unconditional loss (Equations~\ref{eqn:loss_forward_cfg}, \ref{eqn:reverse_time_loss_cfg}) is $30\%$, and the percentage used in the conditional loss (Equations~\ref{eqn:loss_forward}, \ref{eqn:reverse_time_loss}) is $70\%$. All models use the same image-REPA objective with weight 0.5 at layer 8 and no warm-up~\citep{yu2024representation}. Only the defining transport objective, endpoint type, and generation direction vary, making the comparison controlled in data, capacity, optimization, and training budget.

\noindent\textbf{Numerical endpoints.}
Inference uses an endpoint-inclusive uniform Euler--Maruyama grid. Thus the first forward score evaluation is at $t=0$ and the first reverse score evaluation is at $t=1$, slightly outside the score models' truncated training interval. These boundary evaluations are extrapolations of the learned network. Moreover, a finite Euler--Maruyama discretization of the singular limiting bridge drift is not endpoint-exact; all reported finite-step outputs are numerical approximations, and no exact endpoint law is claimed for the discretized learned sampler.

\noindent\textbf{Hardware.}
Each model is trained using eight NVIDIA A100 80\,GB GPUs.

\newpage
\clearpage

\section{Model Architecture}
\label{app:model_architecture}

\subsection{Diffusion Transformer Backbone}
All experiments use the same DiTXA-L/2 (Diffusion Transformer with Cross-Attention)~\citep{peebles2023scalable} backbone ($716,089,872$ parameters), ensuring that the compared transport objectives have identical model capacity. Both endpoints of the bridge---the image latent and the reshaped text representation---have shape $4\times32\times32$. The model divides either endpoint into nonoverlapping $2\times2$ patches, producing $16\times16=256$ transformer tokens. DiTXA-L/2 consists of 24 transformer blocks with hidden dimension 1024, 16 attention heads, and an MLP expansion ratio of four.

Each block contains self-attention, cross-attention to a patchified conditioning endpoint, and a feed-forward network. Diffusion time is represented by a sinusoidal embedding followed by a two-layer MLP, while spatial position is represented using fixed two-dimensional sine--cosine embeddings. Time conditioning is injected through adaLN-Zero modulation. Queries and keys are normalized within each attention head, and the output head maps the final sequence of 256 transformer tokens back to a $4\times32\times32$ tensor. Separate conditioning embedders are used for the text-to-image and image-to-text directions (when the model has to learn both directions). The backbone design is otherwise unchanged across directions and ablations.

\subsubsection{Experimental Instantiation}
We train all compared baselines with the DiT-L setting, corresponding to 710M parameters, such that FLOPs and representational capacity are roughly equal.
We train a separate model for each direction (text-to-image and image-to-text) in the score-based frameworks, using the same underlying SDE noise schedule, because two separate models were more parameter- and compute-efficient in our implementation (mathematically, they represent opposite directions of the same endpoint-conditioned process). The flow baseline requires only one model because the deterministic ODE is bijective whenever the ODE admits a unique flow.

\subsection{Text Token Decoder}\label{app:token_decoder}

To map a generated text endpoint back to discrete tokens, we train a lightweight shared token decoder. The predicted $4\times32\times32$ bridge tensor is first chunked to recover 64 vectors in $\mathbb{R}^{64}$. Each vector is $\ell_2$-normalized and independently processed by the same position-shared MLP,
\[
    64 \longrightarrow 128 \longrightarrow 128
    \longrightarrow |\mathcal{V}|,
\]
with GELU activations, where $\mathcal{V}$ is the Qwen3 tokenizer vocabulary~\citep{zhang2025qwen3}. Sharing this decoder across sequence positions reduces its parameter count and encourages token identity to be encoded consistently regardless of position: since we use context-independent token embeddings, this is the right choice. The decoder is trained with token-level cross-entropy on valid caption positions; small Gaussian perturbations (from $\mathcal{N}\left(0, C(0, 0.01)I\right)$ for score-based or $\mathcal{N}\left(0, 0.01I\right)$ for flow-based models) are added to its input to improve robustness to imperfect text endpoints produced by the transport model. At inference time, we select the highest-logit token at each position and decode the resulting sequence with the Qwen tokenizer. Sequence termination is determined from the first embedding whose unscaled norm is below 0.1, exploiting the fact that padding positions are represented by exact zeros.

\newpage
\clearpage

\section{Dataset}
\label{app:dataset}

We train on the training split of GPIC~\citep{chandrasegaran2026gpic}. We retain examples with a nonempty caption, %of type \emph{tag}, \emph{short}, \emph{medium}, or \emph{long}, 
a successfully decoded image, and a minimum image dimension of at least 256 pixels. Images are resized isotropically so that the shorter side is 256 pixels and then center-cropped to $256\times256$. After filtering, the dataset contains approximately 99 million valid image--caption pairs. We reserve 0.5\% of these pairs as a validation set using a fixed random seed. We use the official held-out test set for our evaluations.

We encode each image with the posterior mean of the \texttt{stabilityai/sd-vae-ft-mse} VAE~\cite{rombach2022high}. This produces a latent
\[
    \mathbf{x}_{\mathrm{img}}\in\mathbb{R}^{4\times32\times32},
\]
which is %stored in half precision and 
scaled by the standard factor 0.18215 during training. We use posterior means rather than posterior samples so that each image has a deterministic representation. 
%Captions, token identifiers, token masks, and caption metadata are stored alongside the image latents in sharded memory-mapped arrays. 
For the REPA auxiliary regularization objective~\citep{yu2024representation}, we additionally precompute $16\times16=256$ patch-level DINOv2 ViT-B/14 features of dimension 768. These features are available for 15\% of the dataset, and the auxiliary loss is masked out for examples without them.

\subsection{Evaluation Splits}

We use GPIC's held-out testing split~\citep{chandrasegaran2026gpic} for our evaluations. For text-to-image evaluations, we use 50K samples. For image captioning evaluations, we use 10K samples (Table~\ref{tab:gen_results}). For stochastic variation experiments (Figure~\ref{fig:variation}), we use 500 samples.

\newpage
\clearpage

\section{Metrics}

\subsection{Raw Generation Quality Metrics}\label{subsec:gen_metrics}

Corresponds to Section~\ref{sec:gen_results}.
FID measures image quality \citep{heusel2017gans}, generative PPL with Qwen3-1.7B oracle measures text quality \citep{yang2025qwen3}, and CLIP score with ViT-B/16 \citep{radford2021learning} measures prompt adherence (for T2I and I2T).

\subsection{Cross-Modal Round Trip Metrics}\label{subsec:cross_modal_round_trip_metrics}

Corresponds to Section~\ref{sec:cross_modal_variation}.
To assess semantic fidelity of reconstructions, we measure cosine similarity in foundation model embedding space (DINOv2-L for images \citep{oquab2023dinov2}; Qwen3-Embedding 8B for text \citep{zhang2025qwen3}) of the restored/generated samples versus the source/ground truth sample: higher is better. To assess diversity, we measure the average cosine similarity among pairs of generated items given the same source data: when fidelity is less than 1, this metric should also be less than 1 (indicating that our model generates bona fide variants).

\newpage
\clearpage

\section{Cross-Modal Round-Trip Stochastic Variation}

% \subsection{Algorithm}

See Algorithm~\ref{alg:roundtrip-image-editing}.

\begin{algorithm}[h]
\caption{Cross-Modal Round-Trip Stochastic Variation}
\label{alg:roundtrip-image-editing}
\begin{algorithmic}[1]
\Require Source image $x^I$; corruption fraction $\rho\in[0,1]$;
Number of variations $V$;
Mode $m\in\{\textsc{Data-to-Data},\textsc{Noise-to-Data},\textsc{Flow}\}$;
Models $M_{I\rightarrow T}$ and $M_{T\rightarrow I}$;
Step size $dt$
\Ensure Edited images $\{\widetilde{x}_v^I\}_{v=1}^{V}$

\State $t_\rho\gets1-\rho$
\Comment{$t=1$: image; $t=0$: text or noise}

\For{$v=1,\ldots,V$}
    \If{$m=\textsc{Data-to-Data}$}
        \State $\begin{aligned}
        x_{t_\rho}^{(v)}
        \gets\operatorname{Solve}
        (M_{I\rightarrow T}, dt;
        x_{\mathrm{init}}=x^I,\,
        t_{\mathrm{start}}=1,\,
        t_{\mathrm{end}}=t_\rho,\,
        x_{\mathrm{cond}}=x^I,\,
        \mathrm{rev}=\mathrm{True})
        \end{aligned}$
        %\Comment{Transport toward text}

        \State $\begin{aligned}
        \widetilde{x}_v^I
        &\gets\operatorname{Solve}(M_{T\rightarrow I},dt;\\
        &\quad x_{\mathrm{init}}=x_{t_\rho}^{(v)},\,
        t_{\mathrm{start}}=t_\rho,\,
        t_{\mathrm{end}}=1,\\
        &\quad x_{\mathrm{cond}}=\varnothing,\,
        \mathrm{rev}=\mathrm{False})
        \end{aligned}$
        %\Comment{Restore toward image}

    \ElsIf{$m=\textsc{Noise-to-Data}$}
        \State $\epsilon_v\sim\mathcal N(0,\mathbf I)$
        \State $x_{t_\rho}^{(v)}
        \sim q_{t_\rho}^{M_{T\rightarrow I}}
        (x_t\mid x_0=\epsilon_v,x_1=x^I)$
        %\Comment{Analytic corruption}

        \State $\begin{aligned}\widetilde{x}_v^I
        \gets\operatorname{Solve}
        (M_{T\rightarrow I}, dt;
        x_{\mathrm{init}}=x_{t_\rho}^{(v)},\,
        t_{\mathrm{start}}=t_\rho,\,
        t_{\mathrm{end}}=1,\,
        x_{\mathrm{cond}}=\varnothing,\,
        \mathrm{rev}=\mathrm{False})
        \end{aligned}$
        %\Comment{Restore toward image}
    \ElsIf{$m=\textsc{Flow}$}
        \State $\begin{aligned}
        x_{t_\rho}^{(v)}
        &\gets\operatorname{SolveODE}(M_{\mathrm{flow}},dt;\,
        x_{\mathrm{init}}=x^I,\,
        t_{\mathrm{start}}=1,\,
        t_{\mathrm{end}}=t_\rho)
        \end{aligned}$
        \State $\begin{aligned}
        \widetilde{x}_v^I
        &\gets\operatorname{SolveODE}(M_{\mathrm{flow}},dt;\,
        x_{\mathrm{init}}=x_{t_\rho}^{(v)},\,
        t_{\mathrm{start}}=t_\rho,\,
        t_{\mathrm{end}}=1)
        \end{aligned}$
        \Comment{deterministic; repeated runs have no sampling diversity}
    \EndIf
\EndFor

\State \Return $\{\widetilde{x}_v^I\}_{v=1}^{V}$
\end{algorithmic}
\end{algorithm}

Algorithm~\ref{alg:roundtrip-image-editing} is printed for an image source. For a text source, set $t_\rho=\rho$, integrate the endpoint-conditioned forward model from $0$ to $t_\rho$, and then integrate the state-only reverse model from $t_\rho$ back to $0$; the noise-to-data construction is obtained by the same endpoint swap. The flow comparison uses the single unconditional field in both directions with the signed Euler update
$x_{k+1}=x_k+(t_{k+1}-t_k)\widehat u(x_k,t_k)$.

\newpage
\clearpage

\section{Industrial-Scale Model Details}

In Table~\ref{tab:cfg}, we compare our 1B model against foundation models. We evaluate the foundation models on the same testing set as our models.

\noindent\textbf{Qwen-3-VL-4B-Instruct:} This was used to generate the captions by the creators of the GPIC dataset~\citep{chandrasegaran2026gpic}. It is therefore a natural teacher/reference anchor for performance on this caption distribution, though not a mathematical upper bound under the reported evaluation metrics. Also, even our largest 1B model still has only a quarter of the parameters. As such, it is encouraging that the performance is as close as it is.

\noindent\textbf{Stable Diffusion 1.5:} This model uses the same VAE that we do. We used the community standard settings: CFG 7.5, 50 steps, and default settings from the HuggingFace pipeline~\citep{rombach2022high}.

\newpage
\clearpage

\section{Technical Details of Bridge-Based Image Editing}
\label{app:bridge_editing}

\paragraph{Setting.}
Let $x_0$ denote a text state and $x_1$ an image latent, represented in the common tensor layout used by our data-to-data bridge. We require either one bidirectional bridge model or two direction-specific models sharing the same transport process. This experiment is specific to a data-to-data bridge because the text and image are themselves the two endpoints of the process.

Our implementation uses a driftless reference SDE ($A=0$), with scalar volatility $\sigma(t)$. The learned forward process is discretized by Euler--Maruyama as
\begin{equation}
    X_{i+1}
    =
    X_i
    +
    \sigma(t_i)^2
    s_\theta(X_i,t_i;x_0,y)\Delta t_i
    +
    \sigma(t_i)\Delta B_i,
    \label{eq:forward_bridge_discrete}
\end{equation}
where $s_\theta$ is the learned forward score, $x_0$ is also supplied as the conditioning endpoint, and $y$ denotes the prompt-type label.

\paragraph{Inferring text from the source image.}
Given a source-image latent $z$, we first integrate the reverse image-to-text bridge from $t=1$ to $t=0$, conditioned on $z$. The resulting continuous state may not correspond exactly to valid token embeddings, due to network approximation and discretization error. We therefore decode a token ID at each position using the model's token decoder (see Appendix~\ref{app:token_decoder}), infer the sequence length using the token-norm stopping criterion (norms close to zero are padding), and snap every retained position to its vocabulary-table embedding.

Specifically, if $\hat{\ell}_j$ is the decoded token at position $j$, its snapped embedding is
\begin{equation}
    e_j
    =
    \rho
    \frac{V_{\hat{\ell}_j}}
         {\lVert V_{\hat{\ell}_j}\rVert_2},
    \label{eq:vocab_snap}
\end{equation}
where $V$ is the vocabulary embedding table, and $\rho$ is the token scaling factor used during training. Positions after the inferred stopping point are set to zero. This produces both a discrete pseudo-caption and an on-manifold text endpoint $x_0^{(0)}$. No dataset caption is used by the editing pipeline.

\paragraph{Automated semantic editing.}
We prompt an LLM to rewrite the inferred caption by replacing $k$ words while leaving the remainder unchanged. We sample several candidate responses and retain the one whose token-level difference is closest to the requested edit size. Thus, the LLM provides an automated and reproducible proxy for a human semantic edit; neither the particular LLM nor exact satisfaction of the requested number of word replacements is required by the bridge procedure.

The edited caption is tokenized using the same tokenizer (Qwen3-Embedding) used to generate the dataset, and re-embedded using the same normalized vocabulary table as in Eq.~\ref{eq:vocab_snap}. Let $e_j^{(0)}$ and $e_j^{(1)}$ denote unit-normalized embeddings at a token position in the original and edited captions. At interpolation coordinate $\lambda\in[0,1]$, positions present in both sequences are interpolated along the unit sphere~\citep{slerp}:
\begin{align}
    \omega_j
        &= \arccos\!\left(
            \left\langle e_j^{(0)},e_j^{(1)}\right\rangle
           \right), \\
    e_j^{(\lambda)}
        &=
        \frac{\sin((1-\lambda)\omega_j)}{\sin\omega_j}e_j^{(0)}
        +
        \frac{\sin(\lambda\omega_j)}{\sin\omega_j}e_j^{(1)}.
    \label{eq:caption_slerp}
\end{align}
For numerical stability, whenever $|\sin\omega_j|$ is numerically small (below $10^{-6}$), we fall back to $e_j^{0}$.
% Whenever $|\sin\omega_j|$ is numerically small, the closed form in Equation~\ref{eq:caption_slerp} requires a stable fallback: coincident vectors remain fixed, while a near-antipodal pair requires a deterministic choice of great-circle direction. 
%\textcolor{red}{\textbf{Implementation note:} the current public helper fixes all small-$|\sin\omega_j|$ pairs; update it with the chosen-geodesic near-antipodal fallback, or verify and state that such pairs are excluded.} 
If a position occurs in only one sequence, its embedding magnitude is linearly faded out or in. The resulting embeddings are multiplied by $\rho$ and rearranged into the bridge tensor layout, producing text endpoints
$\{x_0^{(\lambda)}\}_{\lambda\in\Lambda}$.

\paragraph{Noise backsolving.}
Using independently sampled noise for each $\lambda$ would obscure the effect of the edit with unrelated changes in composition and layout. Instead, we construct a reference path from the unedited text endpoint $x_0^{(0)}$ to the source latent $z$ and infer the stochastic increments that make this path consistent with the learned discretized dynamics.

Let
\begin{equation}
    C(a,b)=\int_a^b \sigma(s)^2\,ds.
\end{equation}
Let $t_i = \frac{i}{N}$. For $t_i<t_{i+1}$, define $v_i=C(t_i,t_{i+1})$ and
$V_i=C(t_i,1)$. The conditional transition of the driftless reference bridge is (following Equation~\ref{eqn:cond_dist_sample}):
\begin{equation}
    q_{\mathrm{ref}}
    \left(
        x_{i+1}\mid x_i,z
    \right)
    =
    \mathcal{N}\left(
        x_i+\frac{v_i}{V_i}(z-x_i),
        \left(v_i-\frac{v_i^2}{V_i}\right)\mathbf{I}
    \right).
    \label{eq:reference_transition}
\end{equation}
Sequentially sampling this transition gives a path
$\bar{x}_0=x_0^{(0)},\ldots,\bar{x}_N=z$. %The final transition is degenerate and therefore terminates exactly at the source latent.

For each step, we rearrange Eq.~\ref{eq:forward_bridge_discrete} to infer its driving increment:
\begin{equation}
    \widehat{\Delta B}_i
    =
    \frac{\bar{x}_{i+1}-\bar{x}_i}{\sigma(t_i)}
    -
    \sigma(t_i)
    s_\theta(\bar{x}_i,t_i;x_0^{(0)},y)\Delta t_i.
    \label{eq:noise_backsolve}
\end{equation}
These quantities are best interpreted as backsolved driving residuals: because the reference path is sampled first and then reconciled with the learned drift, they need not be fresh independent Brownian draws. More precisely, the sampled reference path follows the endpoint-pinned base law $\mathbb{P}^{(x_0^{(0)},z)}$, while replay uses the learned data-generating drift conditioned only on $x_0^{(0)}$. Equation~\ref{eq:noise_backsolve} algebraically reconciles those two discretized dynamics; it is not a discretization of a change of measure between the unconditioned laws $\mathbb{P}$ and $\mathbb{Q}$ (cf.\ Equation~35 of \cite{guo2026abc}).

\begin{algorithm}[t]
\caption{Noise backsolving and structure-preserving replay}
\label{alg:noise_backsolve}
\begin{algorithmic}[1]
\Require Source latent $z$; original text $x_0^{(0)}$; edited text endpoints
         $\{x_0^{(\lambda)}\}_{\lambda\in\Lambda}$; score $s_\theta$;
         volatility $\sigma$; label $y$; grid $0=t_0<\cdots<t_N=1$
\State $\bar{x}_0 \gets x_0^{(0)}$
\For{$i=0,\ldots,N-1$}
    \State Sample
    $\bar{x}_{i+1}\sim
      q_{\mathrm{ref}}(\cdot\mid\bar{x}_i,z)$
      using Eq.~\ref{eq:reference_transition}
    \State $\Delta t_i\gets t_{i+1}-t_i$
    \State $\displaystyle
      \widehat{\Delta B}_i\gets
      \frac{\bar{x}_{i+1}-\bar{x}_i}{\sigma(t_i)}
      -
      \sigma(t_i)s_\theta(
          \bar{x}_i,t_i;x_0^{(0)},y)\Delta t_i$
\EndFor
\For{$\lambda\in\Lambda$}
    \State $X_0^{(\lambda)}\gets x_0^{(\lambda)}$
    \For{$i=0,\ldots,N-1$}
        \State $\displaystyle
        X_{i+1}^{(\lambda)}
        \gets
        X_i^{(\lambda)}
        +
        \sigma(t_i)^2
        s_\theta(
            X_i^{(\lambda)},t_i;x_0^{(\lambda)},y)\Delta t_i
        +
        \sigma(t_i)\widehat{\Delta B}_i$
    \EndFor
    \State $\hat{z}^{(\lambda)}\gets X_N^{(\lambda)}$
\EndFor
\Ensure Structurally aligned edited latents
        $\{\hat{z}^{(\lambda)}\}_{\lambda\in\Lambda}$
\end{algorithmic}
\end{algorithm}

For $\lambda=0$, replay starts from the same text endpoint used during backsolving. Equation~\ref{eq:noise_backsolve} therefore makes every replayed update equal to its corresponding reference-path update, recovering the source latent up to numerical precision. For $\lambda>0$, the initial state and score evaluations change, but the inferred driving increments remain fixed. Since these increments were obtained from a path anchored at the source image, replay tends to preserve its coarse spatial organization while allowing the edited text to alter semantic content.

\paragraph{Experimental configuration.}
We use $N=400$ steps for both image-to-text inference and text-to-image generation, five evenly spaced interpolation angles including both caption endpoints, and classifier-free guidance scale zero. %Noise backsolving requires generation to begin at $t=0$; it is therefore not combined with an intermediate SDEdit-style initialization. 
%Different source images receive independently inferred paths, while 
Every interpolation point for a given source reuses exactly the same inferred increment sequence from Equation~\ref{eq:noise_backsolve}. We decode the final latents with the pretrained image VAE and display each generated image together with the decoded text at its interpolation coordinate.

\newpage
\clearpage

\section{LARRY state-fate benchmark details}\label{sec:larry_details}

\noindent\textbf{Metrics}
We evaluate both pointwise accuracy and distributional structure. Forward and reverse mean-squared error measure how well the model reconstructs paired descendant and progenitor states. RBF-MMD compares the generated and real cell-state distributions, testing whether the model matches the target population. Fate and clone kNN accuracy measure whether each generated cell is nearest to real cells with the correct fate label or lineage clone, which directly probes biological consistency. Finally, cycle-consistency error evaluates whether a generated descendant can be mapped back to its original progenitor, and vice versa. Together, these metrics separate models that merely generate plausible cells from ones that preserve lineage-resolved state-fate relationships.

\noindent\textbf{Dataset Processing}
We construct clone-paired endpoint examples from the in vitro LARRY lineage-tracing dataset by pairing day-2 progenitor cells with day-6 descendant cells that share the same lineage barcode. We preprocess normalized counts by selecting the top 2,000 highly variable genes and projecting cells to a 64-dimensional latent representation using TruncatedSVD. For each eligible clone, we sample up to 32 day-2/day-6 endpoint pairs and split data by clone into train, validation, and test partitions, preventing lineage leakage across splits. The resulting benchmark contains 31,840 training pairs and 3,968 evaluation pairs.

\noindent\textbf{Architecture}
All compared neural methods use the same vector DiT backbone unless otherwise stated. The model tokenizes the 64-dimensional cell latent into 8-dimensional tokens, embeds them with width 768, and applies 12 transformer blocks with 12 attention heads and AdaLN-style conditioning. The conditioning vector includes the bridge time, the experimental/context label, a forward/reverse direction embedding, and the source endpoint state. The shared \bit{} model uses one bidirectional network for both day-2$\to$day-6 and day-6$\to$day-2 transport; the separate-model ablation uses independent networks for the two directions.

\noindent\textbf{Training Details}
We train all models for 20,000 optimization steps with AdamW, learning rate $2\times 10^{-4}$, batch size 512, gradient clipping at 1.0, and bfloat16 mixed precision. 
%Evaluations are run every 5,000 steps on a held-out batch of 512 pairs. 
Score-based bridges are sampled with 250 SDE steps. The main \bit{} model uses the uniform-volatility bridge with $K=0.5$; the cosine baseline uses a cosine-decay volatility schedule with $\epsilon=0.03$; the rectified-flow baseline uses the same DiT architecture with a flow-matching objective; and endpoint regression directly predicts the target endpoint from the source endpoint.

% \textcolor{red}{\textbf{Author query:} Were the periodic 512-pair evaluations used only for checkpoint selection, with Table~\ref{tab:larry_state_fate} then computed once on all 3,968 test pairs? Please state the validation/test split sizes, checkpoint-selection rule, and final evaluation sample explicitly.}

\newpage
\clearpage

\section{Training Pseudocode}\label{app:train_alg}

See Algorithm~\ref{alg:training}.

\begin{algorithm}[h]
\caption{Training the bidirectional bridge models}
\label{alg:training}
\begin{algorithmic}[1]
\Require Paired images and captions $(I,c)$; transport type
$\mathcal{M}\in\{\textsc{Score},\textsc{Flow}\}$; endpoint mode
$\mathcal{E}\in\{\textsc{Data-to-Data},\textsc{Noise-to-Image},
\textsc{Noise-to-Text}\}$; enabled directions
$\mathcal{D}\subseteq\{\rightarrow,\leftarrow\}$; network
$f_\theta$; text decoder $D_\psi$
\State $\mathbf{z}_{T}\gets\Call{EncodeText}{c}$
\Comment{row-major text endpoint in $\mathbb{R}^{4\times32\times32}$}
\State $\mathbf{z}_{I}\gets\Call{VAEEncode}{I}$
\Comment{image endpoint}
\State $\mathbf{c}_{0}\gets\mathbf{z}_{T}$,\quad
       $\mathbf{c}_{1}\gets\mathbf{z}_{I}$
\Comment{clean conditioning endpoints}
\State $\mathbf{x}_{0}\gets
    \begin{cases}
    \mathbf{z}_{T}, & \mathcal{E}\in\{\textsc{Data-to-Data},
      \textsc{Noise-to-Text}\}\\
    \boldsymbol{\epsilon}_{T},\ 
      \boldsymbol{\epsilon}_{T}\sim\mathcal{N}(0,\mathbf{I}),
      & \mathcal{E}=\textsc{Noise-to-Image}
    \end{cases}$
\State $\mathbf{x}_{1}\gets
    \begin{cases}
    \mathbf{z}_{I}, & \mathcal{E}\in\{\textsc{Data-to-Data},
      \textsc{Noise-to-Image}\}\\
    \boldsymbol{\epsilon}_{I},\
      \boldsymbol{\epsilon}_{I}\sim\mathcal{N}(0,\mathbf{I}),
      & \mathcal{E}=\textsc{Noise-to-Text}
    \end{cases}$
\State Sample $t$ from the configured training-time distribution
\If{$\mathcal{M}=\textsc{Flow}$}
    \State \textbf{assert} $\mathcal{E}=\textsc{Data-to-Data}$
    \State $\mathbf{x}_{t}\gets(1-t)\mathbf{x}_{0}+t\mathbf{x}_{1}$
    \State $\mathbf{u}^{\star}\gets\mathbf{x}_{1}-\mathbf{x}_{0}$
    \State $\widehat{\mathbf{u}}\gets f_{\theta}(\mathbf{x}_{t},t)$
    \Comment{one unconditional velocity field}
    \State $\mathcal{L}_{\mathrm{bridge}}\gets
      \|\widehat{\mathbf{u}}-\mathbf{u}^{\star}\|_{2}^{2}$
    \If{an image REPA target is available}
        \State $\mathcal{L}_{\mathrm{bridge}}\gets
        \mathcal{L}_{\mathrm{bridge}}
        +\lambda_{\mathrm{REPA}}\,
        \mathcal{L}_{\mathrm{cos}}
        (R_{\theta}(\mathbf{x}_{t},t),\operatorname{DINO}(I))$
    \EndIf
\Else
    \State $\mathbf{x}_{t}\sim
       p_{\mathrm{br}}(\mathbf{x}_{t}\mid\mathbf{x}_{0},\mathbf{x}_{1})$
    \Comment{sample the Gaussian SDE bridge from Equation~\ref{eqn:cond_dist_sample}}
    \State $\mathcal{L}_{\mathrm{bridge}}\gets0$
    \For{$d\in\mathcal{D}$}
        \State $\mathbf{c}_{d}\gets
          \mathbf{c}_{0}$ if $d=\rightarrow$, otherwise $\mathbf{c}_{1}$
        \State Sample conditioning mask
        $m\sim\operatorname{Bernoulli}(1-p_{\mathrm{uncond}})$
        \If{$d=\rightarrow$}
            \State $\mathbf{s}^{\star}_{d}\gets
            \dfrac{\mathbf{x}_{1}-\mathbf{x}_{t}}{a(1)-a(t)}$
            \Comment{$\nabla_{\mathbf{x}_{t}}\log
            p_{\mathrm{base}}(\mathbf{x}_{1}\mid\mathbf{x}_{t})$}
        \Else
            \State $\mathbf{s}^{\star}_{d}\gets
            \dfrac{\mathbf{x}_{0}-\mathbf{x}_{t}}{a(t)}$
            \Comment{$\nabla_{\mathbf{x}_{t}}\log
            p_{\mathrm{base}}(\mathbf{x}_{t}\mid\mathbf{x}_{0})$}
        \EndIf
        \State $\widehat{\mathbf{s}}\gets
          f_{\theta}(\mathbf{x}_{t},t,\mathbf{c}_{d},d,m)$
        \State $\mathcal{L}_{d}\gets
          w_{d}(t)\|
          \widehat{\mathbf{s}}-\mathbf{s}^{\star}_{d}\|_{2}^{2}$
        \State $\mathcal{L}_{\mathrm{bridge}}\gets
          \mathcal{L}_{\mathrm{bridge}}+\mathcal{L}_{d}$
        \If{an image REPA target is available}
            \State $\mathcal{L}_{\mathrm{bridge}}\gets
            \mathcal{L}_{\mathrm{bridge}}
            +\lambda_{\mathrm{REPA}}\,
            \mathcal{L}_{\mathrm{cos}}
            (R_{\theta}(\mathbf{x}_{t},t),\operatorname{DINO}(I))$
        \EndIf
    \EndFor
\EndIf
\State Update $\theta$ using
$\nabla_{\theta}\mathcal{L}_{\mathrm{bridge}}$
\If{the text endpoint is data rather than noise}
    \State $\widetilde{\mathbf{z}}_{T}\gets
      \mathbf{z}_{T}+\sigma_{\mathrm{dec}}\boldsymbol{\epsilon}$
    \State $\widehat{\mathbf{p}}\gets
      D_{\psi}(\widetilde{\mathbf{z}}_{T})$
    \State $\mathcal{L}_{\mathrm{dec}}\gets
      \operatorname{CE}(\widehat{\mathbf{p}},
      \operatorname{Tokenize}(c))$
    \State Update $\psi$ using
      $\nabla_{\psi}\mathcal{L}_{\mathrm{dec}}$
\EndIf
\State Update EMA parameters:
$\bar{\theta}\gets
\beta\bar{\theta}+(1-\beta)\theta$
\end{algorithmic}
\end{algorithm}

\newpage
\clearpage

\section{Inference Pseudocode}\label{app:inference_alg}

See Algorithm~\ref{alg:inference}.

\begin{algorithm}[h]
\caption{Inference through the learned bridge}
\label{alg:inference}
\begin{algorithmic}[1]
\Require Source $a$; direction $d\in\{\rightarrow,\leftarrow\}$;
transport type $\mathcal{M}$
\Statex \hspace{\algorithmicindent}EMA network $f_{\bar{\theta}}$;
integration steps $N$; guidance strength $\gamma$
\State $\Delta t\gets1/N$
\If{$d=\rightarrow$}\Comment{Text-to-Image}
    \State $\mathbf{c}\gets\Call{EncodeText}{a}$
    \State $\mathbf{x}\gets
      \mathbf{c}$ for data-to-data transport, otherwise
      $\mathbf{x}\sim\mathcal{N}(0,\mathbf{I})$
    \State $\rho\gets+1$
\Else\Comment{Image-to-Text}
    \State $\mathbf{c}\gets\Call{VAEEncode}{a}$
    \State $\mathbf{x}\gets
      \mathbf{c}$ for data-to-data transport, otherwise
      $\mathbf{x}\sim\mathcal{N}(0,\mathbf{I})$
    \State $\rho\gets-1$
\EndIf
\For{$n=0,\ldots,N-1$}
    \State $u_n\gets n/N$
    \State $t_n\gets u_n$ if $d=\rightarrow$, otherwise $t_n\gets1-u_n$
    \Comment{$u_n$: integration time; $t_n$: network time}
    \If{$\mathcal{M}=\textsc{Flow}$}
        \State $\widehat{\mathbf{u}}\gets
          f_{\bar{\theta}}(\mathbf{x},t_n)$
        \Comment{single unconditional field}
        \State $\mathbf{x}\gets
          \mathbf{x}+\rho\,\Delta t\,\widehat{\mathbf{u}}$
    \Else
        \State $\mathbf{s}_{c}\gets
          f_{\bar{\theta}}(\mathbf{x},t_n,\mathbf{c},d,
          \mathrm{conditioned})$
        \If{$\gamma>0$}
            \State $\mathbf{s}_{u}\gets
              f_{\bar{\theta}}(\mathbf{x},t_n,\mathbf{c},d,
              \mathrm{unconditioned})$
            \State $\widehat{\mathbf{s}}\gets
              \mathbf{s}_{c}+
              \gamma(\mathbf{s}_{c}-\mathbf{s}_{u})$
        \Else
            \State $\widehat{\mathbf{s}}\gets\mathbf{s}_{c}$
        \EndIf
        \State Sample $\boldsymbol{\epsilon}\sim\mathcal{N}(0,\mathbf{I})$
        \State $\mathbf{x}\gets\mathbf{x}
          +\sigma(t_n)^2\widehat{\mathbf{s}}\,\Delta t
          +\sigma(t_n)\sqrt{\Delta t}\,\boldsymbol{\epsilon}$
        \Comment{Euler--Maruyama bridge step}
    \EndIf
\EndFor
\Statex \Comment{The finite-$N$ output is a discretization approximation, not an endpoint-exact draw.}
\If{$d=\rightarrow$}
    \State $\widehat{I}\gets
      \Call{VAEDecode}{\mathbf{x}}$
    \State \Return $\widehat{I}$
\Else
    \State $(\mathbf{E}_{1},\ldots,\mathbf{E}_{64})
      \gets\Call{UnpackTextVectors}{\mathbf{x}}$
    \Comment{vectors before decoder normalization}
    \State $\widehat{\mathbf{k}}\gets
      \arg\max D_{\psi}(\mathbf{E}_{1:64})$
    \State $\ell\gets
      \min\{i:\|\mathbf{E}_{i}/\rho_{\mathrm{tok}}\|_{2}<0.1\}$
    \Comment{$\rho_{\mathrm{tok}}$: training-time token scale}
    \State \Return
      $\Call{QwenDecode}{\widehat{\mathbf{k}}_{1:\ell-1}}$
\EndIf
\end{algorithmic}
\end{algorithm}

\newpage
\clearpage

%\section{Additional Experiments}

\section{Large-Scale Model Hyperparameters}
\label{app:large_scale_hparams}

To see how far we can push the performance, we also train a larger model with more compute resources. Overall, it is an improvement (compare Table~\ref{tab:cfg}'s results in row $\omega=0.0$ to Table~\ref{tab:gen_results}).

\noindent\textbf{Changed Settings:} The large-scale model uses DiTXA-XL/2, comprising 28 transformer blocks with hidden dimension 1152, 16 attention heads, and $2\times2$ patches. This results in $1,054,621,840$ parameters. (This is larger than the DiTXA-L/2 backbone used in the controlled comparisons, which has 24 blocks and hidden dimension 1024.) The large-scale model is trained jointly in both the text-to-image and image-to-text directions: each batch produces one loss for each direction, and the two losses are summed before a single optimizer update. %In contrast, the controlled directional ablations train separate models for the two directions. 

Endpoint conditioning is dropped with probability 0.1, compared with 0.3 for the controlled score-model ablations. Image REPA is applied from the beginning of training at transformer layer 8 with weight 0.75 (the controlled image-generating models use the same layer and schedule but a smaller weight of 0.5).
We train the large-scale model for 122,500 steps with a global batch size of 1792 (compared with 200,000 steps and a batch size of 512 in the controlled comparisons). This corresponds to approximately twice as many training examples as the ablation models saw.
Optimization uses AdamW with learning rate $2.5\times10^{-4}$, larger than the controlled-model rate of $1.5\times10^{-4}$; this value was selected for the larger-batch run rather than obtained by a stated linear-scaling rule.

\noindent\textbf{Kept Settings:} The things that stay the same are the AdamW $(\beta_1,\beta_2)=(0.9,0.999)$, zero weight decay, gradient clipping at 1.0, 5,000 steps of linear warm-up followed by a constant learning rate, and primary EMA decay of 0.9995. 
The large-scale model also uses the same periodic bridge SDE with $\alpha=0.95,k=1.0,\epsilon_{\mathrm{SDE}}=0.05$.
These parameters match the periodic-SDE arm of the main text experiments. Bridge times are sampled from a logit-normal distribution with mean 0.4 and standard deviation 0.7 (as opposed to the uniform time sampling in the main text ablation experiments), clipped to
$[9.9\times10^{-4},1-9.9\times10^{-4}]$. 
The DINOv2 targets, token-bridge construction, row-major text layout, dataset, VAE representation, and token-decoder design are otherwise unchanged.

Training is distributed over 16 compute nodes with four NVIDIA A100 80\,GB GPUs per node, for 64 GPUs in total. The controlled comparisons use two nodes and eight GPUs. %The resulting per-GPU batch sizes are 28 and 64, respectively, while synchronization preserves the stated global batch size in each setting. These scale-specific changes were introduced to maximize the quality of the primary model; they are deliberately excluded from the controlled comparisons, where architecture, compute, batch size, optimization, and training duration are held fixed across ablations.

\newpage
\clearpage

\section{Image Editing Along Text--Image Bridges}
\label{sec:bridge_editing}
% \begin{wrapfigure}{r}{0.5\linewidth}
%     \centering
%     \includegraphics[width=\linewidth, trim={0 1.5cm 0 0},clip]{results/prompt_editing/panel_img0036.png}
%     \caption{\textbf{Text Bridge Editing:} See App. \ref{sec:cycle_consistent_editing}.}\label{fig:cycle_editing}
% \end{wrapfigure}
We use our text--image bridges to construct an end-to-end image editing procedure that exposes editing trajectories native to the learned bridge. Given only a source image, we first traverse the bridge in the reverse direction, from image to text, to infer a textual description. 
%We decode this inferred state into tokens and project it onto the vocabulary embedding manifold. 
%Consequently, the procedure does not require a ground-truth caption or paired multimodal input at inference time.
We then modify the inferred description and traverse the bridge forward, using the same implied noise, to generate an edited image. In our experiments, an LLM (which we use as an automated proxy for human edits) is prompted to replace a small number of words while preserving the remainder of the sentence. %The use of an LLM is not fundamental to the method: it is merely an automated proxy for human edits. 
Both the original and edited descriptions are embedded using the bridge's vocabulary table, and we form a continuous path between them using token-wise spherical interpolation to get a spectrum of edits. %Forward bridge trajectories conditioned on points along this path produce a sequence of corresponding image edits.
% A naive comparison of these generations would conflate changes in the text with stochastic sampling variation. We therefore infer the stochastic driving increments associated with a trajectory connecting the original inferred text to the source image:%. For a discretized forward bridge with learned score $s_\theta$, we backsolve each increment as
% \begin{equation}
%     \widehat{\Delta B}_i
%     =
%     \frac{\bar{x}_{i+1}-\bar{x}_i}{\sigma(t_i)}
%     -
%     \sigma(t_i)
%     s_\theta(\bar{x}_i,t_i;x_0,y)\Delta t_i,
%     \label{eq:noise_backsolve_main}
% \end{equation}
% where $\{\bar{x}_i\}$ is a bridge path terminating at the source-image latent. We replay the same inferred increments for every interpolated text condition. The unedited trajectory therefore returns to the source latent, while edited trajectories inherit a closely aligned stochastic and spatial structure. This isolates semantic changes induced through the text endpoint and 
Overall, this experiment illustrates a distinctive capability of text--image bridges: an image can be inverted into an editable semantic representation and then transported back through a continuous family of structurally aligned image edits. See Figure~\ref{fig:extra_editing} for qualitative results, and Sec.~\ref{app:bridge_editing} for setup.

% \section{Additional Text-Image Bridge Editing Results}\label{sec:cycle_consistent_editing}

% See Figure~\ref{fig:extra_editing}.

\begin{figure}[ht]
    \includegraphics[width=\linewidth, trim={0 1.5cm 0 0},clip]{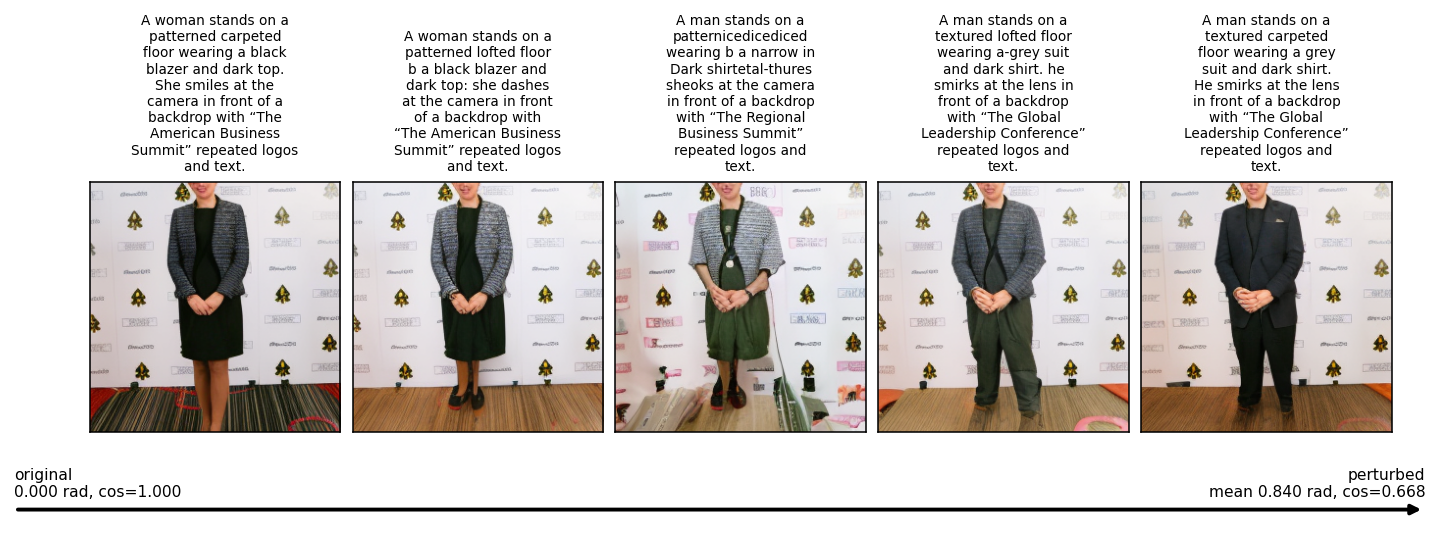}
    \includegraphics[width=\linewidth, trim={0 1.5cm 0 0},clip]{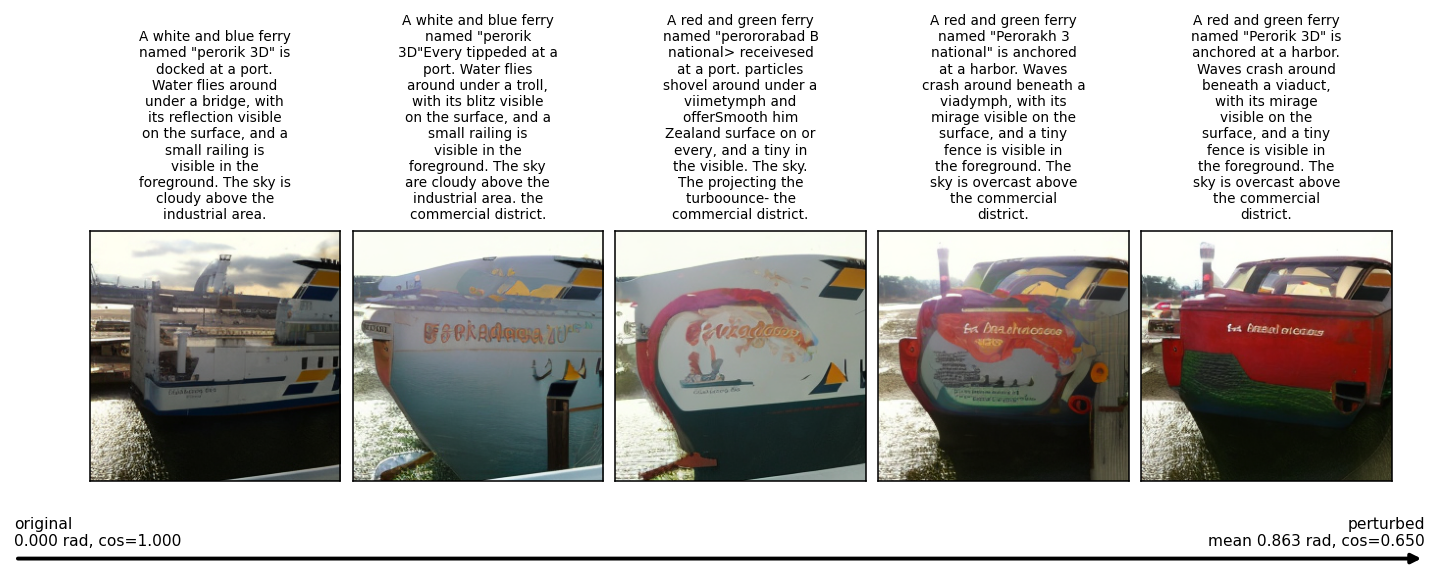}
    \includegraphics[width=\linewidth, trim={0 1.5cm 0 0},clip]{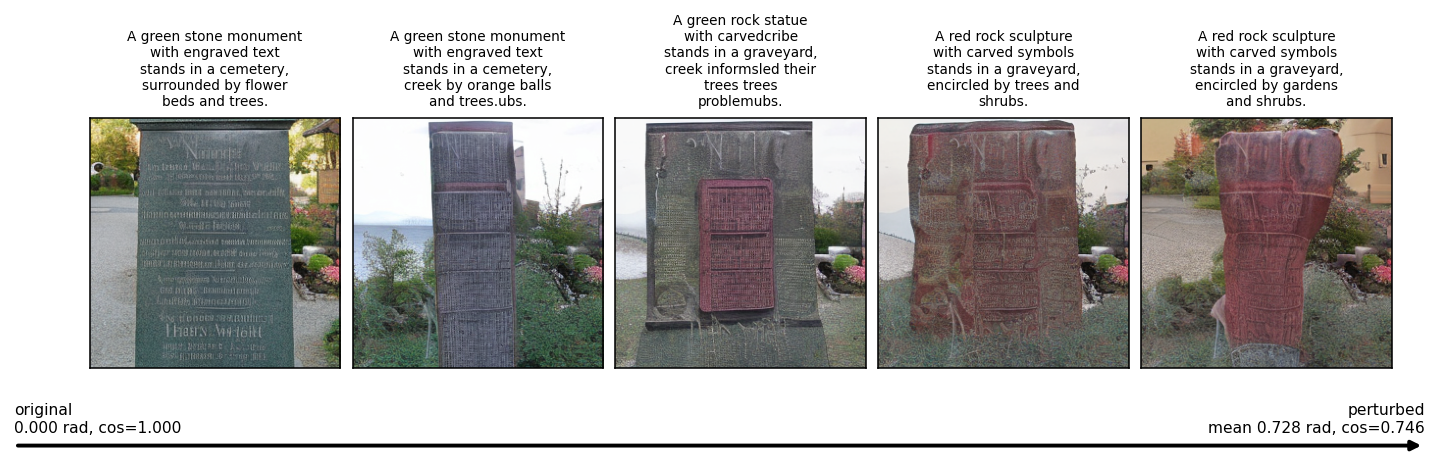}
    \includegraphics[width=\linewidth, trim={0 1.5cm 0 0},clip]{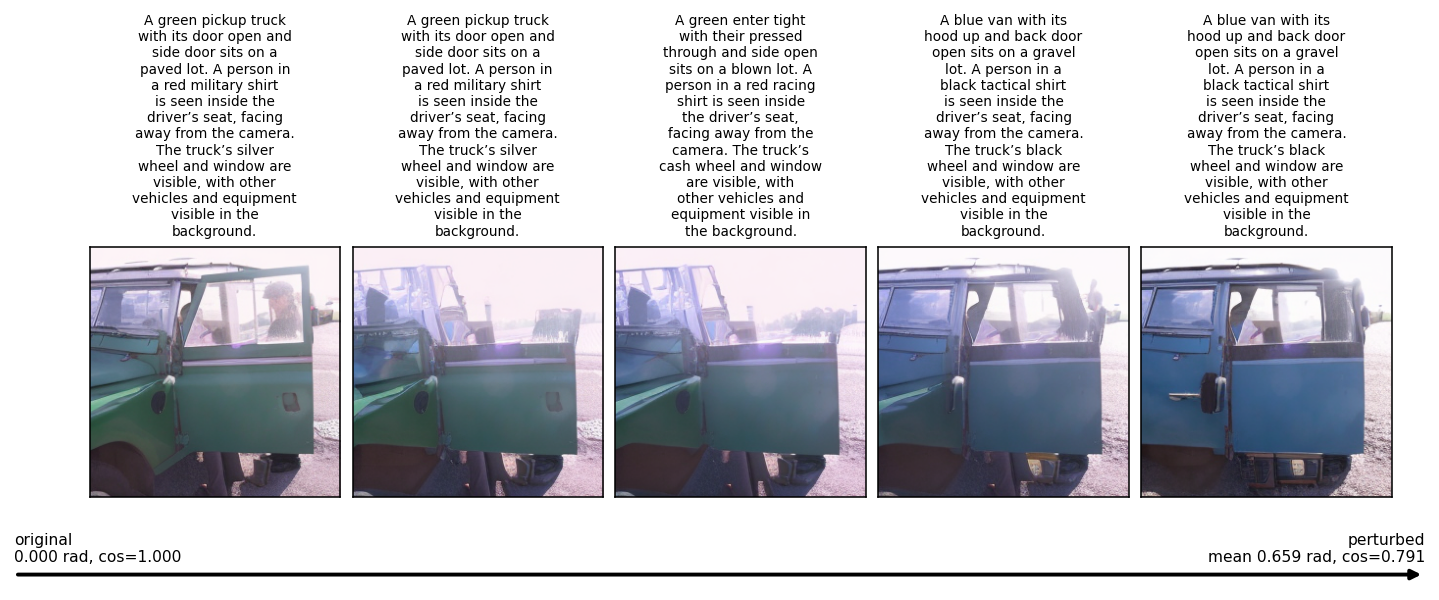}
    \caption{\textbf{Image Editing via Text Bridge.}}\label{fig:extra_editing}
\end{figure}

\newpage
\clearpage

\section{Impact of REPA on Image-to-Text Generation}\label{sec:repa}

REPA~\citep{yu2024representation} has been known to greatly improve image generation quality. Surprisingly, we find that using REPA with \textit{image features} can also help the image-to-\textit{text} generation in our \bit{} framework. In particular, for only a slight hit in generative perplexity, we see a large boost in CLIP score. We think this is because the REPA loss forces the generated text to align with the image's semantic content, but at the cost of missing some fine syntactical details. See Table~\ref{tab:text_gen_repa}. 

\begin{table}[h]
    \centering
    \begin{tabular}{lcc}
        \toprule
        Method & Gen PPL ($\downarrow$) & CLIP ($\uparrow$) \\
        \midrule
        \textbf{\bit{}} (without REPA) & 112.7 & 22.7 \\
        \textbf{\bit{}} (REPA $\lambda=0.5$) & 123.4 & 27.0 \\
        \bottomrule
    \end{tabular}
    \caption{\textbf{Impact of REPA on Image-to-Text:} Same exact training and inference settings, except that one was trained with and one was trained without REPA.}
    \label{tab:text_gen_repa}
\end{table}

\newpage
\clearpage

\section{Qualitative Cross-Modal Round-Trip Stochastic Variation Results}\label{sec:qual_results_variation}

\begin{figure}[h]
\centering
\begin{subfigure}[]{0.49\textwidth}
    \includegraphics[width=\linewidth, trim={0 0 0 0.08\linewidth},clip]{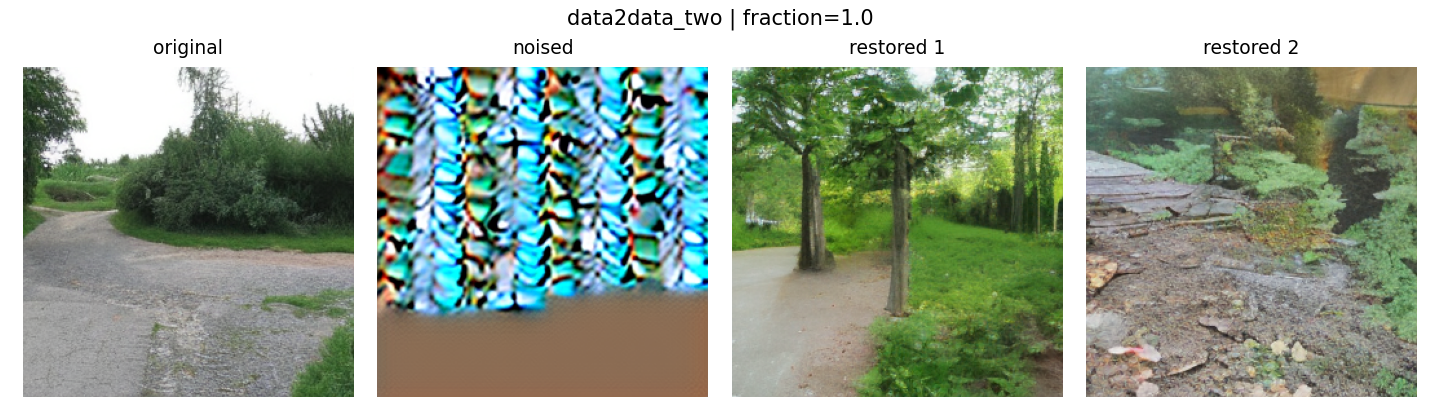}
    \includegraphics[width=\linewidth, trim={0 0 0 0.14\linewidth},clip]{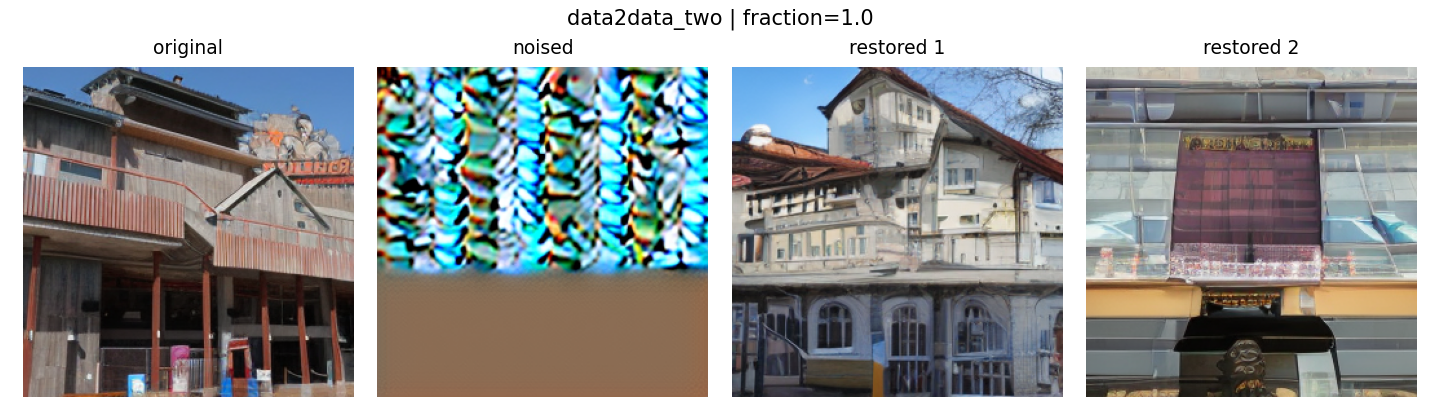}
    \caption{\textbf{\bit{}} (ours)}
\end{subfigure}
\begin{subfigure}[]{0.49\textwidth}
    \includegraphics[width=\linewidth, trim={0 0 0 0.08\linewidth},clip]{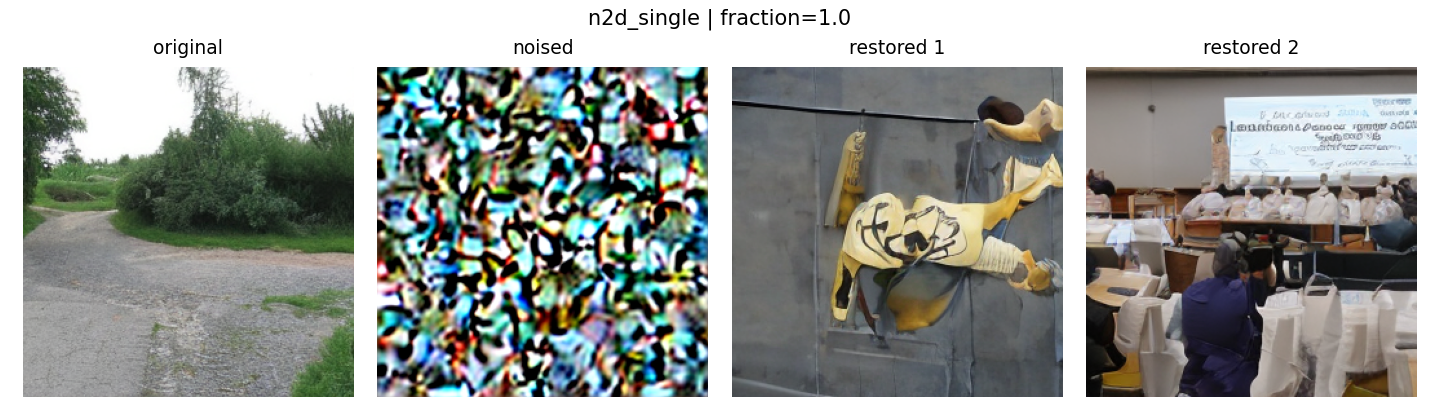}
    \includegraphics[width=\linewidth, trim={0 0 0 0.14\linewidth},clip]{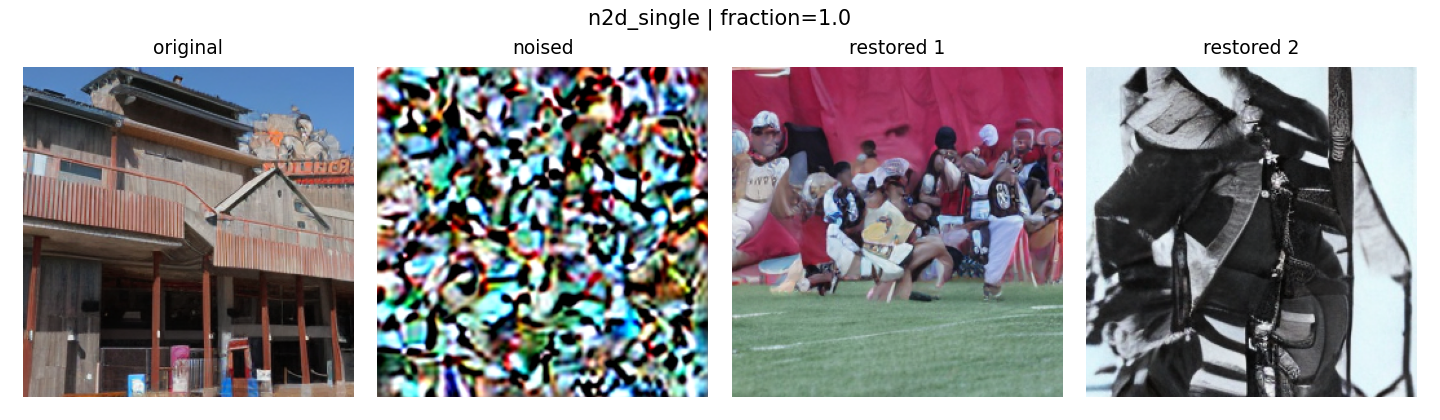}
    \caption{Noise-to-data diffusion \citep{song2020score}.}
\end{subfigure}
\caption{\textbf{Image Stochastic Variation:} Qualitative results corresponding to Figure~\ref{fig:variation}. Because the text endpoint retains source-related information, \bit{} can recover semantically related variations of the source image after reversing the forward corruption process to $t=0$.}\label{fig:variation_qual}
\end{figure}

\begin{figure}[h]
\centering
\begin{tcblisting}{listing only,
    boxsep=0pt,
    top=-4pt, bottom=-4pt, left=2pt, right=2pt,
    toptitle=0pt, bottomtitle=0pt,
    lefttitle=2pt, righttitle=2pt,
    boxrule=1.5pt,
    before skip=2pt, after skip=2pt,
    fonttitle=\ttfamily\scriptsize\bfseries,
    listing options={breaklines=true,                  breakindent=0pt,
    breakautoindent=false,
    basicstyle=\ttfamily\scriptsize},  title=Original Text}
A man and a boy stand together in a park with bare trees. Several people in high-visibility vests are visible in the background near a grassy area.
\end{tcblisting}

\begin{tcblisting}{listing only,
    boxsep=0pt,
    top=-4pt, bottom=-4pt, left=2pt, right=2pt,
    toptitle=0pt, bottomtitle=0pt,
    lefttitle=2pt, righttitle=2pt,
    boxrule=1.5pt,
    before skip=2pt, after skip=2pt,
    fonttitle=\ttfamily\scriptsize\bfseries,
    listing options={breaklines=true,                  breakindent=0pt,
    breakautoindent=false,
    basicstyle=\ttfamily\scriptsize},  title=BIT Data-to-Data (Ours)}
restored 1 : A soldier in uniform stands outdoors on a grassy field, holding papers and speaking. One woman in black clothing is behind him, smiling as they face forward in a patterned crack.
restored 2 : A man stands on a grassy field, looking at a line of bare tree branches filling the frame.
\end{tcblisting}

\begin{tcblisting}{listing only,
    boxsep=0pt,
    top=-4pt, bottom=-4pt, left=2pt, right=2pt,
    toptitle=0pt, bottomtitle=0pt,
    lefttitle=2pt, righttitle=2pt,
    boxrule=1.5pt,
    before skip=2pt, after skip=2pt,
    fonttitle=\ttfamily\scriptsize\bfseries,
    listing options={breaklines=true,                  breakindent=0pt,
    breakautoindent=false,
    basicstyle=\ttfamily\scriptsize}, title=Noise-to-Data Diffusion}
restored 1 : Snow-covered tents stretch across a paved road under a cloudy sky. A dark tunnel handle hub with graffiti markings marks the entrance, flanked by trees and shields on either side of the street.
restored 2 : A butterfly with dark spots and a single archus shell rests on a ship's bottom. Its wings are oriented forward, and sunlight casts a bright shadow below onto it. The background shows pink blurred background and distant seed branches.
\end{tcblisting}
\caption{\textbf{Text Stochastic Variation:} Qualitative text results, corresponding to Fig \ref{fig:variation}. Corruption fraction $90\%$. Diffusion baseline goes off-topic, while \bit{} generates semantically related captions.}\label{fig:variation_text_qual}
\end{figure}

\newpage
\clearpage

\section{Related Works}\label{sec:related_works}

\noindent\textbf{Diffusion Bridges:}
Our work builds on the rich literature in diffusion bridge models, which generalize diffusion models by learning SDEs that interpolate between data distributions (rather than from noise to data) to model conditional distributions~\citep{zhou2023denoising, guo2026abc, kieu2025bidirectional}. DDBM~\citep{zhou2023denoising} is equivalent to the reverse-time SDE of our framework (Theorem~\ref{thm:reverse_process}), but does not consider the forward-time or marginal SDEs and is therefore unidirectional. Our forward-time SDE (Theorem~\ref{thm:forward_process}) is a special case of ABC~\citep{guo2026abc}; unlike that work, we also consider the reverse-time SDE so that we can sample either conditional distribution.
% We make the theoretical contribution unifying ABC~\citep{guo2026abc} (in the two-time setting) and DDBM~\citep{zhou2023denoising}, showing that they describe the same path measure.
BDBM~\citep{kieu2025bidirectional} is the closest to our framework in that they also construct bidirectional models, but their derivations come from the perspective of Chapman-Kolmogorov equations on Gaussian distributions rather than our measure-theoretic perspective on SDEs. \textit{They also do not consider the marginal/unconditional transitions, which precludes their ability to do techniques like classifier-free guidance.
Crucially, none of the previous diffusion bridge papers even attempt text-to-image translation, mostly sticking to translation between observations from the same modality, like image-to-image.}

\noindent\textbf{Data-to-Data Flow:} 
Flow matching with coupled distributions has been proposed as a way to translate directly between different modalities, \textit{e.g.}, image and text. \textit{The main problem is that the flow matching ODE is deterministic, so in a bare-bones implementation, it inherently cannot express the myriad of possibilities in the conditional distribution.} To get around this, there is a variety of hacks proposed in previous works.
FlowTok~\citep{he2025flowtok} and CrossFlow~\citep{liu2025flowing} learn variational encoders~\citep{kingma2013auto} to map text and image into a common domain: crucially, the noise introduced by the variational encoders breaks the inherent determinism of flows, enabling some semblance of stochasticity rather than blurred average predictions. However, the success of their method hinges on careful tuning and training tricks for their encoders that are somewhat specific to the text-to-image domain, \textit{e.g.}, auxiliary contrastive learning with CLIP embeddings. Their framework, as it currently stands, is not clearly applicable out-of-the-box to other domains, while ours is.
\cite{albergo2023stochastic} also addresses data-to-data translation with flows, albeit for tasks such as superresolution and inpainting rather than text-to-image. Their solution to the lack of stochasticity is to perturb source samples with Gaussian noise. This modification removes a clear deterministic reverse mapping when generating in the opposite direction.

\noindent\textbf{Cycle-Consistent Generative Models:} 
CycleGAN~\citep{zhu2017unpaired} also addresses bidirectional translation between domains with a cycle-consistency loss and a GAN loss, but they model it as a deterministic process, which is inherently incorrect. BDBM~\citep{kieu2025bidirectional}, as discussed earlier, has bidirectional capabilities, and they correctly model the translation as stochastic.

\noindent\textbf{Continuous Diffusion for Language:} Recently, there has been interest in methods for continuous-time, continuous-space processes to generate language. Embedded Language Flows~\citep{hu2026elf} has a similar solution to ours, in that they use invertible text token embeddings as the target distribution for a flow matching model. However, they only consider single-modality generation (no image conditioning or generation), and still use the basic construction of Gaussian noise-to-data. Flow Map Language Models~\cite{lee2026flow} and Discrete Flow Maps~\citep{potaptchik2026discrete} learn few-step generators for one-hot token vectors via cross-entropy and consistency losses. Again, they do not consider multimodality tasks like text-to-image.

\noindent\textbf{Text-to-Image:} Text-to-image models are ubiquitous, having moved past the academic sphere to products like Midjourney. However, the dominant research ecosystem models are Flux~\citep{labs2025flux} and Stable Diffusion~\citep{rombach2022high}. They are generally based on noise-to-data conditional diffusion or flow. They are also not invertible, unlike our framework.

\end{document}